 \def\arXiv{1} 
 \documentclass[11pt]{article} 
\usepackage{enumerate}
\newcommand{\notarxiv}[1]{foo}
\newcommand{\arxiv}[1]{ba}
\ifdefined \arXiv
\renewcommand{\arxiv}[1]{#1}%
\renewcommand{\notarxiv}[1]{\ignorespaces}%
\else%
\renewcommand{\arxiv}[1]{\ignorespaces}%
\renewcommand{\notarxiv}[1]{#1}%
\fi%

\arxiv{
	\usepackage{amsthm}
}

\usepackage{thm-restate}
\usepackage{xparse}
\usepackage{enumerate}
\usepackage{amssymb}
\usepackage{amsbsy}
\usepackage{amsfonts}
\usepackage{latexsym}
\usepackage{graphicx}
\usepackage{mathtools}
\usepackage{enumitem}
\usepackage{thmtools}
\usepackage{graphicx}
\usepackage{color}
\usepackage{bbm}
\usepackage{xifthen}
\usepackage{xspace}
\usepackage{titlesec}
\usepackage{setspace}
\usepackage{etoolbox}
\usepackage{xfrac}
\usepackage{esvect}
\usepackage{nicefrac}
\usepackage{cases}
\usepackage{empheq}
\usepackage{booktabs}
\usepackage{tabto}
\usepackage{multirow}
\usepackage{caption}
\usepackage{bbm}
\usepackage{array}
\usepackage{xurl}
\usepackage{comment} 
\usepackage{pifont}
\usepackage{tikz}

\usepackage{makecell}
\usepackage{arydshln}
\usepackage{soul}
\usepackage{adjustbox}
\usepackage{pdflscape}

\notarxiv{
	\usepackage{hyperref}
}
\notarxiv{
	\setcitestyle{numbers,square,comma} 
	\declaretheorem[name=Assumption,sibling=theorem]{assumption}

}

\arxiv{
	\usepackage[linesnumbered,ruled,vlined,boxed,algo2e]{algorithm2e}
	\usepackage[hidelinks]{hyperref}
	\usepackage{nameref}
	\hypersetup{
	colorlinks=true,
	linkcolor=blue!70!black,
	citecolor=blue!70!black,
	urlcolor=blue!70!black}
}

\usepackage{cleveref}
\crefname{assumption}{Assumption}{Assumptions}

\arxiv{
	\usepackage{crossreftools}
	\pdfstringdefDisableCommands{
		\let\Cref\crtCref
		\let\cref\crtcref
	}
}

\arxiv{
	\usepackage[top=1in, right=1in, left=1in, bottom=1in]{geometry}
	\usepackage[numbers,square]{natbib}

	\theoremstyle{plain}
	
	\newtheorem{theorem}{Theorem}
	\newtheorem{lemma}{Lemma}
	
	\newtheorem{proposition}{Proposition}
	\newtheorem{corollary}{Corollary}

	\theoremstyle{definition}

	\newtheorem{assumption}{Assumption}

	\newtheorem*{example*}{Example}
}

\newcommand{\mc}[1]{\mathcal{#1}}

\DeclarePairedDelimiter{\abs}{\lvert}{\rvert} %
\DeclarePairedDelimiter{\brk}{[}{]}
\DeclarePairedDelimiter{\crl}{\{}{\}}
\DeclarePairedDelimiter{\prn}{(}{)}

\DeclarePairedDelimiter{\norm}{\|}{\|}

\newcommand{\overeq}[1]{\overset{#1}{=}}
\newcommand{\overle}[1]{\overset{#1}{\le}}
\newcommand{\overge}[1]{\overset{#1}{\ge}}

\NewDocumentCommand\Ex{s O{} m }{%
	\mathbb{E}%
	\begingroup
	\IfBooleanTF{#1}
	{\ExInn*{#3}}
	{\ExInn[#2]{#3}}%
	\endgroup
}

\DeclarePairedDelimiterX\ExInn[1]{[}{]}{%
	\activatebar
	#1%
}

\RenewDocumentCommand\Pr{sO{}r()}{%
	\mathbb{P}%
	\begingroup
	\IfBooleanTF{#1}
	{\PrInn*{#3}}
	{\PrInn[#2]{#3}}%
	\endgroup
}

\DeclarePairedDelimiterX\PrInn[1](){%
	\activatebar
	#1%
}

\newcommand{\activatebar}{%
	\begingroup\lccode`~=`|
	\lowercase{\endgroup\def~}{\;\delimsize\vert\;}%
	\mathcode`|=\string"8000
}

\newcommand\numberthis{\addtocounter{equation}{1}\tag{\theequation}}

\newcommand{\R}{\mathbb{R}} %
\newcommand{\N}{\mathbb{N}} %
\newcommand{\E}{\mathbb{E}} %
\renewcommand{\P}{\mathbb{P}}	%
\usepackage{xcolor}

 \DeclareMathOperator*{\argmax}{arg\,max}
 \DeclareMathOperator*{\argmin}{arg\,min}

\providecommand{\abs}{\mathop{\rm abs}}

\newcommand{\half}{\frac{1}{2}}

\newcommand{\defeq}{\coloneqq}

\newcommand{\grad}{\nabla}

\newcommand{\xset}{\mathcal{X}}

\renewcommand{\O}[1]{O\left( #1 \right)}
\newcommand{\Otil}[1]{\widetilde{O}\left( #1 \right)}

\newcommand{\inner}[2]{\left<#1,#2\right>}

\newcommand{\numSteps}{T}

\newcommand{\Q}[1][\numSteps]{Q_{#1}}
\newcommand{\minQ}[1][\numSteps]{\hat{Q}_{#1}}
\newcommand{\maxQ}[1][\numSteps]{\bar{Q}_{#1}}
\newcommand{\M}[1][\numSteps]{M_{#1}}
\newcommand{\q}[1][\numSteps]{q_{#1}}

\newcommand{\tq}[1][\numSteps]{\tilde{q}_{#1}}
\newcommand{\noisep}[1][\numSteps]{p_{#1}}
\newcommand{\Gx}[1][\numSteps]{G_{x,#1}}
\newcommand{\Gy}[1][\numSteps]{G_{y,#1}}
\newcommand{\rbar}[1][\numSteps]{\bar{r}_{#1}}
\newcommand{\w}[1][t]{\omega_{#1}}

\newcommand{\dbar}[1][\numSteps]{\bar{d}_{#1}}
\newcommand{\dtilde}[1][\numSteps]{\tilde{d}_{#1}}
\newcommand{\zbar}[1][t]{\hat{z}_{#1}}
\newcommand{\xbar}[1][t]{\hat{x}_{#1}}
\renewcommand{\d}[1][0]{d_{#1}}
\renewcommand{\r}[1][0]{r_{#1}}
\newcommand{\TimeUniformLog}[1][t,\delta]{\theta_{#1}}

\newcommand{\exm}[1][t]{\nabla f \prn*{ \zbar[#1] }}
\newcommand{\exg}[1][t]{\nabla f \prn*{ \xbar[#1] }}
\newcommand{\m}[1][t]{m_{#1}}
\newcommand{\g}[1][t]{g_{#1}}
\newcommand{\tg}[1][t]{\tilde{g}_{#1}}

\newcommand{\filt}{\mathcal{F}}
\NewDocumentCommand{\llft}{ O{\delta} O{t} }{\lambda_{#2}(#1)}



\newcommand{\opt}{_\star}
\newcommand{\xopt}{x\opt}

\newcommand{\dopt}{d_0}
\newcommand{\diam}{D}

\makeatletter
\def\IfEmptyTF#1{%
	\if\relax\detokenize{#1}\relax
	\expandafter\@firstoftwo
	\else
	\expandafter\@secondoftwo
	\fi}
\makeatother

\SetKwInput{Input}{Input}                %
\SetKwInput{Output}{Output}              %
\SetKwInput{Parameters}{Parameters}
\SetKwProg{Function}{function}{}{}
\SetKwComment{Comment}{$\triangleright$\ }{}
\SetKwRepeat{Do}{do}{while}
\SetKwBlock{Repeat}{Repeat}{}
\SetCommentSty{mycommfont}
\makeatletter
\newcommand\Block[2]{%
	#1%
	\algocf@group{#2}%
}
\makeatother

\newcommand{\gradientOracle}[1]{\IfEmptyTF{#1}{\mathcal{G}}{\mathcal{G}(#1)}}

\newcommand{\reps}{r_\epsilon}

\newcommand{\Proj}[2]{\mathbf{\mathrm{Proj}}_{#1}(#2)}

\newcommand{\DoG}{\textsc{DoG}\xspace}
\newcommand{\TDoG}{\textsc{T-DoG}\xspace}

\newcommand{\UDoG}{\textsc{U-DoG}\xspace}
\newcommand{\ADoG}{\textsc{A-DoG}\xspace}
\newcommand{\UniXGrad}{\textsc{UniXGrad}\xspace}
\newcommand{\AcceleGrad}{\textsc{AcceleGrad}\xspace}

\newcommand{\pick}{\sim}
\renewcommand{\L}{\beta}
\newcommand{\logp}{\log_{+}}
\newcommand{\lip}{L}

\newcommand{\etat}{\tilde{\eta}}

\newcommand{\bfunc}{\mathfrak{b}}
\newcommand{\bk}[1]{\bfunc_{#1}}
\newcommand{\bkbar}[1]{\bar{\bfunc}_{#1}}
\newcommand{\btbar}[1]{\hat{\bfunc}_{#1}}
\newcommand{\bstar}{\bfunc_\star}

\newcommand{\sfunc}{\sigma}

\newcommand{\skbar}[1]{\bar{\sfunc}_{#1}}
\newcommand{\sstar}{\sfunc_\star}

\newcommand{\stob}{\varsigma}

\newcommand{\empVar}[1][t]{V_{#1}}

\newcommand{\NB}{\mathfrak{B}} %

\newcommand{\kset}{\mathcal{K}}

\newcommand{\DR}[2]{\mathcal{D}_\mathcal{R}\prn*{{#1}, {#2}}}

\newcommand{\citeg}[1]{\citep[e.g.,][]{#1}}  %

\newcommand{\indic}[1]{\mathbbm{1}_{\!\left\{#1\right\}}} %

\newcommand{\tf}{\tilde{f}}

\makeatletter
\long\def\@makecaption#1#2{
  \vskip 0.8ex
  \setbox\@tempboxa\hbox{\small {#1:} #2}
  \parindent 1.5em  %
  \dimen0=\hsize
  \advance\dimen0 by -3em
  \ifdim \wd\@tempboxa >\dimen0
  \hbox to \hsize{
    \parindent 0em
    \hfil 
    \parbox{\dimen0}{\def\baselinestretch{0.96}\small
      {#1:} #2
    } 
    \hfil}
  \else \hbox to \hsize{\hfil \box\@tempboxa \hfil}
  \fi
}
\makeatother

\newcommand{\insertfigure}[3]{
\begin{figure}[t]
	\centering
	\includegraphics[width=\linewidth]{figures/#1}
	\caption{#2}\label{#3}
\end{figure}
}

\newcommand{\insertfigurepage}[6]{
\begin{figure}[t]
    \centering
	\hspace*{10pt}
    \includegraphics[width=\linewidth]{figures/#1}
    \caption{#2}\label{#3}
    \vspace{30pt} %
    \includegraphics[width=\linewidth]{figures/#4}
    \caption{#5}\label{#6}
\end{figure}
}

\newcommand{\yair}[1]{{\bf \color{blue} Yair: #1}}
\newcommand{\ollie}[1]{{\bf \color{green!50!black} Ollie: #1}}
\newcommand{\maor}[1]{{\bf \color{cyan} Maor: #1}}
\newcommand{\itai}[1]{{\bf \color{orange} Itai: #1}}
\newcommand{\TBD}[1]{{\bf \color{red} TBD: #1}}
\newcommand{\TODO}[1]{{\bf \color{purple} TODO: #1}}

\renewcommand{\yair}[1]{\ignorespaces}
\renewcommand{\ollie}[1]{\ignorespaces}
\renewcommand{\maor}[1]{\ignorespaces}
\renewcommand{\itai}[1]{\ignorespaces}
\renewcommand{\TBD}[1]{\ignorespaces}
\renewcommand{\TODO}[1]{\ignorespaces} \notarxiv{\usepackage{times}}

\title{Accelerated Parameter-Free Stochastic Optimization}

\notarxiv{
	\coltauthor{%
	 \Name{Itai Kreisler} \Email{kreisler@mail.tau.ac.il}\\
	 \Name{Maor Ivgi} \Email{maor.ivgi@cs.tau.ac.il}\\
	 \Name{Oliver Hinder} \Email{ohinder@pitt.edu}\\
	 \Name{Yair Carmon} \Email{ycarmon@tauex.tau.ac.il}\\
	}
}
\arxiv{
	\author{Itai Kreisler\thanks{Tel Aviv University,
			\href{mailto:kreisler@mail.tau.ac.il}{\texttt{kreisler@mail.tau.ac.il}}, 
			\href{mailto:maor.ivgi@cs.tau.ac.il}{\texttt{maor.ivgi@cs.tau.ac.il}}, 
			\href{mailto:ycarmon@tauex.tau.ac.il}{\texttt{ycarmon@tauex.tau.ac.il}.}
		}~~~~~~~Maor Ivgi\footnotemark[1]
		~~~~~~Oliver Hinder\thanks{University of Pittsburgh, \href{mailto:ohinder@pitt.edu}{\texttt{ohinder@pitt.edu}}}
		~~~~~~Yair Carmon\footnotemark[1]}

	\date{}
}

\begin{document}

\maketitle

\begin{abstract}%
  We propose a method that achieves near-optimal rates for \emph{smooth} stochastic convex optimization and requires essentially no prior knowledge of problem parameters.
  This improves on prior work which requires knowing at least the initial distance to optimality $\dopt$.  
  Our method, $\UDoG$, combines $\UniXGrad$ (\citet{kavis2019unixgrad}) and $\DoG$ (\citet{ivgi2023dog}) with novel iterate stabilization techniques.
  It requires only loose bounds on $\dopt$ and the noise magnitude, provides high probability guarantees under sub-Gaussian noise, and is also near-optimal in the non-smooth case.
  Our experiments show consistent, strong performance on convex problems and mixed results on neural network training.

\end{abstract}

\notarxiv{
}

\section{Introduction}\label{sec:intro}

We consider the problem of minimizing a smooth convex function using access to an unbiased stochastic gradient oracle.
This is a fundamental problem in machine learning, including many important special cases such as logistic and linear regression.
Moreover, the smoothness assumption is crucial for developing one of the most widely used improvements for the classical gradient method: Nesterov acceleration~\citep{nesterov1983method}.

Nesterov acceleration obtains the optimal rate of convergence for this problem but is strongly reliant on knowing the problem parameters.
Specifically, \citet{lan2012optimal}, who first demonstrated the theoretical value of Nesterov acceleration on smooth \emph{stochastic} convex
functions, requires knowledge of the smoothness parameter $\beta$, the distance $\dopt$ from the initial point to the optimum, and a value $\sigma$ for which the noise is $\sigma$-sub-Gaussian.
Accelerated adaptive methods~\citep{cutkosky2019anytime,kavis2019unixgrad}
do not require knowledge of $\beta$ and $\sigma$, but assume knowledge of $\dopt$.
For \emph{non-smooth} stochastic convex optimization, \emph{parameter-free methods} \citeg{orabona2016coin, cutkosky2018black, bhaskara2020online,  mhammedi2020lipschitz, jacobsen2022parameter, carmon2022making, ivgi2023dog} require only loose knowledge of problem parameters to obtain near-optimal rates.   
Finding such parameter-free methods for \emph{smooth} stochastic optimization is a longstanding open problem.

\paragraph{Our contribution.} 
We solve this open problem, designing an accelerated parameter-free method which we call \UniXGrad-\DoG, or $\UDoG$ for short.
$\UDoG$ combines the ``universal extragradient'' (\UniXGrad) framework~\citep{kavis2019unixgrad} with the ``distance over gradient'' (\DoG) technique~\citep{ivgi2023dog}.
More specifically, we replace the domain diameter $D$ in the \UniXGrad step size numerator with the maximum distance from the initial point, similar to the DoG step size numerator.
Furthermore, we use this maximum distance to automatically tune the ``momentum'' parameter $\alpha_t$ of \UniXGrad.
Finally, we modify the \UniXGrad step size denominator to ensure the stability of the iterate sequence. 
$\UDoG$ only requires a loose upper bound $\hat{\sigma}$ on $\sigma$ and lower bound $\reps$ on $D$.\footnote{In fact, we only require \emph{local} upper bounds of the form $\hat{\sigma}(x)$ on the noise sub-Gaussianity. 
}
As long as $\hat{\sigma}$ is loose by at most a $\sqrt{T}$ factor and $\reps$ is loose by any $\mathrm{poly}(T)$ factor, we obtain a near-optimal, high-probability rate of convergence; \Cref{table:summary} states $\UDoG$'s guarantees and compares it to prior work.
Moreover, $\UDoG$ simultaneously enjoys a near-optimal, parameter-free rate of convergence for \emph{non-smooth} problems.

We conduct preliminary experiments with $\UDoG$ as well as another algorithm, $\ADoG$, which combines \AcceleGrad~\citep{levy2018online} and DoG.
On convex optimization problems, both $\UDoG$ and $\ADoG$ often substantially improve over $\DoG$, especially at large batch sizes, with $\ADoG$ outperforming $\UDoG$, likely due to not requiring an extra-gradient computation at each step.
On several problems, $\ADoG$ matches the performance of carefully tuned SGD with Nesterov momentum.
On neural network optimization problems, however, we observe that both \UDoG and \ADoG do not consistently improve over \DoG.

\newcommand{\yes}{{\color{green!75!black}\text{\ding{51}}}}
\newcommand{\yesyes}{{\color{green!40!black}\text{\ding{51}}}}
\newcommand{\no}{{\color{red!75!black}\text{\ding{55}}}}
\newcommand{\nyes}{{\color{black}\text{\ding{51}}}}
\newcommand{\nno}{{\color{black}\text{\ding{55}}}}
\newcommand{\yesno}{%
  \begin{tikzpicture}[baseline=-0.75ex]
    \node[inner sep=0pt] at (0,0) {\color{orange}\ding{51}}; %
    \node[inner sep=0pt] at (0.02,0) {\color{orange}\ding{55}}; %
  \end{tikzpicture}%
}

\begin{table}[]
	\centering
	\everymath{}
	\newcolumntype{S}{>{\centering\arraybackslash}p{16pt}}
	\begin{tabular}{@{}lcSSSlc@{}}
		\toprule
		\multirow{2}{*}{\begin{tabular}[c]{@{}l@{}}Algorithm\\ name\end{tabular}} & \multirow{2}{*}{\begin{tabular}[c]{@{}c@{}}Unbounded \\ domain?\end{tabular}} & \multicolumn{3}{c}{Insensitive to...} & \multicolumn{1}{c}{\multirow{2}{*}{\begin{tabular}[c]{@{}c@{}}Rate of\\ convergence\end{tabular}}} & \multirow{2}{*}{\begin{tabular}[c]{@{}c@{}}High\\ probability?\end{tabular}} \\
		 &  & $\dopt$/$\diam$ & $\beta$ & $\sigma$ & \multicolumn{1}{c}{} &  \\ \midrule
		\multirow{2}{*}{$\UDoG$ (this work)} & $\nyes$ & $\yes$ & $\yesyes$ & $\yes$ & $\Otil{ \frac{\beta \dopt^2}{T^2} + \frac{\sigma \dopt}{\sqrt{T}} + \frac{\hat{\sigma} \dopt}{T} }$ & $\yes$ \\
		 & $\nno$ & $\yes$ & $\yesyes$ & $\yesyes$ & $\Otil{ \frac{\beta \diam^2}{T^2} + \frac{\sigma \diam}{\sqrt{T}} }$ & $\yes$ \\ \midrule
		$\UniXGrad$~\cite{kavis2019unixgrad} & $\nno$ & $\no$ & $\yesyes$ & $\yesyes$ & $\O{ \frac{\beta \diam^2}{T^2} + \frac{\sigma \diam}{\sqrt{T}} }$ & $\no$ \\
		\citet{cutkosky2019anytime} & $\nyes$ & $\no$ & $\yesyes$ & $\yes$ & $\Otil{ \frac{\beta \dopt^2}{T^2} + \frac{\sigma \dopt}{\sqrt{T}} }$ & $\no$ \\
		\citet{lan2012optimal} & $\nyes$ & $\no$ & $\no$ & $\no$ & $\O{ \frac{\beta \dopt^2}{T^2} + \frac{\sigma \dopt}{\sqrt{T}} }$ & $\yes$ \\
		\DoG~\cite{ivgi2023dog} / CO~\cite{cutkosky2018black} & $\nyes$ & $\yes$ & $\yesyes$ & $\yesno$ & $\Otil{ \frac{\beta \dopt^2}{T} + \frac{\sigma \dopt}{\sqrt{T}} + \frac{\hat{\lip} \dopt}{T} }$ & $\yes$ / $\no$ \\ \bottomrule
		\end{tabular}
	\everymath{}
	\caption{
	Comparison of $\UDoG$ and prior work on $\beta$-smooth stochastic optimization with $\sigma$-sub-Gaussian noise. ``Unbounded domain'' indicates if the algorithm is defined over the whole Euclidean space or a bounded subspace. In the former case we express rates in terms of the initial distance to optimality $\dopt$ and in the latter case we use the 
	domain diameter $\diam$. Under ``Insensitive to...'' we mark $\no$ if the suboptimality bound grows polynomially with error in the parameter, $\yes$ if it only affects logarithmic factors or low order terms, and $\yesyes$ if there is no dependence on the parameter at all.
	The marker \protect\yesno{} indicates algorithms that require an upper bound 
	$\hat{\lip}$ on gradient norm, which may be much larger the the upper bound $\hat{\sigma}$ on the noise. The notation $\Otil{\cdot}$ hides polylogarithmic factors. 
	\label{table:summary}
	}
\end{table}

\subsection{Related work}

\paragraph{Non-smooth stochastic optimization.}
The majority of tuning-insensitive stochastic optimization methods are developed for online convex optimization. Online regret bounds immediately translate to suboptimality guarantees for non-smooth stochastic optimization using online-to-batch conversion \cite[Section 3]{orabona2021modern}. Proposed methods divide roughly into \emph{adaptive} algorithms such as adaptive SGD~\citep{mcmahan2017survey,gupta2017unified}, AdaGrad~\citep{duchi2011adaptive,mcmahan2010adaptive} and variants~\citeg{kingma2015adam,reddi2018convergence,shazeer2018adafactor}, and \emph{parameter-free} methods \cite{streeter2012no,orabona2013dimension,mcmahan2014unconstrained,orabona2016coin,cutkosky2018black,cutkosky2019artificial,bhaskara2020online,mhammedi2020lipschitz,jacobsen2022parameter}. Adaptive methods typically require no knowledge of the stochastic gradient bound but need to know the initial distance to optimality (or the domain diameter), while parameter-free methods are robust to uncertainty in the distance but require some (loose) bound on the stochastic gradient norms. 

Recent work~\citep{carmon2022making,ivgi2023dog} develops parameter-free methods that hew closer to SGD and eschew online-to-batch conversion for high-probability guarantees in the stochastic setting; $\UDoG$ continues this line. In particular, it extends the core mechanism of $\DoG$ \citep{ivgi2023dog} wherein iterate movement serves as a proxy for the distance to optimality. D-Adaptation~\citep{defazio2023learning}, DoWG~\citep{khaled2023dowg}, and Prodigy~\citep{mishchenko2023prodigy} use a similar mechanism, but only provide guarantees for the non-stochastic setting. Ensuring the validity of the mechanism (i.e., that iterates never move too far away from the optimum) is a key challenge in its analysis. This challenge becomes greater in the smooth setting, where selecting too small of a step size nullifies the benefit of acceleration. Much of our algorithmic and analytical innovation addresses this challenge. 

\paragraph{Non-stochastic smooth optimization.} 
Without noise, Nesterov acceleration requires knowledge of the smoothness constant $\beta$ but not the distance to optimality~\cite{nesterov1983method,nesterov2013introductory}. The methods~\cite{levy2018online,kavis2019unixgrad} reverse this tradeoff, requiring the distance but not $\beta$. Line search techniques such as~\citep{beck2009fast,carmon2022optimal} provide much stronger adaptivity, attaining the optimal gradient evaluation complexity up to an additive term that depends logarithmically on the uncertainty in $\beta$. However, line search can be challenging to employ efficiently in the stochastic setting as we can no longer accurately evaluate the function. 
Indeed, there are many works that analyze stochastic line search techniques \citeg{paquette2020stochastic,vaswani2019painless} but none have obtained convergence guarantees close to that of \citet{lan2012optimal}.

\paragraph{Smooth stochastic optimization.} 
Several adaptive and parameter-free methods~\cite{gupta2017unified,cutkosky2018black,carmon2022making,ivgi2023dog,khaled2023dowg} converge faster on smooth functions. However, they do not improve all the way to the optimal rate (see \Cref{table:summary}) due to a missing ``momentum'' component. \citet{cutkosky2019anytime} gives an improved online-to-batch conversion framework that endows adaptive SGD with momentum and accelerated rates in the smooth case, but requires a bound on the distance to optimality. \citet{kavis2019unixgrad} propose $\UniXGrad$, combining ideas from~\citep{cutkosky2019anytime} with the mirror-prox/extragradient algorithm~\cite{nemirovski2004prox,diakonikolas2018accelerated} and online learning~\cite{mcmahan2017survey,rakhlin2013optimization} to obtain optimal rates assuming bounded domains of known diameter $\diam$ and assuming that $\dopt$ is of the order of $\diam$. $\UDoG$ modifies $\UniXGrad$ and removes both assumptions, yielding the first parameter-free accelerated method.  
\section{Preliminaries and algorithmic framework}\label{sec:preliminaries}

In this section, we set up our notation and terminology, and use them to present the general $\UDoG$ template (\Cref{alg: general unixgrad dog}) defining the algorithm up to the choice of adaptive step sizes, which we gradually develop in the following sections.

\paragraph{Basic notation and conventions.} Throughout, $\norm{\cdot}$ denotes the Euclidean norm, $\log$ is base $e$ and $\logp(x) \defeq 1+\log(x)$. 
The function $\Proj{\xset}{\cdot}$ denotes Euclidean projection onto set $\xset$. We say that $f : \kset\to\R$ is $\L$-smooth if $\grad f$ is $\L$-Lipschitz, i.e., $\norm{ \grad f(u) - \grad f(v) } \le \L \| u - v \|$ for all $u, v \in \kset$. We write $\brk*{\cdot}_{+} := \max\{ \cdot, 0 \}$.

\vspace{\baselineskip}
In this work, we minimize an objective function $f$ via queries to a stochastic gradient estimator $\gradientOracle{}$.
We make the following assumption in all of our theoretical analyses.

\begin{assumption}[Made throughout]\label{ass:global}
	The objective function $f:\kset\to\R$ is convex, $\lip$-Lipschitz, $\beta$-smooth,\footnote{Our results hold in the non-Lipschitz or non-smooth cases by setting $\lip=\infty$ or $\beta=\infty$, respectively. In the non-smooth case, we define $\grad f(x)\defeq \E \gradientOracle{x}$ and assume it is a subgradient of $f$.}  has closed convex domain $\kset$, and its minimum is attained at some $\xopt \in \argmin_{x\in\kset} f(x)$.  For all $x\in\kset$, the gradient estimator $\gradientOracle{}$ satisfies $\E \gradientOracle{x} = \grad f(x)$.
\end{assumption}
\SetAlgoNlRelativeSize{-1}
\begin{algorithm2e}
	\notarxiv{\setstretch{1.5}}
	\arxiv{\setstretch{1.4}}
	\caption{\UDoG (\UniXGrad-\DoG) template}
	\label{alg: general unixgrad dog}
	\LinesNumbered
	\DontPrintSemicolon
	\Input{
		Initial $x_{0}\in \kset$, iteration budget $T$, initial movement $\reps$, step sizes $\{\eta_{x,t},\eta_{y,t}\}$
	}
	Set $y_{0} = x_{0}$ \;%
	\For{$t = 0, 1, 2, \ldots, T-1$ }{
		Set $\alpha_t = \sum_{k=0}^{t} {\rbar[k]}/{\rbar[t]}$ ~~and~~ $\w[t] = \alpha_t \rbar[t]$ ~~for~~ 
		$\displaystyle\rbar[t] = \max_{k \le t} \max\crl*{\norm{y_k - x_0}, \norm{x_k-x_0},\reps}$
		\;
		$x_{t+1} = \Proj{\kset}{y_{t} - \alpha_t \eta_{x,t} \m }$ \tabto{5cm}for~~ $\m[t]\sim \gradientOracle{\zbar}$ \tabto{7.9cm} and~~ $\zbar = \frac{\w[t] y_{t} + \sum_{k=0}^{t-1} \w[k] x_{k+1} }{\sum_{k=0}^{t} \w[k]}$\;
		$y_{t+1} = \Proj{\kset}{y_{t} - \alpha_t \eta_{y,t} \g }$ \tabto{5cm} for~~ $\g[t]\sim \gradientOracle{\xbar}$ \tabto{7.9cm} and~~ $\xbar = \frac{\w[t] x_{t+1} + \sum_{k=0}^{t-1} \w[k] x_{k+1} }{\sum_{k=0}^{t} \w[k]}$\;
	}
	\Return $\xbar[T]$ 

\end{algorithm2e}

\paragraph{Presenting $\UDoG$.} \Cref{alg: general unixgrad dog} provides the general template of $\UDoG$. As in $\UniXGrad$ \citep{kavis2019unixgrad}, each iteration of the algorithm consists of two stochastic gradient steps, with each stochastic gradient queried at a moving average of iterates. Unlike $\UniXGrad$, the moving average weights $\w[t]$ and the step size multipliers $\alpha_t$ are not fixed in advance, but are instead dynamically set based on the maximum distance moved from the origin, denoted
\[
	\rbar[t] \defeq \max_{k \le t} \max\crl*{\norm{y_k - x_0}, \norm{x_k-x_0},\reps}.
\]
The parameter $\reps$ serves as a (loose) lower bound on  $\norm{x_0 - \xopt}$; typically, $\rbar[t]$ grows rapidly and then plateaus at a level roughly approximating $\norm{x_0 - \xopt}$. When that happens, the sequence $\alpha_t = \sum_{k \le t} \rbar[k] / \rbar[t]$ grows linearly in $t$, similar to  $\alpha_t = t+1$ in $\UniXGrad$. 

To complete the specification of $\UDoG$ we must set the step size sequence. $\UniXGrad$ assumes $\kset$ the domain has Euclidean diameter $D$ and picks step sizes of the form $\eta_{x,t} = \eta_{y,t} = \frac{\sqrt{2}D}{\sqrt{1+\Q[t-1]}}$ where
\[
	\Q[t] \defeq \sum_{k=0}^{t} \q[k] ~~\mbox{and}~~\q[t] \defeq \alpha_t^2 \norm{ \g - \m }^2.
	\numberthis \label{eq: Q notation}
\]
To handle unknown domain size and unbounded domains, $\UDoG$ follows $\DoG$ in using $\rbar[t]$ as the step size numerator in lieu of $D$. Thus, the $\UDoG$ step size admits the general form
\begin{align}
	\eta_{x,t} = \frac{\rbar[t]}{\sqrt{\Gx[t]}} 
	~~\text{and}~~ 
	\eta_{y,t} = \frac{\rbar[t]}{\sqrt{\Gy[t]}},
	~~\mbox{where}~~\Gx[0]\le \Gy[0] \le \Gx[1] \le \cdots
	.
	\label{eq:step-size-form}
\end{align}
In the appendix, we also use the notation 
\begin{align}
	\etat_{x,t} = \frac{1}{\sqrt{\Gx[t]}} 
	~~\text{and}~~ 
	\etat_{y,t} = \frac{1}{\sqrt{\Gy[t]}}
	.
	\label{eq: etat notation}
\end{align}
For bounded domains, setting
$\Gx[t] =\Gy[t] = 1+ \Q[t-1]$ 
recovers the $\UniXGrad$ guarantees up to logarithmic factors. However, for unbounded domains, ensuring the stability of $\UDoG$ (i.e., that $\rbar[t]$ never grows much larger than $\norm{x_0-\xopt}$) requires more careful selection of $\Gx[t],\Gy[t]$. Enforcing iterate stability without compromising the rate of convergence is the main challenge we overcome. To that end, we define a few frequently appearing quantities:
\begin{align*}
	\r[t] &\defeq \max\crl*{\norm{y_k - x_0}, \norm{x_k-x_0}}~,~
	\d[t] \defeq \norm{y_t - \xopt}~,~
	\dbar[t]\defeq\max_{k\le t}\d[k] ~,~\\
	\M[t]&\defeq \max_{k\le t} \crl*{\alpha_k^2 \norm{\m[k]}^2 }
	~\mbox{and}~\theta_{t,\delta} \defeq \log \frac{60 \log \prn*{6t}}{\delta}
	.
\end{align*}

\paragraph{$\UniXGrad$ as a special case.} For a domain with Euclidean diameter $D$, setting $\reps = D\sqrt{2}$ and $\Gx[t]=\Gy[t]=1+\Q[t-1]$ recovers  $\UniXGrad$ (with Euclidean distance generating function) exactly, as it implies $\rbar[t] = D\sqrt{2}$ for all $t$ and hence $\alpha_t=t+1$. 
\section{Analysis in the noiseless case}\label{sec:noiseless}

We begin our analysis under the simplifying assumption that gradients are computed exactly.

\begin{assumption}\label{ass:noiseless}
	In addition to \Cref{ass:global}, we assume that $\gradientOracle{x} = \grad f(x)$ with probability 1.
\end{assumption}

\noindent
This noiseless setting allows us to isolate and address the key challenges of exploiting smoothness and stabilizing the iterates.

\subsection{General suboptimally bound}\label{subsec:noiseless-subopt}

Our first result is a bound on the suboptimality of $\UDoG$ for general step sizes; see \Cref{app:noiseless-subopt-proof} for complete proof. To interpret \Cref{prop:noiseless-subopt} recall that $\d$ is the initial distance to the optimum and the definition of $\Q[t]$ given in \eqref{eq: Q notation}.

\begin{proposition}\label{prop:noiseless-subopt}
	In the noiseless setting (\Cref{ass:noiseless}), suppose the \UDoG step sizes~\eqref{eq:step-size-form} satisfy  $\Gx[t] \ge \Q[t-1]$ %
	for all $t\ge 0$. Then for every $t\ge0$ and for any number $s\ge0$, we have
	\begin{align*}
		f\prn{\xbar[t]} - f\prn{x\opt}
		\le \O{\frac{ s^{3/2} \L \prn*{ \rbar[t+1] + \d}^2 + \prn*{ \rbar[t+1] + \d} \brk*{ \sqrt{\max\{\Gy[t],\Q[t]\}} - s \sqrt{\Q[t]} }_{+}}{ \prn*{\sum_{k=0}^{t} \rbar[k]/\rbar[t+1] }^2 } }.
		\numberthis \label{eq:noiseless-subopt}
	\end{align*}
\end{proposition}

Before sketching the proof of \Cref{prop:noiseless-subopt}, let us explain how it yields the desired rates of convergence if we momentarily set aside iterate stability and assume $\rbar[t]\le \diam$ for all $t$, e.g., because the domain has diameter $\diam$. In this case, we may choose $\Gx[t]=\Gy[t]=\Q[t-1]$ similarly to $\UniXGrad$. Substituting $s=1$ in~\cref{eq:noiseless-subopt} guarantees suboptimality 
$O\prn[\Big]{\frac{\beta \diam^2}{\prn*{\sum_{k=0}^{t} \rbar[k]/\rbar[t+1] }^2}}$. As shown in~\cite[Lemma 3]{ivgi2023dog}, we have $\max_{t<T} \sum_{k=0}^{t} \rbar[k]/\rbar[t+1] = \Omega\prn*{ T \log^{-1}(\rbar[T]/\reps )}$, meaning that for some $t<T$ we obtain the near-optimal rate $O\prn*{ \frac{\beta\diam^2}{T^2} \log^2\frac{D}{\reps}}$. Moreover, since $\alpha_t \le t+1$ for all $t$, when all gradients are bounded by $L$ we have $\Q[t] = O(L^2 \sum_{k\le t} \alpha_k^2) = O(L^2 t^3)$. Substituting $s=0$ in  \cref{eq:noiseless-subopt} and reusing our bound on the denominator gives the near-optimal rate $O\prn*{ \frac{\lip \diam}{\sqrt{T}} \log^2\frac{D}{\reps}}$ in the non-smooth setting. 
We also see that setting $\reps=\Omega(D)$ recovers the $\UniXGrad$ guarantees in the noiseless setting, which is to be expected since $\reps=D\sqrt{2}$ recovers $\UniXGrad$ itself as explained in the previous section. 

Our proof of \Cref{prop:noiseless-subopt} combines ideas from the analyses of $\UniXGrad$ and $\DoG$. It centers on the weighted ``regret'' $\mc{R}_t \defeq \sum_{k=0}^{t}\w[k] \inner{\g[k]}{x_{k+1}-\xopt}$ where $\w[k]=\alpha_k \rbar[k]$. This is similar to the weighted regret considered for $\UniXGrad$ with additional weighting by $\rbar[t]$ used in the \DoG analysis. Algebraic manipulation of $\mc{R}_t$ gives (recall that $\d[t] = \norm{y_t -\xopt}$),
\[
	\mc{R}_t \le \O{ \rbar[t+1]^2 \sqrt{\Q[t] } +
	\sum_{k=0}^t \prn*{\d[k]^2 - \d[k+1]^2}\sqrt{\Gy[k]} -
	\sum_{k=0}^t \norm{x_{k+1}-y_k}^2 \sqrt{\Q[k]}
	}.
\]
We use a telescoping argument from \DoG in order to bound $\sum_{k=0}^t\prn*{\d[k]^2 - \d[k+1]^2}\sqrt{\Gy[k]}$ by \linebreak $\O{ \rbar[t+1](\rbar[t+1] + \d[0])\sqrt{\Gy[t]} }$. Next, following \UniXGrad we leverage smoothness to write 
\[
	\norm{x_{k+1}-y_k}^2 = 
	\prn*{\frac{\sum_{i=0}^k \w[i]}{\w[k]}}^2
	\norm{\xbar[k] - \zbar[k]}^2 
	\overge{\text{\tiny Lem.~\ref{lem: weight sum of alphas}}} 
	\frac{\alpha_k^2}{4} \norm{\xbar[k] - \zbar[k]}^2 
	\ge \frac{\alpha_k^2}{4\L^2} \norm*{\exg[k] - \exm[k]}^2 = \frac{\q[k]^2}{4\L^2} ,
\]
where the last equality is the first time we assumed exact $\gradientOracle{}$. We then show that, for all $S\ge0$,
\[
	\sum_{k=0}^t \norm{x_{k+1}-y_k}^2 \sqrt{\Q[k]} \ge \sum_{k=0}^t \frac{\q[k]^2}{\L^2} \sqrt{\Q[k]} \ge \Omega\prn*{S\sqrt{\Q[t]} - S^{3/2}\L};
	\numberthis \label{eq:noiseless-psketch-smoothness-leverage}
\]
this is a streamlined version of key arguments in~\cite{levy2018online,kavis2019unixgrad} where the authors carefully split the sum above based on the value of the adaptive step size. Taking $S=s\cdot \rbar[t+1](\rbar[t+1]+\d[0])$ and substituting back, we get 
\[
	\mc{R}_t \le \O{ 
	s^{3/2} \L \rbar[t+1]\prn*{ \rbar[t+1] + \d}^2 + \rbar[t+1]\prn*{ \rbar[t+1] + \d} \brk*{ \sqrt{\max\{\Gy[t],\Q[t]\}} - s \sqrt{\Q[t]} }_{+}
	}.
	\numberthis \label{eq:noiseless-psketch-almost-gen-regret-bound}
\]
To conclude the proof, we use the following $\UniXGrad$  ``anytime online-to-batch conversion''~\cite{cutkosky2019anytime} bound:
\[
	f(\xbar[t]) - f(\xopt) \le \sum_{k=0}^{t} \frac{\w[k]}{\sum_{i=0}^t \w[i]} \inner{\exg[k]}{x_{k+1}-\xopt} = \frac{\mc{R}_t}{\sum_{k=0}^{t} \w[k] },
	\numberthis \label{eq:noiseless-psketch-online-to-batch}
\]
where the last equality is the second and final time the proof uses the noiseless gradient assumption. Dividing \cref{eq:noiseless-psketch-almost-gen-regret-bound} by 
\[
\sum_{k=0}^{t} \w[k] 
\overge{\text{\tiny Lem.~\ref{lem: weight sum of alphas}}} \half \rbar[t] \alpha_t^2 
=
\frac{1}{2} \rbar[t] \prn*{ \sum_{k=0}^t \rbar[k] / \rbar[t] }^2
\ge 
\frac{1}{2} \rbar[t+1] \prn*{ \sum_{k=0}^t \rbar[k] / \rbar[t+1] }^2,
\numberthis \label{eq:noiseless-psketch-sum-of-wl}
\]
and employing \eqref{eq:noiseless-psketch-online-to-batch} yields the suboptimality bound~\eqref{eq:noiseless-subopt}.

\subsection{Iterate stability}
In the discussion following \Cref{prop:noiseless-subopt} above, we provisionally imagined that the iterates were bounded ($\rbar[t] \le \diam$ for all $t$) and argued that in this case simply setting $\Gx[t]=\Q[t-1]$ and $\Gy[t]=\Q[t]$ suffices for obtaining optimal rates whenever $\diam=O(\d)$. However, in unconstrained settings this choice of step size is hopeless, as it makes $\eta_{x,0}$ infinite, implying divergence at the first step!\footnote{For constrained domains, however, this choice results in a valid scheme where the first step jumps to the domain boundary. Indeed, \UniXGrad also behaves this way for sufficiently scaled-up instances since it uses a fixed, arbitrary value for $\eta_{x,0}$. This underscores $\UniXGrad$'s strong reliance on the bounded domain assumption.}

In the following proposition, we identify two conditions that together guarantee the iterates remain appropriately bounded. The complete proof appears in \Cref{app:noiseless-stability-proof}.

\begin{proposition}\label{prop:noiseless-stability}
	In the noiseless setting (\Cref{ass:noiseless}), let $s>0$ and define $c_t = 12 \logp^2\prn*{ \frac{s + \Q[t]}{s} }$. If $\reps\le\d$ and the \UDoG step sizes~\eqref{eq:step-size-form} satisfy $(i)$ $\Gy[t] \ge c_{t}^2 (s+\Q[t])$ (with $\Gx[0] \ge 144 s$), and $(ii)$  $\max\{\norm{x_{t+1}-y_t}, \norm{y_{t+1}-x_{t+1}}\} \le \frac{2\rbar[t]}{c_{t-1}}$ for all $t \ge 0$, then we have 
	\begin{align*}
		\dbar[t] \le 2 \d ~~ \text{and} ~~ \rbar[t] \le 4 \d
		~~\mbox{for all }t\ge0
		.
	\end{align*}
\end{proposition}

Let us briefly explain the two requirements in \Cref{prop:noiseless-stability}. Requirement $(i)$ folds two conditions into one. The first is that we increase the \UniXGrad denominator by a logarithmic factor---this is analogous to the step size attenuation necessary to ensure the stability of \DoG (i.e., the \TDoG step size \cite[Section~3.3]{ivgi2023dog}). 
The second is more subtle, requiring that $\Gy[t]$ upper bound $\Q[t]$ (rather than $\Q[t-1]$ as in \UniXGrad and \Cref{prop:noiseless-subopt}) and hence depend on $\norm{\g[t]-\m[t]}$. This is essential for guaranteeing stability but is also the cause of considerable technical difficulty in the noisy setting. Requirement $(ii)$ simply asks that $\UDoG$ iterates at time $t$ move by no more than a fraction of the estimated distance to optimality $\rbar[t]$; a reasonable requirement if the estimate is good. 

The proof of \Cref{prop:noiseless-stability} is a careful application of the \TDoG stability proof \cite[Proposition~2]{ivgi2023dog} to the \UDoG template. The key to the proof is the following modification of the \UniXGrad online-to-batch conversion bound~\eqref{eq:noiseless-psketch-online-to-batch}, which states that for any optimum $\xopt$ we have
\[
	\mc{R}'_t \defeq 
	\sum_{k=0}^t \eta_{y,k} \alpha_k \inner{\g[k]}{x_{k+1} - \xopt} \overeq{(\star)}
	\sum_{k=0}^t \eta_{y,k} \alpha_k \inner{\exg[k]}{x_{k+1} - \xopt} \ge 0,
	\numberthis \label{eq:noiseless-stability-psketch-regret'}
\]
where $(\star)$ holds only in the noiseless setting. We algebraically manipulate $\mc{R}'_t$ similarly to the weighted regret in the proof of \Cref{prop:noiseless-subopt}. Writing $\Q[t]' = c_{t-1}^2 (s+\Q[t])$, we obtain
\[
	0 \le \mc{R}'_t \le \sum_{k=0}^t \prn*{\d[k]^2 - \d[k+1]^2 +
		 \frac{\q[k]\rbar[k]^2}{\sqrt{\Gy[k]\Q[k]'}} 
		+ \frac{\sqrt{\Q[k]'}-\sqrt{\Gx[k]}}{\sqrt{\Gy[k]}}
		\prn*{\norm{x_{k+1}-y_{k}}^2 + \norm{x_{k+1}-y_{k+1}}^2}
	}.
\]
Our requirements $\Gy[k] \ge c_{t}^2 (s+\Q[t])$ (which entails  $\Gx[k] \ge \Gy[k-1] \ge c_{t-1}^2 (s+\Q[t-1])$) and  $\norm{x_{k+1}-y_{k}}^2 + \norm{x_{k+1}-y_{k+1}}^2 \le \frac{8\rbar[k]^2}{c_{k-1}^2}$, allow us, with some more algebra, to bound the last two summands by $\frac{9\q[k]\rbar[k]^2}{c_k (s+\Q[k])}$.
 From here, the proof proceeds identically to the $\TDoG$ analysis~\citep[Section~3.3]{ivgi2023dog}: we get that $\sum_{k=0}^t \frac{9\q[k]^2\rbar[k]}{c_k (s+\Q[k])} \le \frac{\rbar[t]^2}{16}$ 
by the choice of $c_t$, and substituting back obtain that $\d[t+1]^2 \le \d^2 + \frac{\rbar[t]^2}{16}$, which by straightforward induction implies the desired bounds on $\dbar[t]$ and $\rbar[t]$. 

\subsection{Rate of convergence in the noiseless case}\label{subsec:noiseless-rate}
With the conditional stability guarantee of \Cref{prop:sochastic-stability} in place, we are ready to face a central challenge: finding step sizes $\eta_{x,t},\eta_{y,t}$ that satisfy the proposition's conditions but still lead to good rates of convergence in the smooth case. Our solution is (recalling the notation $\M[t] = \max_{k\le t}\{\alpha_k^2 \norm{\m[k]}^2\}$):
\begin{equation}
	\label{eq: step size option 2}
	\begin{aligned}
		\eta_{x,t} &= \frac{\rbar[t]}{ 12 \logp^2\prn*{ \frac{\norm{\m[0]}^2+ \Q[t-1]}{\norm{\m[0]}^2} } \sqrt{ \max\crl*{ \norm{\m[0]}^2+ \Q[t-1], \M[t] } }} \\ %
		\eta_{y,t} &= \frac{\rbar[t]}{ 12 \logp^2\prn*{ \frac{\norm{\m[0]}^2+ \Q[t]}{\norm{\m[0]}^2} } \sqrt{ \max\crl*{ \norm{\m[0]}^2+ \Q[t], \M[t] } }}
		.
	\end{aligned}
\end{equation}

Clearly, the step sizes~\eqref{eq: step size option 2} satisfy the first condition in \Cref{prop:noiseless-stability} with $s=\norm{\m[0]}^2$. To see why the second condition holds, note that, since $\sqrt{\M[t]} \ge \alpha_t \norm{\m[t]}$, we have $\eta_{x,t}\le \frac{\rbar[t]}{c_t \alpha_t \norm{\m[t]}}$. By the contractive property of projections, we therefore have
\[
	\norm{x_{t+1}-y_t} \le \eta_{x,t} \alpha_t \norm{\m[t]} \le \frac{\rbar[t]}{c_t} \le \frac{2\rbar[t]}{c_t}. 
\]
A similar argument also shows that $\norm{x_{t+1}-y_{t+1}}\le \frac{2\rbar[t]}{c_t}$, fulfilling the conditions of \Cref{prop:noiseless-stability} (see \Cref{lem: rbar grwoth}).  

Now the question becomes: how does the introduction of $\M[t]$ into the step size affect suboptimality? In the non-smooth case the effect is minimal, as we anyway bound $\Q[t]$ with $O(\lip^2 t^3)$, and $\M[t] = O(\lip^2 t^2)$ is of a lower order. In the smooth case, however, $\M[t]$ is potentially more harmful, since while \Cref{prop:noiseless-subopt} allows us to cancel the dependence on $\Q[t]$ by setting $s=c_t$, it leaves $\M[t]$ hanging in the numerator, yielding $f(\xbar[t])-f(\xopt) \le \O{\frac{1}{\alpha_t^2}\prn*{c_t^{3/2}\beta \d[0]^2 + c_t\d[0] \sqrt{\M[t]}}}$.  

Fortunately, smoothness allows us to relate $\M[t]$ back to the optimality gap $f(\xbar[t])-f(\xopt)$. In particular, in the unconstrained setting $\kset=\R^n$ we have
\[
\norm{\m[t]}^2 \le 2\norm{\g[t]-\m[t]}^2 + 2\norm{\g[t]}^2 \le 
2\Q[t]/\alpha_t^2 + 4\L \brk*{f(\xbar[t])-f(\xopt)}, 	
\] 
where the last transition used that $\g[t] =\exg[t]$ in the noiseless setting. 
Combining this bound with \Cref{prop:noiseless-subopt}, we obtain
 \[
	f(\xbar[t])-f(\xopt) \le \O{\frac{c_t^{3/2}\beta \d[0]^2 + \sqrt{c_t^2\L\d[0]^2 \max_{k\le t} \alpha_k^2 \brk*{f(\xbar[k])-f(\xopt)}}}{\alpha_t^2}},\]
from which $f(\xbar[t])-f(\xopt) \le\O{\frac{c_t^2 \L \d[0]^2}{\alpha_t^2}}$ follows by induction. Thus
we arrive at our final guarantee in the noiseless case: \Cref{thm:noiseless-main} (see full proof in \Cref{app: noiseless-main}).

\begin{theorem}\label{thm:noiseless-main}
	In the noiseless setting (\Cref{ass:noiseless}) with $\kset = \R^n$ and $\reps \le \d[0]$,  
	using the step sizes \cref{eq: step size option 2}, we get that $\dbar[T] \le 2 \d$, $\rbar[T] \le 4 \d$ and, for $\tau = \argmax_{t < T} \sum_{i \le t} \frac{\rbar[i]}{\rbar[t+1]}$, the suboptimality is
	\begin{align*}
		f\prn{\xbar[\tau]} - f\prn{x\opt}
		\le \O{c_{\reps, T} 
			\min\crl*{ \frac{\beta \d[0]^2}{T^2} , \frac{ \lip \d[0] }{\sqrt{T}}  
			}}
		,
	\end{align*}
	where $c_{\reps, T} = \logp^4\prn*{1 + \frac{T \min\crl*{ \L \d^2, \lip \d } }{f\prn{x_0} - f\prn{x\opt}} } \logp^2\prn*{ \frac{\d}{\reps} }$.
\end{theorem}

\section{Analysis in the stochastic case}\label{sec:stochastic}
In this section, we extend the $\UDoG$ guarantees to the noisy case. We start by assuming that the gradient noise is bounded, a setting that captures most of the remaining technical challenges. We then generalize our results to sub-Gaussian noise by means of a black-box reduction~\cite{attia2023sgd}. Finally, we specialize the $\UDoG$ guarantee for mini-batches of bounded gradient estimates and conclude with a discussion of the (weak) dependence of our result on problem parameter bounds. Throughout this section, we denote the empirical variance at time $t$ by
\begin{equation}
	\empVar[t] \defeq \frac{1}{t+1} \sum_{k=0}^{t} \prn*{\norm{\g - \exg}^2 + \norm{\m - \exm}^2}.
	\label{eq:emp-var-def}
\end{equation}
We also recall the notation
\[
	\theta_{t,\delta} \defeq \log \frac{60 \log \prn*{6t}}{\delta}.
\]

\subsection{Analysis with bounded noise}
We formalize the bounded noise assumption as follows.
\begin{assumption}\label{ass:bounded-noise}
In addition to \Cref{ass:global}, we assume that $\norm{\gradientOracle{x} - \grad f(x)} \le \bfunc(x)$ with probability 1 for all $x\in\kset$, for some (known\footnote{We may view $\bfunc$ as a coarse upper bound on the true noise magnitude, as it only affects low order terms in our bounds.}) function $\bfunc:\kset\to \R_+$.
\end{assumption}
\noindent 
For the iterates of $\UDoG$ we define
\begin{equation}
	\label{eq: b_t def}
	\bk{t} \defeq \bfunc(\xbar[t])~~\mbox{and}~~\bkbar{t} \defeq \max\crl[\Big]{\max_{i\le t} \bk{i}, \bfunc(\zbar[0])}.
\end{equation}
With the assumption and notation in place, we state the stochastic equivalent of \Cref{prop:noiseless-subopt} in the following (see proof in \Cref{app:stochastic-subopt-proof}). 

\newcommand{\noiselessbound}{\text{RHS}_{\cref{eq:noiseless-subopt}}}

\begin{proposition}\label{prop:stochastic-subopt}
	In the bounded noise setting (\Cref{ass:bounded-noise}), suppose the \UDoG step sizes~\eqref{eq:step-size-form} satisfy $\Gx[t] \ge \Q[t-1]$ for every $t\ge 0$. Then for any $\NB > 0$, $T\in\N$, and $\delta\in(0,1)$, with probability at least $1-\delta-\P\brk{ \bkbar{T-1} > \NB }$ we have, for all $t<T$ and $s\ge0$, 
	\begin{align*}
		f\prn{\xbar[t]} - f\prn{x\opt}
		&\le \O{ \noiselessbound + \frac{ \prn{1+s} \prn*{ \rbar[t+1] + \d } \sqrt{ t^3 \theta_{t+1,\delta} \empVar[t] + \prn*{ t \theta_{t+1,\delta} \NB }^2 } }{ \prn*{\sum_{k=0}^{t} \rbar[k]/\rbar[t+1] }^2 } }
	\end{align*}
	where $\noiselessbound = \frac{ s^{3/2}  \L \prn*{ \rbar[t+1] + \d}^2 + \prn*{ \rbar[t+1] + \d} \brk*{ \sqrt{\max\{\Gy[t],\Q[t]\}} - s \sqrt{\Q[t]} }_{+} }{ \prn*{\sum_{k=0}^{t} \rbar[k]/\rbar[t+1] }^2}$ as in \Cref{prop:noiseless-subopt}.
\end{proposition}

\Cref{prop:stochastic-subopt} is a fairly straightforward extension of its noiseless counterpart. The bound~\eqref{eq:noiseless-psketch-smoothness-leverage} continues to hold if we replace $\Q[t]$ with $\minQ[t] = 
\sum_{k=0}^t \alpha_k^2\min\crl{\norm{\g[k]-\m[k]}^2, \norm{\exg[k]-\exm[k]}^2}$. Proceeding as in the proof of \Cref{prop:noiseless-subopt}, we conclude that 
\[
	f\prn{\xbar[t]} - f\prn{x\opt} \le \O{
		\noiselessbound + \tfrac{
			s(\rbar[t+1]+\d[0])\prn*{\Q[t]^{1/2} - 2\minQ[t]^{1/2}}
			+ \sum_{k=0}^{t}\w[k] \inner{\exg[k] - \g[k]}{x_{k+1}-\xopt}
			}{ \prn*{\sum_{k=0}^{t} \rbar[k]/\rbar[t+1] }^2 }
	}.
\]
We show that $\Q[t]^{1/2} - 2\minQ[t]^{1/2} \le \O{\sqrt{t^3 \empVar[t]}}$ by 
straightforward manipulation. Furthermore, using a time-uniform empirical-Bernstein-type concentration bound~\cite{howard2021time,ivgi2023dog} (\Cref{lem: bound on noise mul x dist}) to show that (with the appropriate high probability) the martingale difference sum
 $\sum_{k=0}^{t}\w[k] \inner{\exg[k] - \g[k]}{x_{k+1}-\xopt}$ is bounded by $\O{\rbar[t]\dbar[t+1] \sqrt{ t^3 \theta_{t+1,\delta} \empVar[t] + \prn*{ t \theta_{t+1,\delta} \NB }^2 }}$. 

Next, we extend our iterate stability guarantee to the stochastic setting (see proof in \cref{app:sochastic-stability-proof}).
\begin{proposition}\label{prop:sochastic-stability}
	In the bounded noise setting \Cref{ass:bounded-noise}, let $s>0$, $T\in\N$ and $\delta\in(0,1)$,  and define $c_t = 400 \theta_{T,\delta} \logp^2\prn*{ \frac{s + \Q[t]}{s} }$.
	Suppose that $\reps\le\d$ and the \UDoG step sizes~\eqref{eq:step-size-form} satisfy, with probability 1, for all $t\ge0$: $(i)$ $\Gy[t] \ge c_{t}^2 (s+  \Q[t])$ (with $\Gx[0] \ge 400^2 \theta_{T,\delta}^2 s$), $(ii)$  $\max\{\norm{x_{t+1}-y_t}, \norm{y_{t+1}-x_{t+1}}\} \le \frac{2\rbar[t]}{c_{t-1}}$, $(iii)$ $\sqrt{\Gy[t]} \ge c_t \alpha_t \max\{\norm{\exg[t]-\g[t]},\norm{\exg[t]-\m[t]}\}$, and $(iv)$  $\eta_{y,t}$ is independent of $\g[t]$ given $x_0, \ldots, x_t$. Then, we have with probability of at least $1-\delta$,
	\begin{align*}
		\dbar[t] \le 2 \d ~~ \text{and} ~~ \rbar[t] \le 4 \d
		~~\mbox{for all } t<T
		.
	\end{align*}
\end{proposition}

Conditions $(i)$ and $(ii)$ of \Cref{prop:sochastic-stability} are identical to their noiseless counterparts in \Cref{prop:noiseless-stability}, while conditions $(iii)$ and $(iv)$ are new and facilitate the application of a concentration bound to the weighed regret $\mc{R}_t'$ defined in \cref{eq:noiseless-stability-psketch-regret'}. In particular, the condition $(iv)$ ensures that $\sum_{k=0}^t \eta_{y,k} \alpha_k \inner{\g[k]-\exg[k]}{x_{k+1} - \xopt}$ is a martingale difference sequence, and condition $(iii)$ guarantees boundedness required by our concentration bound (\Cref{lem: uni noise bound}). With this high probability bound in place, the proof continues in the same vein as the noiseless case.

When searching for step sizes meeting the conditions of \Cref{prop:sochastic-stability} we encounter two challenges. First, condition $(iii)$ asks $\Gy[t]$ to be large compared to a quantity depending on the exact gradient $\exg[t]$, which we cannot access directly. We solve it using the bounds given in~\eqref{eq: b_t def}. Simply adding $c_t^2 (t+1)^2 \bkbar{t}^2 \ge c_t^2 \alpha_t^2 \bk{t}^2$ to $\Gy[t]$ guarantees that $\sqrt{\Gy[t]} \ge c_t\alpha_t \norm{\exg[t]-\g[t]}$. Moreover, using $\norm{u}^2 + \norm{v}^2 \ge \half \norm{v+u}^2$, we have
\[
	\norm{\g[t]-\m[t]}^2 + \bkbar{t}^2 \ge \norm{\g[t]-\m[t]}^2 + \norm{\exg[t]-\g[t]}^2 \ge \half  \norm{\exg[t]-\m[t]}^2.
\]
Therefore, taking $\Gy[t] = c_t^2 (s + 2\Q[t] + 2(t+1)^2 \bkbar{t}^2)$ fulfills condition $(iii)$. However, it violates condition $(iv)$ which leads us to the second challenge: how to avoid dependence on $g_t$? 
To address this challenge, we employ the somewhat unusual trick of drawing a \emph{fresh stochastic gradient} $\tg[t] \sim \gradientOracle{\xbar[t]}$ which is, by construction, independent of $\g[t]$ given $\xbar[t]$. We can now replace the forbidden $\norm{\g[t]-\m[t]}$ with the valid upper bound $2 \norm{\tg[t] - \m[t]} + 8 \bkbar{t}$ and thus satisfy conditions $(i)$ and $(iii)$ without violating condition $(iv)$. 

To satisfy condition $(ii)$ we introduce $\M[t]$ to $\Gy[t]$ as done in the noiseless setting and make another modification to ensure the monotonicity required in~\eqref{eq:step-size-form}. Writing,
\[
	\tq[t] \defeq 2\alpha_t^2 \norm{ \tg - \m }^2
	~~\mbox{,}~~
	\maxQ[t] \defeq \sum_{k=0}^{t} \max\crl*{ \q[k], \tq[k] }
	~~\mbox{and}~~
	\noisep[t] \defeq 8 \prn*{t+1}^2 \bkbar{t}^2
	, 
	\numberthis \label{def: qtilde Qbar and pt}
\]
our final step sizes are:
\begin{equation}
	\label{eq: step size option 3}
	\begin{aligned}
		\eta_{x,t} &= \frac{\rbar[t]}{ 400 \theta_{T,\delta} \logp^2\prn*{ 1 +\frac{\noisep[t-1] + \maxQ[t-1]}{\norm{\m[0]}^2 + \noisep[0]} } \sqrt{ \max\crl*{ \norm{\m[0]}^2 + \noisep[0] +  \noisep[t-1] + \maxQ[t-1], \M[t] } }} \\
		\eta_{y,t} &= \frac{\rbar[t]}{ 400 \theta_{T,\delta} \logp^2\prn*{ 1 + \frac{\noisep[t] + \tq[t] + \maxQ[t-1]}{\norm{\m[0]}^2+ \noisep[0]} } \sqrt{ \max\crl*{ \norm{\m[0]}^2 + \noisep[0] + \noisep[t] + \tq[t] + \maxQ[t-1], \M[t] } }}
		.
	\end{aligned}
\end{equation}
Similar to the \TDoG step sizes~\cite[Section~3.3]{ivgi2023dog}, our step sizes depend logarithmically on the desired confidence level $\delta$ and double-logarithmically on the maximum iteration budget $T$.

With all the pieces in place, we now state our main result (see proof in \Cref{app: sochastic-main}).
\begin{theorem}\label{thm:sochastic-main}
	In the bounded noise setting (\Cref{ass:bounded-noise}) with $\kset = \R^n$, for any $T\in\N$ and $\delta\in(0,\frac{1}{5})$, consider \UDoG with step sizes~\eqref{eq: step size option 3}. With probability at least $1-5\delta$, we have $\dbar[T] \le 2 \d$, $\rbar[T] \le 4 \d$ and for  $\tau = \argmax_{t < T} \sum_{i \le t} \frac{\rbar[i]}{\rbar[t+1]}$ and $\bstar \defeq \max_{x:\norm{x-x\opt} \le 2\d[0]}\crl*{ \bfunc(x) }$ we have
	\begin{equation*}
		f\prn{\xbar[\tau]} - f\prn{x\opt} \le \O{ 
			c_{\delta,\reps, T} 
			\prn*{ 
				\min\crl*{ \frac{\beta \d[0]^2}{T^2} , \frac{ \lip \d[0] }{\sqrt{T}}} 
				+ \frac{\d[0] \sqrt{\empVar[T]}}{\sqrt{T}} + \frac{\d[0] \bstar}{T}
			}
		}
		,\numberthis \label{eq:stochastic-main-bound}
	\end{equation*}
	where 
		$c_{\delta,\reps, T} =
		\log^2\prn*{\frac{\logp \prn*{T}}{\delta} } \logp^{4}\prn*{ 1 + T \frac{\bstar + \min\crl*{ \L\d^2 , \lip\d }}{ f(x_0)-f(\xopt) } } \logp^{2}\prn*{ \frac{\d}{\reps} }$ and $\empVar[t]$, defined in \cref{eq:emp-var-def}, is the empirical noise variance.
\end{theorem}
\noindent
We remark that under our assumptions it is straightforward to replace the empirical variance $\empVar[t]$ in~\cref{eq:stochastic-main-bound} with its expectation without altering other non-logarithmic terms in the bound, e.g., via Hoeffding's inequality.

\subsection{From bounded to sub-Gaussian noise}
The bounded noise assumption makes analysis convenient but is not entirely satisfactory since averaging independent bounded-noise estimators does not reduce the probability 1 noise bound, preventing us from making statements about mini-batch scaling. To address this issue, we consider the following standard assumption.

\begin{assumption}\label{ass:sub-guassian-noise}
	In addition to \Cref{ass:global}, we assume that $\norm{\gradientOracle{x} - \grad f(x)}$ is $\sigma^2(x)$-sub-Gaussian for all $x\in\kset$, for some (known) $\sigma:\kset\to \R_+$. That is, \[\Pr*(\norm{\gradientOracle{x} - \grad f(x)} \ge z) \le 2 \exp\prn*{ -{z^2}/{\sigma^2(x)} }\] for all $z\ge 0$ and $x\in\kset$.  
\end{assumption}

To move from bounded to sub-Gaussian we utilize a reduction due to~\citet{attia2023sgd} that allows us to essentially replace $\bfunc(\cdot)$ with $\sigma(\cdot)$ in \Cref{thm:sochastic-main} at the cost of additional logarithmic factors. 
To that end, we define $\skbar{t} \defeq \max\crl*{\max_{i\le t} \sigma(\xbar[k]), \sigma(\zbar[0])}$, 
as well as $\sstar \defeq \max_{x:\norm{x-x\opt} \le 2\d[0]} \sfunc(x)$ and $\stob_{t,\delta} \defeq 3 {\log^{1/2} \prn[\big]{ \frac{15\prn*{t+1}^2}{\delta} }}$. With this notation in hand, we state our guarantee for the sub-Gaussian setting (see proof in \Cref{app: subGuassian-noise}).

\begin{corollary}\label{coro:subGuassian-noise} 
	Consider the sub-Gaussian noise setting (\Cref{ass:sub-guassian-noise}) with $\kset = \R^n$ and $\delta\in(0,\frac{1}{6})$,
	using the step sizes~\eqref{eq: step size option 3} with $\bkbar{t} = \skbar{t} \stob_{t,\delta}$, then with probability at least $1-6\delta$ we get that 
	$\dbar[T] \le 2 \d$, $\rbar[T] \le 4 \d$, and the suboptimality bound~\eqref{eq:stochastic-main-bound} holds for $\bstar = \sstar \stob_{T-1,\delta}$.
\end{corollary}

\subsection{Corollary: mini-batch of bounded noise}
Finally, we leverage our result for sub-Gaussian noise to demonstrate that \UDoG automatically benefits from increasing mini-batch size (see proof in \Cref{app: mini-batch}).

\begin{assumption}\label{ass:minibatch}
	In addition to \Cref{ass:global}, we assume that $\gradientOracle{x}$ is the average of $B$ unbiased estimates of $\grad f(x)$, each bounded by $\lip$ with a known upper bound $\hat{\lip}\ge \lip$. 
\end{assumption}

\begin{corollary}
	\label{coro:mini-batch}
	In the mini-batch setting (\Cref{ass:minibatch}) 
	with $\kset = \R^n$, for any $T\in\N$ and $\delta\in(0,\frac{1}{6})$, consider \UDoG with step sizes~\eqref{eq: step size option 3} where $\bkbar{t} = \sqrt{2}\frac{\hat{\lip}}{\sqrt{B}}\stob_{t,\delta}$.
	With probability at least $1-6\delta$ we have $\dbar[T] \le 2 \d$, $\rbar[T] \le 4 \d$ and, for  $\tau = \argmax_{t < T} \sum_{i \le t} \frac{\rbar[i]}{\rbar[t+1]}$, %
	\begin{equation*}
		f\prn{\xbar[\tau]} - f\prn{x\opt} \le \O{ 
			c_{\delta,\reps, T} 
			\prn*{ 
				\frac{\beta \d[0]^2}{T^2}
				+ \frac{\prn*{ \lip + \hat{\lip} / \sqrt{T} } \d[0]}{\sqrt{T B}}
			}
		}
		,
	\end{equation*}
	where
		$c_{\delta,\reps, T} = \sqrt{\logp \prn*{ \frac{T}{\delta} }}
		\log^2\prn*{\frac{\logp \prn*{T}}{\delta} } \logp^{4}\prn*{ 1 + T \frac{\hat{\lip} + \min\crl*{ \L\d^2 , \lip\d }}{ f(x_0)-f(\xopt) } } \logp^{2}\prn*{ \frac{\d}{\reps} }$.
\end{corollary}

\subsection{Discussion: how parameter-free is our algorithm?}
With our results established, we now discuss in more detail the extent to which our algorithms and complexity bounds are free of a-priori knowledge of problem parameters. \UDoG requires a lower bound $\reps$ on the initial distance to the optimum $\d[0]$, and pointwise upper bounds $\bk{t}$ on the noise magnitude at each iteration. \Cref{thm:sochastic-main} provides suboptimality bounds that depend poly-logarithmically on $\frac{\d[0]}{\reps} $ which quantifies how $\reps$ underestimates $\d[0]$. Many works \citeg{bhaskara2020online,carmon2022making,cutkosky2019artificial,cutkosky2018black,ivgi2023dog,jacobsen2022parameter,mcmahan2014unconstrained,mhammedi2020lipschitz,orabona2016coin} treat such logarithmic dependence as the \emph{definition} of a parameter-free algorithm, and in that strict sense our method is certainly parameter-free. The noise bounds impact our suboptimality guarantees polynomially via the additive term $\bstar/T$ where  $\bstar \defeq \max_{x:\norm{x-x\opt} \le 2\d[0]}\crl*{ \bfunc(x) }$, potentially implying greater sensitivity to problem parameters. Neverthless, we argue that our method fully deserve the title ``parameter-free'' for the following reasons.
\begin{enumerate}[leftmargin=*] 
	\item \textbf{The noise bound only contributes a low-order error term. }
	To see why $\bstar / T$ is low-order, let $\btbar{T}$ be the largest stochastic gradient error in the first $T$ iterations of $\UDoG$.
	Then the empirical variance satisfies $\empVar[T] = \O{\btbar{T}^2}$ and the noise-dependent part of \Cref{thm:sochastic-main} is $\Otil{ \frac{\d[0] \btbar{T}}{\sqrt{T}} + \frac{\d[0] \bstar}{T} }$.
	Therefore, as long as $\bstar / \btbar{T} = O(\sqrt{T})$ (i.e. the noise bound is loose by $\sqrt{T}$ or less) we get the near-optimal dependence on the unknown, true noise magnitude $\btbar{T}$. 
	
	\item \textbf{The low-order term and noise bound assumption are unavoidable.}
	Recent work \cite[Theorem 6]{attia2024free} proves that \emph{any} algorithm with logarithmic dependence on uncertainty in distance to optimality \emph{must} suffer the low-order error term $\bstar / T$, and hence also require an a-priori noise bound (concurrent work \cite{khaled2024tuning,carmon2024price} also shows similar results). Moreover, prior parameter-free algorithms
	assume known bounds on stochastic gradient magnitude, which is stronger than assuming noise bounds.
	In this sense, our method is as parameter-free as it gets. 
	
	\item \textbf{The noise bound is often easy to obtain and vanishes as batch size grows.} \Cref{coro:mini-batch} and \Cref{ass:minibatch} give a general setting where a noise bound is readily available.
	For a concrete instantiation, consider logistic regression with normalized covariates.
	In this case $\hat{\lip}=1$ and the noise bound at batch size $B$ is $\Otil{1/B}$, which decreases as the batch size grows.
\end{enumerate}

\section{Experiments}\label{sec:experiments}

\newcommand{\lscaption}[2]{Training a linear model with ViT-B/32 features and least-squares loss on {#1}. Top: Train loss. Bottom: Test accuracy after iterate averaging. First column: Batch size scaling of complexity to reach target performance. Second column: Learning curves. Third column: ASGD performance at all learning rates and momenta, contrasted with \DoG variants.}

\newcommand{\convexcaption}[2]{Training a linear model with ViT-B/32 features and log loss on {#1}. Top: Train loss. Bottom: Test accuracy after iterate averaging. First column: Batch size scaling of complexity to reach target performance. Second column: Learning curves. Third column: ASGD performance at all learning rates and momenta, contrasted with \DoG variants.}

\newcommand{\lscaptioncurves}[2]{Training a linear model with ResNet50 features and least-squares loss on {#1}. Top: Loss vs. batches processed training for different batch sizes. Bottom: Test accuracy of averaged model vs. batches processed for different batch sizes. 
}

\newcommand{\convexcaptioncurves}[2]{Training a linear model with ViT-B/32 features and log loss on {#1}. 
Top: Loss vs. batches processed training for different batch sizes. Bottom: Test accuracy of averaged model vs. batches processed for different batch sizes. 
}

\newcommand{\libsvmcaption}[2]{Training a linear model with log loss on {#1}. Top: Train loss. Bottom: Test accuracy after iterate averaging. First column: Batch size scaling of complexity to reach target performance. Second column: Learning curves. Third column: ASGD performance at all learning rates and momenta, contrasted with \DoG variants.}

\newcommand{\libsvmcaptioncurves}[2]{Training a linear model with log loss on {#1}. 
Top: Loss vs. batches processed training for different batch sizes. Bottom: Test accuracy of averaged model vs. batches processed for different batch sizes. 
}

\newcommand{\nonconvexcaptioncurves}[2]{Training a {#1} model from scratch on  {#2}. Top: Loss vs. batches processed training for different batch sizes. Bottom: Test accuracy vs. batches processed for averaged iterates at varied batch sizes.}

\insertfigure{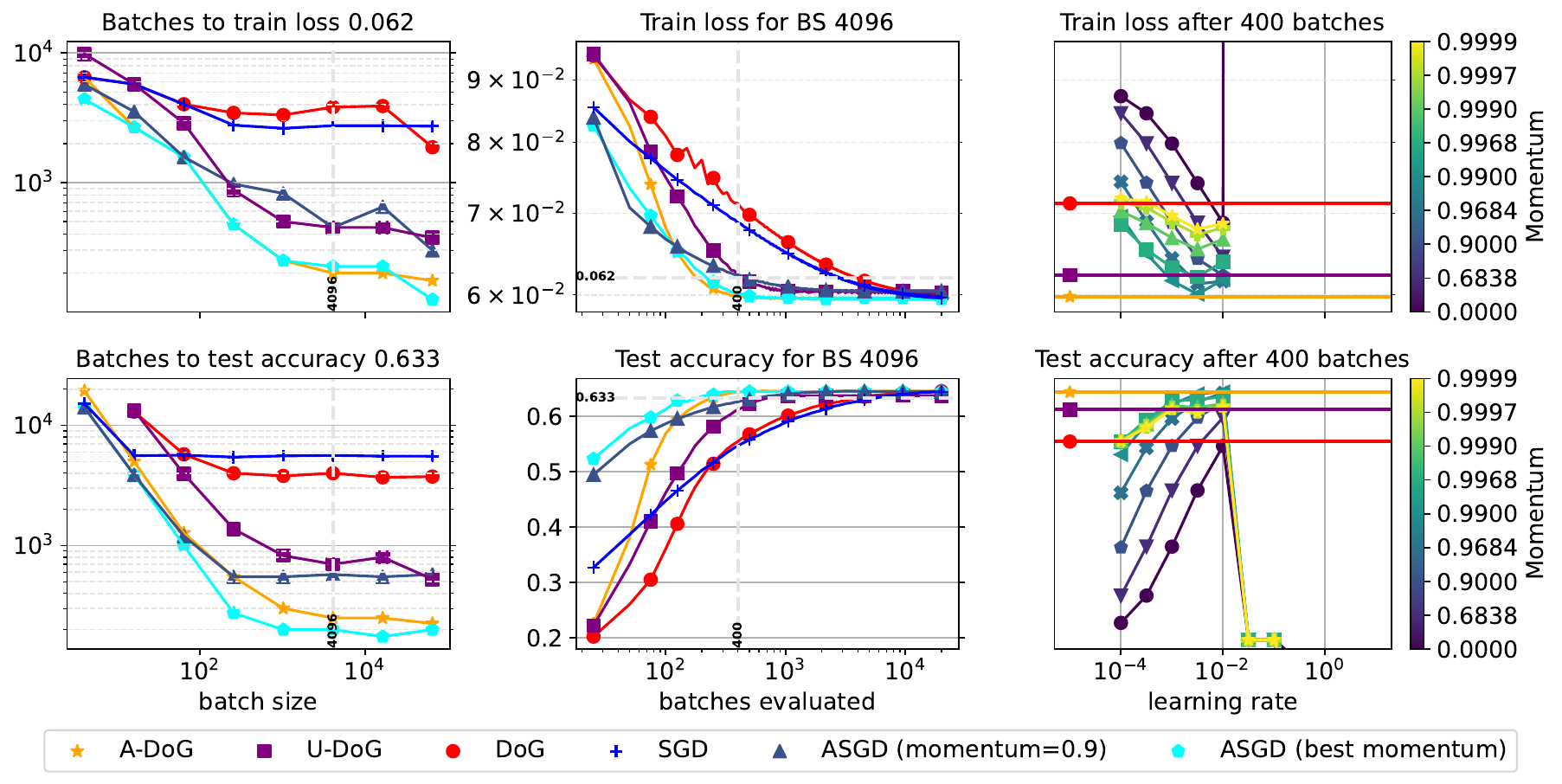}{\lscaption{SVHN}{90}}{fig:main-results}

We test \UDoG on a suite of experiments on convex and non-convex learning problems.  
We also heuristically derive and experiment with an algorithm we call \ADoG, which integrates ideas from \AcceleGrad~\citep{levy2018online} and \DoG. Namely, it uses the \AcceleGrad step with $\DoG$ numerator and $\alpha_t$ as in \UDoG. The pseudocode for \ADoG is given in \Cref{alg: accelegrad dog} in \Cref{app:experiments-accelegrad-dog}.

We compare our algorithms to \DoG as well as carefully tuned SGD with constant Nesterov momentum (ASGD for short) across a wide range of batch sizes. Detailed experimental results and analyses, as well as implementation details, are presented in \Cref{app:experimental-details}.

Our testbed consists of multiple classification problems based on the VTAB benchmark~\citep{zhai2019large} and libsvm datasets~\citep{libsvm}, which we solve with both multiclass log loss and least squares loss, as well as a synthetic noiseless linear regression problem (see \Cref{app-subsec:convex}). In addition, we perform preliminary experiments in the non-convex setting, including training neural networks from scratch on CIFAR-10 and VTAB datasets, and fine-tuning a CLIP model on ImageNet (see \Cref{app-subsec:non-convex}). 

On convex optimization problems, both $\UDoG$ and $\ADoG$ often substantially improve over $\DoG$, with $\ADoG$ achieving results comparable to well-tuned ASGD and outperforming $\UDoG$, likely by avoiding extra-gradient computations. \Cref{fig:main-results} illustrates these results on a particular dataset and least-squares loss function configuration and \Cref{app-subsec:convex} repeats this figure for additional configurations. The left panels in the figure show the rate of convergence of $\ADoG$, $\UDoG$ and ASGD plateaus at a larger batch size compared to \DoG and SGD without momentum. This is the typical effect of acceleration in stochastic optimization~\citep{shallue2019measuring}, and is also supported by \Cref{coro:mini-batch} which shows that, for sufficiently large batch size, $\UDoG$ converges at rate scaling as $1/T^2$. In contrast, non-accelerated methods like $\DoG$ and SGD converge with rate scaling as $1/T$. The right panels of the figure show that, at a tight computational budget, the performance of ASGD is very sensitive to the tuning of both step size and momentum, with only the very best values matching the performance of \ADoG. When using logarithmic instead of least-squares loss, the test accuracy becomes more robust to large step size choices (see \Cref{fig:convex-svhn} in the appendix). This is partly because the log loss is Lispchitz which prevents complete divergence at any fixed step size.

In our preliminary non-convex experiments on neural network models (reported in detail in \Cref{app-subsec:convex,app-subsec:non-convex}), we find that $\UDoG$ often fails to converge to competitive results, while $\ADoG$ is competitive with \DoG on most VTAB tasks, but under-performs it for CIFAR-10 and ImageNet fine-tuning, indicating that it is not a yet a viable general-purpose neural network optimizer.

\arxiv{
\newcommand{\acks}[1]{
	\section*{Acknowledgments}
	#1
}
\newpage
}

\acks{
We thank Konstantin Mishchenko for helpful discussion.
This work was supported by the NSF-BSF program, under NSF grant \#2239527 and BSF grant \#2022663. MI acknowledges support from the Israeli Council of Higher Education. OH acknowledges support from Pitt Momentum Funds, and AFOSR grant \#FA955023-1-0242. YC acknowledges support from the Israeli Science Foundation (ISF) grant no. 2486/21 and the Alon Fellowship.
}

\notarxiv{\newpage}
\arxiv{
	\bibliographystyle{abbrvnat}
}

\appendix

\notarxiv{
\crefalias{section}{appendix} %
}

\newpage
\tableofcontents
\newpage

\section{Proof for \Cref{sec:noiseless} (the noiseless setting)}

\subsection{Proof of \Cref{prop:noiseless-subopt}}
\label{app:noiseless-subopt-proof}

\begin{proof}
	Define
	\begin{align*}
		\rho_{t} &\defeq \frac{1}{\sqrt{\Q[t]}}~\text{and}\\
		\minQ[t] &\defeq \sum_{k=0}^{t} \alpha_k^2 \min\crl*{ \norm{ \exg[k] - \exm[k] }^2 , \norm{ \g[k] - \m[k] }^2 }
		.
	\end{align*}
	Note that in the noiseless setting $\minQ = \Q$.
	However, most of the proof carries over to the noisy setting as well.
	Therefore, until a later stage of the proof, we do not use that $\m[t]=\exm[t]$, $\g[t]=\exg[t]$ and $\minQ[t] = \Q[t]$ in the noiseless setting.
	
	Recall the notation $\etat_{x,t} = \frac{1}{\sqrt{\Gx[t]}}$ and $ 
	\etat_{y,t} = \frac{1}{\sqrt{\Gy[t]}}$. 
	Algebraic manipulation gives us that for all $k\ge0$
	\begin{align*}
		\rbar[k] \alpha_k \inner{\g[k] }{ x_{k+1}-\xopt}
		&\le \frac{\rbar[k]^2 \alpha_k^2 \rho_k}{2} \norm{\g[k] - \m[k]}^2 - \sum_{k=0}^{t} \frac{1}{2\rho_k}\norm{x_{k+1} - y_{k}}^2\\
		&\qquad+ \prn*{ \frac{1}{2\rho_k} - \frac{1}{2\etat_{x,k}} } \prn*{ \norm{x_{k+1} - y_{k}}^2 + \norm{x_{k+1}-y_{k+1}}^2 }\\
		&\qquad+ \frac{1}{2\etat_{y,k}} \prn*{ \norm{\xopt - y_{k}}^2 - \norm{\xopt - y_{k+1}}^2 };
	\end{align*}
	see \Cref{lem: unixgrad inequality} for a proof.
	Therefore, by summing over both sides of the inequality we get that for all $t\ge0$
	\begin{align*}
		\sum_{k=0}^{t}\rbar[k] \alpha_k \inner{\g[k] }{ x_{k+1}-\xopt}
		&\le \underbrace{\frac{\rbar[t]^2}{2} \sum_{k=0}^{t} \frac{\alpha_k^2 \norm{\g[k] - \m[k]}^2}{ \sqrt{ \sum_{j=0}^{k} \alpha_j^2 \norm{\g[j] - \m[j]}^2 } }}_{(A)}
		\underbrace{- \sum_{k=0}^{t} \frac{1}{2\rho_k}\norm{x_{k+1} - y_{k}}^2}_{(B)}\\
		&\qquad+ \underbrace{4\rbar[t+1]^2\sum_{k=0}^{t} \brk*{ \frac{1}{\rho_k} - \frac{1}{\etat_{x,k}} }_+}_{(C)}
		+ \underbrace{\frac{1}{2}\sum_{k=0}^{t} \frac{1}{\etat_{y,k}} \prn*{\d[k]^2 -\d[k+1]^2 }}_{(D)}
		.
	\end{align*}
	
	\paragraph{Bounding $(A)$:}
	We have $\sum_{k=0}^{t} \frac{\alpha_k^2 \norm{\g[k] - \m[k]}^2}{ \sqrt{ \sum_{j=0}^{k} \alpha_j^2 \norm{\g[j] - \m[j]}^2 } } \le 2\sqrt{ \sum_{k=0}^{t} \alpha_k^2 \norm{\g[k] - \m[k]}^2 }$; see \Cref{lem: sqrt inequalities} with $s_k = \alpha_k^2 \norm{\g[k]-\m[k]}^2$, and therefore
	\begin{align*}
		\frac{\rbar[t]^2}{2} \sum_{k=0}^{t} \frac{\alpha_k^2 \norm{\g[k] - \m[k]}^2}{ \sqrt{ \sum_{j=0}^{k} \alpha_j^2 \norm{\g[j] - \m[j]}^2 } }
		= \frac{\rbar[t]^2}{ \rho_t}
		.
	\end{align*}
	
	\paragraph{Bounding $(B)$:}
	We have that for all $k\ge0$
	\begin{align*}
		\norm{\exg[k] - \exm[k]}^2
		&\overset{(1)}{\le} \L^2 \norm{\xbar[k] - \zbar[k]}^2\\
		&= \frac{\L^2 \rbar[k]^2 \alpha_k^2}{\prn*{\sum_{0=1}^{k}\rbar[i]\alpha_i}^2}\norm{x_{k+1} - y_{k}}^2\\
		&\overset{(2)}{\le} \frac{ 4\L^2 }{ \alpha_k^2 }\norm{x_{k+1} - y_{k}}^2
		,
	\end{align*}
	where $(1)$ is from the $\L$-smoothness of $f$, and $(2)$ is because $\rbar[k] \alpha_k^2 \le 2 \sum_{0=1}^{k}\rbar[i]\alpha_i$ by \Cref{lem: weight sum of alphas} .
	Therefore,
	\begin{align*}
		-\norm{x_{k+1} - y_{k}}^2
		&\le -\frac{ \alpha_k^2\norm{\exg[k] - \exm[k]}^2 }{ 4\L^2  } 
		.
	\end{align*}
	Thus,
	\begin{align*}
		- \sum_{k=0}^{t} \frac{1}{2\rho_k}\norm{x_{k+1} - y_{k}}^2
		\le - \sum_{k=0}^{t} \frac{ \alpha_k^2\norm{\exg[k] - \exm[k]}^2 }{8 \L^2 \rho_k} 
	\end{align*}
	
	\paragraph{Bounding $(C)$:}
	As $\frac{1}{\rho_k} - \frac{1}{\etat_{x,k}}$ is not necessarily non-negative for all $k\in\crl*{0,\dots,t}$ we define the set of indices for which it is non-negative as
	\begin{align*}
		I \triangleq \crl*{ k \in \crl*{0,1,\dots,t} ~:~ \frac{1}{\rho_t} - \frac{1}{\etat_{x,t}} \ge 0 }
		.
	\end{align*}
	Define $i_k$ as the $k$-th smallest index in $I$, and define $i_{\abs{I}+1} \defeq t+1$.
	We note that for all $k \in I$ then $i_k \le i_{k+1} - 1 \le t$.
	Therefore,
	\begin{align*}
		(C)
		&= 4\rbar[t+1]^2\sum_{k=1}^{\abs{I}} \prn*{ \frac{1}{\rho_{i_k}} - \frac{1}{\etat_{x,i_k}} } %
		\le 4\rbar[t+1]^2\sum_{k=1}^{\abs{I}} \prn*{ \frac{1}{\rho_{\brk{i_{k+1}-1}}} - \frac{1}{\etat_{x,i_k}} }\\
		&\le \frac{4\rbar[t+1]^2}{\rho_{t}} + 4\rbar[t+1]^2\sum_{k=2}^{\abs{I}} \prn*{ \frac{1}{\rho_{\brk{i_{k}-1}}} - \frac{1}{\etat_{x,i_k}} }%
		\le \frac{4\rbar[t+1]^2}{\rho_{t}} + 4\rbar[t+1]^2\sum_{k=0}^{t-1} \brk*{ \frac{1}{\rho_{k}} - \frac{1}{\etat_{x,k+1}}}_{+}
		.
	\end{align*}
	
	\paragraph{Bounding $(D)$:}
	\begin{align*}
		\frac{1}{2}\sum_{k=0}^{t} \frac{1}{\etat_{y,t}} \prn*{\d[k]^2 -\d[k+1]^2 }
		&= \frac{d_{0}^2}{2\etat_{y,0}} - \frac{d_{t+1}^2}{2\etat_{y,t}} + \frac{1}{2}\sum_{k=0}^{t} \prn*{ \frac{1}{\etat_{y,k}} - \frac{1}{\etat_{y,k-1}} }\d[k]^2\\
		&\le \frac{d_{0}^2}{2\etat_{y,0}} - \frac{d_{t+1}^2}{2\etat_{y,t}} + \frac{\dbar[t]^2}{2}\sum_{k=0}^{t} \prn*{ \frac{1}{\etat_{y,k}} - \frac{1}{\etat_{y,k-1}} }
		.
	\end{align*}
	By performing telescopic summation we obtain
	\begin{align*}
		(D)
		&\le \frac{\dbar[t+1]^2 -\d[t+1]^2}{2\etat_{y,t}}
		.
	\end{align*}
	Let $s\in \argmax_{k\le t+1} \d[k]$, we have that $\dbar[t+1]^2 -\d[t+1]^2 = \dbar[s]^2 -\d[t+1]^2 = \prn*{ \dbar[s] -\d[t+1]} \prn*{ \dbar[s] +\d[t+1] } \le \norm{ y_s - y_{t+1} } \prn*{ \dbar[s] +\d[t+1] } \le \prn*{ \rbar[s] + r_{t+1} } \prn*{ \dbar[s] +\d[t+1] } \le 4 \rbar[t+1] \dbar[t+1]$.
	Thus,
	\begin{align*}
		(D)
		&\le \frac{ 2\rbar[t+1] \dbar[t+1] }{\etat_{y,t}}
		.
	\end{align*}
	
	\paragraph{Bounding $(A)+(B)+(C)+(D)$:}
	Combining all of the above, we obtain that
	\begin{align*}
		\sum_{k=0}^{t}\rbar[k] \alpha_k \inner{\g[k] }{ x_{k+1}-\xopt}
		&\le 5\rbar[t+1] \prn*{ \rbar[t+1] + \dbar[t+1] } \max \crl*{ \frac{1}{ \rho_t} , \frac{1}{\etat_{y,t}}}\\
		&\qquad+ 4\rbar[t+1]^2\sum_{k=0}^{t-1} \brk*{ \frac{1}{\rho_{k}} - \frac{1}{\etat_{x,k+1}}}_{+}
		- \sum_{k=0}^{t} \frac{ \alpha_k^2\norm{\exg[k] - \exm[k]}^2 }{8 \L^2 \rho_k} 
		.
	\end{align*}
	Therefore, as for any we have that $\Gx[k] \ge \Q[k-1]$,
	\begin{align*}
		\sum_{k=0}^{t}\rbar[k] \alpha_k \inner{\g[k] }{ x_{k+1}-\xopt}
		&\le 5\rbar[t+1] \prn*{ \rbar[t+1] + \dbar[t+1] } \sqrt{ \max \crl*{ \Gy[t] , \Q[t] } }
		- \sum_{k=0}^{t} \frac{ \alpha_k^2\norm{\exg[k] - \exm[k]}^2 }{8 \L^2 \rho_k} 
		.
	\end{align*}
	Let $s\ge0$ and recall that $\frac{1}{\rho_k} = \sqrt{\Q[k]}$. We get that
	\begin{align*}
		\sum_{k=0}^{t}\rbar[k] \alpha_k \inner{\g[k] }{ x_{k+1}-\xopt}
		&\le 10 s \rbar[t+1] \prn*{ \rbar[t+1] + \dbar[t+1] } \sqrt{ \minQ[t] } - \sum_{k=0}^{t} \frac{ \alpha_k^2\norm{\exg[k] - \exm[k]}^2 }{8 \L^2 }\sqrt{ \Q[k] }  \\
		&\qquad + 5\rbar[t+1] \prn*{ \rbar[t+1] + \dbar[t+1] } \prn*{ s \sqrt{ \Q[t] } - 2s \sqrt{ \minQ[t] } }\\
		&\qquad + 5\rbar[t+1] \prn*{ \rbar[t+1] + \dbar[t+1] } \prn*{ \sqrt{ \max \crl*{ \Gy[t] , \Q[t] } } - s \sqrt{ \Q[t] } } \numberthis \label{eq: partial regret bound}
		.
	\end{align*}
	We have that
	\begin{align*}
		10 s \rbar[t+1] & \prn*{ \rbar[t+1] + \dbar[t+1] } \sqrt{ \minQ[t] } - \sum_{k=0}^{t} \frac{\alpha_t^2\norm{\exg - \exm}^2}{8\L^2 } \sqrt{ \Q[k] } \\
		&\le 10 s \rbar[t+1] \prn*{ \rbar[t+1] + \dbar[t+1] } \sqrt{ \minQ[t] } - \sum_{k=0}^{t} \frac{ \alpha_k^2 \min\crl*{ \norm{ \exg[k] - \exm[k] }^2 , \norm{ \g[k] - \m[k] }^2 } }{8\L^2 } \sqrt{ \minQ[k] }
		.
	\end{align*}
	Define $B_k^2= \alpha_k^2 \min\crl*{ \norm{ \exg[k] - \exm[k] }^2 , \norm{ \g[k] - \m[k] }^2 }$, $c_1 = 10s\rbar[t+1](\rbar[t+1]+\dbar[t+1])$, and $c_2 = 8\L^2$.
	\Cref{lem: Bound on sqrt sum B} gives us that for all $t\ge0$
	\begin{align*}
		c_1 \sqrt{\sum_{k=0}^{t} B_k^2 } - \sum_{k=0}^{t} \frac{B_k^2}{c_2} \sqrt{\sum_{j=0}^{k} B_j^2 }
		&\le 2 c_1^{3/2} c_2^{1/2}
		.
	\end{align*}
	Therefore,
	\begin{align*}
		10 s \rbar[t+1] & \prn*{ \rbar[t+1] + \dbar[t+1] } \sqrt{ \minQ[t] } - \sum_{k=0}^{t} \frac{\alpha_t^2\norm{\exg - \exm}^2}{8\L^2 } \sqrt{ \Q[k] }
		\\
		&
		\le 2\prn*{10s \rbar[t+1] \prn*{ \rbar[t+1] + \dbar[t+1] }}^{3/2} (8\L)^{1/2}
		\le 180 s^{3/2} \rbar[t+1] \prn*{ \rbar[t+1] + \dbar[t+1] }^2 \L
		.
	\end{align*}
	Combining this result with \cref{eq: partial regret bound} yields that for all $t\ge0$ and $s\ge0$
	\begin{align*}
		\sum_{k=0}^{t}\rbar[k] \alpha_k \inner{\g[k] }{ x_{k+1}-\xopt}
		&\le 180 s^{3/2} \rbar[t+1] \prn*{ \rbar[t+1] + \dbar[t+1] }^2 \L\\
		&\qquad + 5\rbar[t+1] \prn*{ \rbar[t+1] + \dbar[t+1] } \prn*{ s \sqrt{ \Q[t] } - 2s \sqrt{ \minQ[t] } }\\
		&\qquad + 5\rbar[t+1] \prn*{ \rbar[t+1] + \dbar[t+1] } \prn*{ \sqrt{ \max \crl*{ \Gy[t] , \Q[t] } } - s \sqrt{ \Q[t] } } \numberthis \label{eq: general regret bound}
		.
	\end{align*}
	\Cref{lem: sub-optimality inequality} gives us that 
	\begin{align*}
		f\prn*{ \xbar[t] } - f(\xopt)
		&\le  \frac{ 1 }{ \sum_{k=0}^{t} \rbar[k] \alpha_k }  \sum_{k=0}^{t}\rbar[k] \alpha_k \inner{\exg[k] }{ x_{k+1}-\xopt}
		.
	\end{align*}
	Now, by additionally using the fact that in the noiseless setting 
	\begin{align*}
		\minQ[t] &= \Q[t] ~~\text{and}\\
		\sum_{k=0}^{t}\rbar[k] \alpha_k \inner{\exg[k] }{ x_{k+1}-\xopt}
		&= \sum_{k=0}^{t}\rbar[k] \alpha_k \inner{\g[k] }{ x_{k+1}-\xopt}
	\end{align*}
	we get that
	\begin{align*}
		f\prn*{ \xbar[t] } - f(\xopt)
		&\le 180 s^{3/2} \frac{ \rbar[t+1] }{ \sum_{k=0}^{t} \rbar[k] \alpha_k } \L \prn*{ \rbar[t+1] + \dbar[t+1] }^2\\
		&+ 5 \frac{ \rbar[t+1] }{ \sum_{k=0}^{t} \rbar[k] \alpha_k } \prn*{ \rbar[t+1] + \dbar[t+1] } \prn*{ \sqrt{ \max \crl*{ \Gy[t] , \Q[t] } } - s \sqrt{ \Q[t] } }
		.
	\end{align*}
	Finally, by using the fact that $\dbar[t+1] \le \d[0] + \rbar[t+1]$ and because $\rbar[k] \alpha_k^2 \le 2 \sum_{0=1}^{k}\rbar[i]\alpha_i$ for all $k\ge0$ (\Cref{lem: weight sum of alphas}), we obtain that
	\begin{align*}
		f\prn{\xbar[t]} - f(\xopt)
		\le \O{\frac{ s^{3/2} \L \prn*{ \rbar[t+1] + \d}^2 + \prn*{ \rbar[t+1] + \d} \brk*{ \sqrt{ \max \crl*{ \Gy[t] , \Q[t] } } - s \sqrt{\Q[t]} }_{+}}{ \prn*{\sum_{k=0}^{t} \rbar[k]/\rbar[t+1] }^2 } }.
	\end{align*}
\end{proof}

\subsection{Proof of \Cref{prop:noiseless-stability}}
\label{app:noiseless-stability-proof}

\begin{proof}
	For any $h>0$ (in this case $h=12$), define
	\begin{align*}
		c_t &\defeq h \logp^2\prn*{ \frac{s + \Q[t]}{s} }~~\text{and}\\
		\rho_t &\defeq \frac{1}{c_{t-1} \sqrt{s + \Q[t]}} 
		.
	\end{align*}
	
	\Cref{lem: unixgrad inequality} gives us that, for all $t\ge0$,
	\begin{align*}
		\rbar[t] \alpha_t \inner{\g }{ x_{t+1}-\xopt}
		&\le \frac{\rbar[t]^2 \alpha_t^2 \rho_t}{2} \norm{\g-\m}^2
		+ \prn*{ \frac{1}{2\rho_t} - \frac{1}{2\etat_{x,t}} } \prn*{ \norm{x_{t+1} - y_{t}}^2 + \norm{x_{t+1}-y_{t+1}}^2 }\\
		&\qquad+ \frac{1}{2\etat_{y,t}} \prn*{ \norm{\xopt - y_{t}}^2 - \norm{\xopt - y_{t+1}}^2 }
		.
	\end{align*}
	From the definitions  of $\rho_t$ and $\etat_{x,t}=1/\sqrt{\Gx[t]}\le 1/\sqrt{\Gy[t-1]}\le 1/\rho_{t-1}$ we obtain that
	\begin{align*}
		\frac{1}{2\rho_t} - \frac{1}{2\etat_{x,t}} \le \frac{c_{t-1}^2}{2} \rho_t \prn*{ \Q[t] - \Q[t-1] }
		;
	\end{align*}
	See proof in \Cref{lem: diff of etat}.
	Now, because we also have that $\max\{\norm{x_{t+1}-y_t}, \norm{y_{t+1}-x_{t+1}}\} \le \frac{2\rbar[t]}{c_{t-1}}$, we get
	\begin{align*}
		\rbar[t] \alpha_t \inner{\g }{ x_{t+1}-\xopt}
		&\le \frac{9}{2} \rbar[t]^2 \rho_t \alpha_t^2 \norm{\g-\m}^2
		+ \frac{1}{2\etat_{y,t}} \prn*{ \d[t]^2 - \d[t+1]^2 }
		.
	\end{align*}
	Thus,
	\begin{align*}
		2 \etat_{y,t} \rbar[t] \alpha_t \inner{\g }{ x_{t+1}-\xopt}
		&\le 9 \rbar[t]^2 \etat_{y,t} \rho_t \alpha_t^2 \norm{\g-\m}^2
		+ \prn*{ \d[t]^2 - \d[t+1]^2 }
		.
	\end{align*}
	Consequentially, by summing the two sides of the inequality, we get that for all $t\ge0$
	\begin{align*}
		2 \sum_{k=0}^{t} \etat_{y,k} \rbar[k] \alpha_k \inner{\g[k] }{ x_{k+1}-\xopt}
		&\le 9 \sum_{k=0}^{t} \rbar[k]^2 \etat_{y,k} \rho_k \alpha_k^2 \norm{\g[k] - \m[k]}^2
		+ \sum_{k=0}^{t} \prn*{ \d[k]^2 - \d[k+1]^2 }\\
		&\le \frac{ 9 \rbar[t]^2 }{ h^2 } \sum_{k=0}^{t} \frac{ \Q[k] - \Q[k-1] }{ \prn*{ s + \Q[k] } \logp^2\prn*{ \frac{s + \Q[k]}{s} } }
		+ \sum_{k=0}^{t} \prn*{ \d[k]^2 - \d[k+1]^2 }
		.
	\end{align*}
	\Cref{lem:bound-a-k-infinite-sum} gives us that
	\begin{align*}
		\sum_{k=0}^{t} \frac{ \Q[k] - \Q[k-1] }{ \prn*{ s + \Q[k] } \logp^2\prn*{ \frac{s + \Q[k]}{s} } } \le 1
		.
	\end{align*}
	Therefore, we obtain that
	\begin{align*}
		2 \sum_{k=0}^{t} \etat_{y,k} \rbar[k] \alpha_k \inner{\g[k] }{ x_{k+1}-\xopt}
		&\le \frac{ 9 \rbar[t]^2 }{ h^2 }
		+ \sum_{k=0}^{t} \prn*{ \d[k]^2 - \d[k+1]^2 }
		.
	\end{align*}
	Thus,
	\begin{align*}
		2 \sum_{k=0}^{t} \etat_{y,k} \rbar[k] \alpha_k \inner{\exg[k] }{ x_{k+1}-\xopt}
		\le & \frac{ 9 \rbar[t]^2 }{ h^2 }
		+ 2 \sum_{k=0}^{t} \etat_{y,k} \rbar[k] \alpha_k \inner{ \exg[k] - \g[k] }{ x_{k+1}-\xopt}
		\\
		&
		+ \sum_{k=0}^{t} \prn*{ \d[k]^2 - \d[k+1]^2 }
		.
	\end{align*}
	Consequentially, as \Cref{lem: bound for distance} gives us that
	\begin{align*}
		\sum_{k=0}^{t} \etat_{y,k} \rbar[k] \alpha_k \inner{\exg[k] }{ x_{k+1}-\xopt} \ge 0
		,
	\end{align*}
	we get that
	\begin{align*}
		0
		\le \frac{ 9 \rbar[t]^2 }{ h^2 }
		+ 2 \sum_{k=0}^{t} \eta_{y,k} \alpha_k \inner{ \exg[k] - \g[k] }{ x_{k+1}-\xopt}
		+ \sum_{k=0}^{t} \prn*{ \d[k]^2 - \d[k+1]^2 }
		.
	\end{align*}
	Therefore, we get that for all $t\ge0$
	\begin{align}
		\label{eq: bound d t using prev steps}
		\d[t+1]^2
		\le \frac{ 9 \rbar[t]^2 }{ h^2 }
		&+ 2 \sum_{k=0}^{t} \eta_{y,k} \alpha_k \inner{ \exg[k] - \g[k] }{ x_{k+1}-\xopt}
		+ \d[0]^2
		.
	\end{align}
	
	As we are in the noiseless case, and $h=12$, we get that for all $t\ge0$
	\begin{align*}
		\d[t+1]^2
		&\le \frac{ \rbar[t]^2 }{ 16 }
		+ \d[0]^2\\
		&\le \prn*{ \d[0] + \frac{1}{4}\rbar[t] }^2
		.
	\end{align*}
	Finally, \Cref{lem: d bound recursive} now gives us that for all $t\ge0$
	\begin{align*}
		\d[t] \le 2 \d[0] ~~\text{and}~~ \r[t] \le 4 \d[0]
		.
	\end{align*}
\end{proof}

\subsection{Proof of \Cref{thm:noiseless-main}}
\label{app: noiseless-main}
\begin{proof}
	Define
	\begin{align*}
		c_t = 12 \logp^2\prn*{ \frac{\norm{\m[0]}^2 + \Q[t]}{\norm{\m[0]}^2} }
		.
	\end{align*}
	From \Cref{lem: rbar grwoth}, we get that for all $t \ge 0$ the distance between iterates is not large:
	\begin{align*}
		\max\crl*{ \norm{ x_{t+1} - y_{t} } , \norm{ x_{t+1} - y_{t+1} } } &\le \frac{2\rbar[t]}{c_{t-1}}
		.
	\end{align*}
	Now, we fulfill all the conditions for \Cref{prop:noiseless-stability} and therefore, for all $t\ge0$
	\begin{align*}
		\dbar[t] \le 2 \d ~~ \text{and} ~~ \rbar[t] \le 4 \d
		.
	\end{align*}

	\Cref{prop:noiseless-subopt} gives that for all $t\ge0$ and for all $s\ge0$
	\begin{align*}
		f\prn{\xbar[t]} - f\prn{x\opt}
		\le \O{\frac{ s^{3/2} \L \prn*{ \rbar[t+1] + \d}^2 + \prn*{ \rbar[t+1] + \d} \brk*{ \sqrt{\Gy[t]} - s \sqrt{\Q[t]} }_{+}}{ \prn*{\sum_{k=0}^{t} \rbar[k]/\rbar[t+1] }^2 } }
		.
	\end{align*}
	By using the fact that $\rbar[t] \le 4 \d$, we get that for all $t\ge0$
	\begin{align}
		\label{eq: prop5 eq0}
		f\prn{\xbar[t]} - f\prn{x\opt}
		\le \O{\frac{ s^{3/2} \L \d^2 + \d \brk*{ \sqrt{\Gy[t]} - s \sqrt{\Q[t]} }_{+}}{ \prn*{\sum_{k=0}^{t} \rbar[k]/\rbar[t+1] }^2 } }
		.
	\end{align}

	Recall that 
	\begin{align*}
		\tau = \argmax_{t < T} \sum_{i \le t} \frac{\rbar[i]}{\rbar[t+1]}
		.
	\end{align*}
	To show the non-smooth rate, we set $s=0$ and obtain
	\begin{align*}
		\sqrt{\Gy[\tau]}
		&\le c_{T} \sqrt{ \max_{k\le T-1} \crl*{\alpha_k^2 \norm{\m[k]}^2 } + \sum_{k=0}^{T-1} \alpha_k^2 \norm{ \g[k] - \m[k] }^2 }
		\le c_{T} \sqrt{ T^2 \lip^2 + T^3 \lip^2 }
		\le 2\lip T^{3/2} c_{T}
		.
	\end{align*}
	This result, with \cref{eq: prop5 eq0}, gives us that
	\begin{align}
		\label{eq: prop5 1}
		f\prn{\xbar[\tau]} - f\prn{x\opt}
		\le \O{\frac{ \lip \d T^{3/2}  }{ \prn*{\sum_{k=0}^{\tau} \rbar[k]/\rbar[t+1] }^2 } c_{T}  }
		.
	\end{align}

	To show the smooth rate, setting $s=2c_{t+1}$ yields
	\begin{align*}
		\sqrt{\Gy[t]} - s \sqrt{\Q[t]}
		&\le c_{t+1} \prn*{ \sqrt{ \Q[t] + \M[t] }  - 2\sqrt{ \Q[t] } }
		\le c_{t+1} \prn*{ \sqrt{ \M[t] }  - \sqrt{ \Q[t] } }
		.
	\end{align*}
	For some $\kappa_{t}\le t$ we have that $\sqrt{ \M[t] } = \alpha_{\kappa_{t}} \norm{ \m[\kappa_{t}] }$.
	In addition, 
	the smoothness of $f$ implies that $\norm{\grad f(z)}^2 \le 2\L [ f(z)-f(\xopt)]$ for all $z\in\xset$. Combining this fact with the triangle inequality gives us that, in the noiseless setting,
	\begin{align*}
		\alpha_{\kappa_{t}} \norm{ \m[\kappa_{t}] }
		=
		\alpha_{\kappa_{t}} \norm{ \exm[\kappa_{t}] }
		&\le \alpha_{\kappa_{t}} \norm{\exg[\kappa_{t}] - \exm[\kappa_{t}]} + \alpha_{\kappa_{t}} \sqrt{2\L} \sqrt{f\prn{\xbar[\kappa_{t}]} - f\prn{\xopt}}
		.
	\end{align*}
	Thus,
	\begin{align*}
		\sqrt{ \M[t] }
		&\le \sqrt{ \Q[t] } + \alpha_{\kappa_{t}} \sqrt{2\L} \sqrt{f\prn{\xbar[\kappa_{t}]} - f\prn{\xopt}}
		.
	\end{align*}
	Therefore,
	\begin{align*}
		\sqrt{\Gy[t]} - s \sqrt{\Q[t]}
		&\le \alpha_{\kappa_{t}} \sqrt{2 c_{t+1}^2 \L} \sqrt{f\prn{\xbar[\kappa_{t}]} - f\prn{\xopt}}
		.
	\end{align*}
	This result, together with \cref{eq: prop5 eq0}, give us that for all $t\ge0$, there exist $\kappa_{t}\le t$ such as
	\begin{align*}
		f\prn{\xbar[t]} - f\prn{x\opt}
		\le \O{\frac{ c_{t+1}^{3/2} \L \d^2 + \alpha_{\kappa_{t}} \sqrt{ c_{t+1}^2 \L \d^2} \sqrt{f\prn{\xbar[\kappa_{t}]} - f\prn{\xopt}} }{ \prn*{\sum_{k=0}^{t} \rbar[k]/\rbar[t+1] }^2 }  }
		.
	\end{align*}
	Using the previous inequality and \Cref{lem: remove sqrt sub opt} we obtain that for all $t\ge0$ that
	\begin{align}
		\label{eq: prop5 2}
		f\prn{\xbar[t]} - f\prn{x\opt}
		\le \O{\frac{ \L \d^2 }{ \prn*{\sum_{k=0}^{t} \rbar[k]/\rbar[t+1] }^2 } c_{t+1}^2 }
		.
	\end{align}
	Combining the result from \cref{eq: prop5 1} and \cref{eq: prop5 2} gives
	\begin{align}
		\label{eq: prop5 4}
		f\prn{\xbar[\tau]} - f\prn{x\opt}
		\le \O{\frac{ \min\crl*{ \L \d^2 , \lip \d T^{3/2} }  }{ \prn*{\sum_{k=0}^{\tau} \rbar[k]/\rbar[t+1] }^2 } c_{T}^2  }
		.
	\end{align}
	
	\Cref{lem:bound-a-ratios} gives us that
	\begin{equation*}
		\sum_{k=0}^{\tau} \rbar[k]/\rbar[t+1] \ge \frac{1}{e}\prn*{ \frac{\numSteps}{\log_{+}(\rbar[\numSteps] / \reps)} -1}.
	\end{equation*}
	Thus, if $\numSteps \ge 2\log_{+}(\rbar[\numSteps] / \reps)$ then
	\begin{align*}
		\frac{1}{ \sum_{k=0}^{\tau} \rbar[k]/\rbar[t+1] }
		\le \O{ \frac{ 1 }{T} \logp\prn*{ \frac{\rbar[T]}{\reps} } }
		.
	\end{align*}
	Therefore, from \cref{eq: prop5 4}, we obtain
	\begin{align}
		\label{eq: prop5 3}
		f\prn{\xbar[\tau]} - f\prn{x\opt}
		\le \O{\frac{ \min\crl*{ \L \d^2 , \lip \d T^{3/2} }  }{ T^2 } c_{T}^2 \logp^2\prn*{ \frac{\rbar[T]}{\reps} } }
		.
	\end{align}
	We have that
	\begin{align*}
		c_{T}
		&\le \O{ \logp^2\prn*{ \frac{\norm{\m[0]}^2 + \Q[T-1]}{\norm{\m[0]}^2} } } \\&
		\overle{(i)} \O{ \logp^2\prn*{1 + \frac{T^3 \min\crl*{ \L \d, \lip } }{\norm{\exm[0]}^2} } }
		\le \O{ \logp^2\prn*{1 + T\frac{\min\crl*{ \L \d^2, \lip \d } }{f\prn{x_0} - f\prn{x\opt}} } }
		,
	\end{align*}
	due to $(i)$ the noiseless setting and $f$ being $\L$-smooth and $\lip$-Lipschitz, and $(ii)$ convexity, which implies  $f\prn{x_0} - f\prn{x\opt} \le \d \norm{\exm[0]}$
	Finally, from \cref{eq: prop5 3}, we obtain
	\begin{align}\label{eq:noiseless-main-bound pf}
		f\prn{\xbar[\tau]} - f\prn{x\opt}
		\le \O{\frac{ \min\crl*{ \L \d^2 , \lip \d T^{3/2} }  }{ T^2 } \logp^4\prn*{1 + T\frac{\min\crl*{ \L \d^2, \lip \d } }{f\prn{x_0} - f\prn{x\opt}} } \logp^2\prn*{ \frac{\d}{\reps} } }
		.
	\end{align}
	
	Finally, for $\numSteps \le 2\log_{+}(\rbar[\numSteps] / \reps)$ the theorem holds trivially since $f\prn{\xbar[\tau]} - f\prn{x\opt} \le \min\crl*{ \L \dbar[ \tau ]^2, \lip \dbar[ \tau ] }$ and $\dbar[\tau] \le 2\d[0]$ by \Cref{prop:sochastic-stability}. Therefore,
	\begin{align*}
		f\prn{\xbar[\tau]} - f\prn{x\opt}
		\le \O{\min\crl*{ \L \d[ 0 ]^2, \lip \d[ 0 ] }} %
		\le \O{ \frac{\min\crl*{ \L \dbar[ 0 ]^2, \lip \dbar[ 0 ] }}{ \numSteps^2 } \log_{+}^2(\rbar[\numSteps] / \reps) }
		,
	\end{align*}
	and so the bound \Cref{eq:noiseless-main-bound pf} holds in all cases, concluding the proof.
\end{proof}

\section{Proofs for \Cref{sec:stochastic} (the stochastic setting)}

\subsection{Proof of \Cref{prop:stochastic-subopt}}
\label{app:stochastic-subopt-proof}

\begin{proof}
	Define 
	\begin{align*}
		\minQ[t] \defeq \sum_{k=0}^{t} \alpha_k^2 \min\crl*{ \norm{ \exg[k] - \exm[k] }^2 , \norm{ \g[k] - \m[k] }^2 }
		.
	\end{align*}
	
	Our proof continues from \cref{eq: general regret bound} in the proof 
	\Cref{prop:noiseless-subopt}, which also holds for stochastic gradients.  
	\begin{align*}
		\sum_{k=0}^{t}\rbar[k] \alpha_k \inner{\g[k] }{ x_{k+1}-\xopt}
		&\le 180 s^{3/2} \rbar[t+1] \prn*{ \rbar[t+1] + \dbar[t+1] }^2 \L\\
		&\qquad + 5\rbar[t+1] \prn*{ \rbar[t+1] + \dbar[t+1] } \prn*{ s \sqrt{ \Q[t] } - 2s \sqrt{ \minQ[t] } }\\
		&\qquad + 5\rbar[t+1] \prn*{ \rbar[t+1] + \dbar[t+1] } \prn*{ \sqrt{ \max \crl*{ \Gy[t] , \Q[t] } } - s \sqrt{ \Q[t] } }
		.
	\end{align*}
	For all $k\ge0$
	\begin{align*}
		\norm{ \g[k] - \m[k] }^2
		&\le 2\norm{ \exg[k] - \exm[k] }^2 + 2\norm{ (\g[k] - \exg[k]) - (\m[k] -\exm[k]) }^2\\
		&\le 2 \min\crl*{ \norm{ \exg[k] - \exm[k] }^2 , \norm{ \g[k] - \m[k] }^2 } + 4 \norm{\m - \exm}^2 + 4 \norm{\g - \exg}^2
		.
	\end{align*}
	Thus, for all $k\ge0$
	\begin{align*}
		\norm{ \g[k] - \m[k] }^2
		&\le 2 \min\crl*{ \norm{ \exg[k] - \exm[k] }^2 , \norm{ \g[k] - \m[k] }^2 } + 4 \norm{\m - \exm}^2 + 4 \norm{\g - \exg}^2
		.
	\end{align*}
	Multiplying by $\alpha_k^2$, summing and recalling that $\alpha_k \le k+1$ implies $\Q[t] \le 2\minQ[t] + 4(t+1)^3 \empVar[t]$, where $\empVar[t] = \frac{1}{t+1} \sum_{k=0}^{t} \prn*{\norm{\g - \exg}^2 + \norm{\m - \exm}^2}$ is the empirical variance.
	Substituting into \cref{eq: general regret bound}, we get that 
	\begin{align*}
		\sum_{k=0}^{t}\rbar[k] \alpha_k \inner{\g[k] }{ x_{k+1}-\xopt}
		&\le 180 s^{3/2} \rbar[t+1] \prn*{ \rbar[t+1] + \dbar[t+1] }^2 \L\\
		&\qquad + 5\rbar[t+1] \prn*{ \rbar[t+1] + \dbar[t+1] } \prn*{ \sqrt{ \max \crl*{ \Gy[t] , \Q[t] } } - s \sqrt{ \Q[t] } }\\
		&\qquad + 10 s \rbar[t+1] \prn*{ \rbar[t+1] + \dbar[t+1] } \sqrt{ \prn{t+1}^3 \empVar[t] }\numberthis\label{eq: noisy subopt 1}
		.
	\end{align*}
\Cref{lem: bound on noise mul x dist} gives us that with probability of at least $1-\delta-\P\brk*{ \bkbar{T-1} > \NB }$, for all $t\in \crl*{ 0,1, \dots, T-1 } $,
	\begin{align*}
		\abs*{ \sum_{k=0}^{t}\rbar[t] \alpha_k \inner{\exg[k] - \g[k]}{ x_{k+1}-\xopt } }
		&\le 8 \alpha_t \rbar[t] \prn*{ \rbar[t+1] + \d[0] } \sqrt{ \theta_{t+1,\delta}\sum_{k=0}^{t} \norm{ \exg[k] - \g[k] }^2 + \prn*{ \theta_{t+1,\delta} \NB }^2 }
		.
	\end{align*}
	Using the previous equality and the definition of $V_t$ we obtain that
	\begin{align*}
		\sum_{k=0}^{t} \rbar[k] \alpha_k & \inner{\exg[k] }{ x_{k+1}-\xopt}\\
		&= \sum_{k=0}^{t} \rbar[k] \alpha_k \inner{\g[k] }{ x_{k+1}-\xopt} + \sum_{k=0}^{t}\rbar[k] \alpha_k \inner{\exg[k] - \g[k] }{ x_{k+1}-\xopt}\\
		&\le \sum_{k=0}^{t} \rbar[k] \alpha_k \inner{\g[k] }{ x_{k+1}-\xopt}
		+ 8 \alpha_t \rbar[t] \prn*{ \rbar[t+1] + \d[0] } \sqrt{ \prn{t+1} \theta_{t+1,\delta} \empVar[t] + \prn*{ \theta_{t+1,\delta} \NB }^2 }.\numberthis \label{eq: noisy subopt 2}
	\end{align*}

	\Cref{lem: sub-optimality inequality} gives us that
	\begin{align*}
		f\prn*{ \xbar[t] } - f(\xopt)
		&\le \frac{ 1 }{ \sum_{k=0}^{t} \rbar[k] \alpha_k } \sum_{k=0}^{t} \rbar[k] \alpha_k \inner{\exg[k] }{ x_{k+1}-\xopt}
		.
	\end{align*}
	By combining the above inequality with \cref{eq: noisy subopt 1} and \cref{eq: noisy subopt 2}, we obtain
	\begin{align*}
		f\prn*{ \xbar[t] } - f(\xopt)
		&\le 180 s^{3/2} \frac{ \rbar[t+1] }{ \sum_{k=0}^{t} \rbar[k] \alpha_k } \L \prn*{ \rbar[t+1] + \dbar[t+1] }^2\\
		&+ 5 \frac{ \rbar[t+1] }{ \sum_{k=0}^{t} \rbar[k] \alpha_k } \prn*{ \rbar[t+1] + \dbar[t+1] } \prn*{ \sqrt{ \max \crl*{ \Gy[t] , \Q[t] } } - s \sqrt{ \Q[t] } }\\
		& + 10 \prn*{1+s} \frac{ \rbar[t+1] }{ \sum_{k=0}^{t} \rbar[k] \alpha_k } \prn*{ \rbar[t+1] + \dbar[t+1] } \sqrt{ \prn{t+1}^3 \empVar[t] + \prn*{ \theta_{t+1,\delta} \NB }^2 }
		.
	\end{align*}
	Now, as \Cref{lem: weight sum of alphas} gives us that $\rbar[t] \alpha_t^2 \le 2\sum_{k=0}^{t} \rbar[k] \alpha_k$, we obtain that
	\begin{align*}
		f\prn*{ \xbar[t] } - f(\xopt)
		&\le 360 s^{3/2} \frac{ 1 }{ \prn*{\sum_{k=0}^{t} \rbar[k]/\rbar[t+1] }^2 } \L \prn*{ \rbar[t+1] + \dbar[t+1] }^2\\
		&+ 10 \frac{ 1 }{ \prn*{\sum_{k=0}^{t} \rbar[k]/\rbar[t+1] }^2 } \prn*{ \rbar[t+1] + \dbar[t+1] } \prn*{ \sqrt{ \max \crl*{ \Gy[t] , \Q[t] } } - s \sqrt{ \Q[t] } }\\
		& + 20 \frac{ 1 }{ \prn*{\sum_{k=0}^{t} \rbar[k]/\rbar[t+1] }^2 } \prn*{ \rbar[t+1] + \dbar[t+1] } \sqrt{ \prn{t+1}^3 \empVar[t] + \prn*{ \theta_{t+1,\delta} \NB }^2 }
		.
	\end{align*}
	Finally, because that $\dbar[t+1] \le \d[0] + \rbar[t+1]$, we get that for any $\NB > 0$ with probability of at least $1-\delta-\P\brk{ \bkbar{T-1} > \NB }$ we have that for all $t<T$ and for any number $s\ge0$
	\begin{align*}
		f\prn{\xbar[t]} - f\prn{\xopt}
		&\le \O{ \text{RHS}_{\cref{eq:noiseless-subopt}} + \frac{ \prn{1+s} \prn*{ \rbar[t+1] + \d } \sqrt{ t^3 \theta_{t,\delta} \empVar[t] + \prn*{ t \theta_{t,\delta} \NB }^2 } }{ \prn*{\sum_{k=0}^{t} \rbar[k]/\rbar[t+1] }^2 } }
	\end{align*}
	where $\text{RHS}_{\cref{eq:noiseless-subopt}} = \frac{ s^{3/2}  \L \prn*{ \rbar[t+1] + \d}^2 + \prn*{ \rbar[t+1] + \d} \brk*{ \sqrt{ \max \crl*{ \Gy[t] , \Q[t] } } - s \sqrt{\Q[t]} }_{+} }{ \prn*{\sum_{k=0}^{t} \rbar[k]/\rbar[t+1] }^2}$ is the error term appearing in \Cref{prop:noiseless-subopt}.
\end{proof}

\subsection{Proof of \Cref{prop:sochastic-stability}}
\label{app:sochastic-stability-proof}

\begin{proof}
	The proof continues from \cref{eq: bound d t using prev steps} in the proof of \Cref{prop:noiseless-stability}, which also holds for stochastic gradients. Substituting $h=400$ in \cref{eq: bound d t using prev steps} gives, for all $t\ge0$, 
	\begin{align*}
		\d[t+1]^2
		\le \frac{ 9 \rbar[t]^2 }{ 400^2 }
		&+ 2 \sum_{k=0}^{t} \eta_{y,k} \alpha_k \inner{ \exg[k] - \g[k] }{ x_{k+1}-\xopt}
		+ \d[0]^2
		.
	\end{align*}
	Now, \Cref{lem: uni noise bound} gives us that with probability at least $1-\delta$, for all $t<T$
	\begin{align*}
		\abs*{ \sum_{k=0}^{t} \alpha_k \eta_{y,k} \inner{ \g[k] - \exg[k] }{ x_{k+1} - \xopt } }
		&\le \frac{ 12 \theta_{t+1,\delta} }{ 400 \theta_{T,\delta} } \rbar[t] \prn*{ \rbar[t+1] + \d[0] }
		\\&
		\le \frac{ 12 \theta_{t+1,\delta} }{ 400 \theta_{T,\delta} } \prn*{ \rbar[t] \rbar[t+1] + \rbar[t] \d }
		\le \frac{12}{400} \prn*{ 1 + \frac{3}{400} } \rbar[t]^2 + \frac{12}{400}\rbar[t] \d
		.
	\end{align*}
	Therefore,
	\begin{align*}
		\d[t+1]^2
		\le \frac{ 81 \rbar[t]^2 }{ 400^2 }
		& + \frac{24}{400}\rbar[t]^2 + \frac{24}{400}\rbar[t] \d
		+ \d[0]^2
		\le \frac{ \rbar[t]^2 }{ 16 } + \frac{ \rbar[t] \d }{ 2 } + \d
		.
	\end{align*}
	Thus, with probability of at least $1-\delta$, for all $t<T$
	\begin{align*}
		\d[t+1]^2
		&\le \prn*{ \d[0] + \frac{1}{4}\rbar[t] }^2
		.
	\end{align*}
	Finally, \Cref{lem: d bound recursive} gives us that with probability of at least $1-\delta$ for all $t<T$
	\begin{align*}
		\d[t] \le 2 \d[0] ~~\text{and}~~ \r[t] \le 4 \d[0]
		.
	\end{align*}
\end{proof}

\subsection{Proof of \Cref{thm:sochastic-main}}
\label{app: sochastic-main}
\begin{proof}
	Recall the notation
	\begin{align*}
	\tq[t] \defeq 2\alpha_t^2 \norm{ \tg - \m }^2
	~~\mbox{,}~~
	\maxQ[t] \defeq \sum_{k=0}^{t} \max\crl*{ \q[k], \tq[k] }
	~~\mbox{and}~~
	\noisep[t] \defeq 8 \prn*{t+1}^2 \bkbar{t}^2,
	\end{align*}
	and that our step sizes are of the form~\eqref{eq:step-size-form} with
	\[
	\Gy[t] = \hat{c}_{t}^2 \max\crl*{ \norm{\m[0]}^2  +\noisep[0]+ \noisep[t] + \tq[t] + \maxQ[t-1], \M[t] },	
	\]
	where 
	\[
	\hat{c}_t = 400 \theta_{T,\delta} \logp^2\prn*{1+ \frac{ \noisep[t] + \tq[t] + \maxQ[t-1]}{\norm{\m[0]}^2 + \noisep[0]} }
	.	
	\]

	We begin by verifying the conditions of \Cref{prop:sochastic-stability} with $s=\norm{\m[0]}^2 + \noisep[0]$, where condition $(iv)$ holds by construction. By \Cref{ass:bounded-noise} we have
	\[
	\norm{\g[t] - \tg[t]}^2 \le 2\norm{\g[t] - \exg[t]}^2 + 2\norm{\tg[t] - \exg[t]}^2 \le 4\bkbar{t}^2.
	\]
	Therefore, since $t+1\ge \alpha_t$, we have
	\[
		\tq[t] + \noisep[t] \ge \alpha_t^2\prn*{2\norm{\tg[t]-\m[t]}^2 + 2\norm{\g[t] - \tg[t]}^2 } \ge \alpha_t^2\norm{\g[t]-\m[t]}^2 = \q[t],
	\]
	and consequently
	\[
		\tq[t] + \noisep[t] + \maxQ[t-1] \ge \Q[t].
	\]

	Defining
	\[
		c_t = 400 \theta_{T,\delta} \logp^2\prn*{ 1+\frac{\Q[t]}{\norm{\m[0]}^2 + \noisep[0]} },
	\]
	we conclude that
	\[
	\Gy[t] \ge 
	c_t^2 \max\crl*{\norm{m_0}^2 + \noisep[0] + \Q[t], \M[t]}
	\ge
	c_t^2 (\norm{m_0}^2 + \noisep[0] + \Q[t])
	\]
	so that condition $(i)$ of \Cref{prop:sochastic-stability} holds. Next, since
	\[
	\Gy[t] \ge c_t^2 \max\{\Q[t],\M[t]\} \ge c_t^2\alpha_t^2 \max\{\norm{\g[t]-\m[t]}^2,\norm{\m[t]}^2\},
	\]
	\Cref{lem: rbar grwoth} guarantees condition $(ii)$ of  \Cref{prop:sochastic-stability}. Finally, we note that 
	\[
	\noisep[t] \ge 8\alpha_t^2 \max\{\norm{\g[t] -\exg[t]}^2, \norm{\tg[t] -\exg[t]}^2	\}
	\]
	and
	\[
		\noisep[t] + \tq[t] \ge \alpha_t^2 \prn*{2\norm{\m[t]-\tg[t]}^2 + 2\norm{\tg[t]-\exg[t]}^2 } \ge \alpha_t^2 \norm{\m[t] -\exg[t]}^2.
	\]
	Therefore, as $\sqrt{\Gy[t]} \ge c_t \sqrt{\noisep[t] + \tq[t]}$, condition $(iii)$ of \Cref{prop:sochastic-stability} holds.

	As all the conditions for \Cref{prop:sochastic-stability} hold, with probability of at least $1-\delta$, for all $t\ge0$
	\begin{align*}
		\dbar[t] \le 2 \d ~~ \text{and} ~~ \rbar[t] \le 4 \d
		.
	\end{align*}
	Recalling that $\bstar \defeq \max_{x:\norm{x-x\opt} \le 2\d[0]}\crl*{ \bfunc(x) }$, this also implies that $\P\brk{ \bkbar{T-1} > \bstar }\le \delta$. 

	We now combine the conclusions of \Cref{prop:sochastic-stability} with \Cref{prop:noiseless-subopt} to obtain a suboptimality bound for $\UDoG$. Substituting $\Pr*(\rbar[T] \le 4\d) \le \delta$ and $\P\brk{ \bkbar{T-1} > \bstar }\le \delta$ into \Cref{prop:stochastic-subopt} we get that, with probability at least $1-3\delta$, for all $t < T$ and $s\ge 0$, 
	\begin{align}
		\label{eq: prop9 eq first}
		f\prn{\xbar[t]} - f\prn{x\opt}
		&
		\le \O{ \frac{
		s^{3/2} \L \d^2 + \d \brk*{ \sqrt{\Gy[t]} - s \sqrt{\Q[t]} }_{+} +	
		\prn{1+s} \d \sqrt{ t^3 \theta_{t+1,\delta} \empVar[t] + \prn*{ t \theta_{t+1,\delta} \bstar }^2 } }{ \prn*{\sum_{k=0}^{t} \rbar[k]/\rbar[t+1] }^2 } }.
	\end{align}
	To simplify $\Gy[t]$ in the bound above, we invoke \Cref{lem: bound on g tide variance} which gives that, with probability at least   $1-\delta-\P\brk{ \bkbar{T-1} > \bstar }\ge 1-2\delta$, for all $t < T$, 
	\begin{align*}
			\maxQ[t] \le 5 \Q[t] + 80 \prn*{t+1}^3 \sqrt{ \theta_{t+1,\delta} } \empVar[t] + 2\prn*{t+1}^2 \theta_{t+1,\delta} \bstar^2
			,
	\end{align*}
	and hence
	\[
		\sqrt{\Gy[t]} \le \hat{c}_t \sqrt{\maxQ[t] + 2\M[t] + 2\noisep[t]}=
		O\prn*{
			\hat{c}_t\sqrt{\Q[t]} + \hat{c}_t\sqrt{\M[t]} + \hat{c}_t\theta_{T,\delta} \sqrt{t^3 V_t + t^2 \bstar^2}
		}.
	\]
	Combining this with the bound~\eqref{eq: prop9 eq first} and replacing $s$ with $s\hat{c}_t\sqrt{3}$, we get that with probability at least $1-5\delta$, for all $t<T$ and $s\ge 0$,
	\begin{align}
		\label{eq: prop9 clean subopt}
		f\prn{\xbar[t]} - f\prn{x\opt}
		&
		\le \O{ \frac{
			s^{3/2}\hat{c}_{t}^{3/2} \L \d^2 + \hat{c}_t \d \prn*{ \brk*{(1-s)\sqrt{\Q[t]} + \sqrt{\M[t]}}_+ + (1+s)\theta_{T,\delta} \sqrt{ t^3 V_t + t^2 \bstar^2}} }{ \prn*{\sum_{k=0}^{t} \rbar[k]/\rbar[t+1] }^2 } }.
	\end{align}

	The remainder of the proof parallels the proof of \Cref{thm:noiseless-main}, where we specialize our bound to the Lipschitz and smooth cases by choosing different values of $s$. For the Lipschitz case, we use the facts that 
	\[
		\Q[t] \le 4\sum_{k\le T} \alpha_t^2 \prn*{
			\norm{\exg[k]}^2 + \norm{\exm[k]}^2 + \norm{\g[k]-\exg[k]}^2 +\norm{\m[k]-\exm[k]}^2} = O( \lip^2 T^3 + \empVar[T] T^3)
	\]
	
	$\Q[t] = O(\lip^2 T^3)$ and $\M[t] \le O(\lip^2 T^2)$
	and (under the event $\dbar[T] \le 2\d$)
	\[
		\M[t] \le \max_{k \le T} \crl*{2\alpha_t^2 \prn*{\norm{\exm[k]}^2 + \norm{\m[k] - \exm[k]}^2}} = O( \lip^2 T^2 + \bstar^2 T^2),
	\]
	giving the suboptimality bound. Substituting these expression and $s=0$ into~\eqref{eq: prop9 clean subopt} we get, for all $t<T$, 
	\begin{align}
		\label{eq: prop9 1}
		f\prn{\xbar[t]} - f\prn{x\opt}
		&\le \O{ \hat{c}_t\frac{ \lip \d T^{3/2} + \d \theta_{T,\delta} \sqrt{ T^3  \empVar[T] + T^2\bstar^2 } }{ \prn*{\sum_{k=0}^{t} \rbar[k]/\rbar[t+1] }^2 } }.
	\end{align}
	
	For the smooth case and any $t<T$, let $\kappa_t \le t$ be such that 
	For some $\kappa_{t}\le t$ we have that
	\begin{align*}
		\sqrt{ \M[t] } = \alpha_{\kappa_{t}} \norm{ \m[\kappa_{t}] }
		.
	\end{align*}
	The smoothness of $f$ implies that $\norm{\grad f(z)}^2 \le 2\L [ f(z)-f(\xopt)]$ for all $z\in\xset$. Combining this fact with the triangle inequality gives us that
	\begin{align*}
		\alpha_{\kappa_{t}} \norm{ \m[\kappa_{t}] }
		\le \alpha_{\kappa_{t}} \norm{\exg[\kappa_{t}] - \exm[\kappa_{t}]}
		+ \alpha_{\kappa_{t}} \norm{ \m[\kappa_{t}] - \exm[\kappa_{t}] }
		+ \alpha_{\kappa_{t}} \sqrt{2\L} \sqrt{f\prn{\xbar[\kappa_{t}]} - f\prn{\xopt}}
	\end{align*}
	and therefore,
	\begin{align*}
		\sqrt{ \M[t] }
		&\le \sqrt{ \Q[t] } + \sqrt{ \prn{t+1}^3 \empVar[t] } + \alpha_{\kappa_{t}} \sqrt{2\L} \sqrt{f\prn{\xbar[\kappa_{t}]} - f\prn{\xopt}}.
	\end{align*}
	Substituting into \cref{eq: prop9 clean subopt} and taking $s=2$, we get, for all $t<T$, 
	\[
		f\prn{\xbar[t]} - f\prn{x\opt}
		\le \O{ \frac{
			\hat{c}_{t}^{3/2} \L \d^2 + \hat{c}_t \d \theta_{T,\delta} \sqrt{ T^3 V_T + T^2 \bstar^2} + \alpha_{\kappa_{t}} \sqrt{ \hat{c}_{t+1}^2 \L \d^2} \sqrt{f\prn{\xbar[\kappa_{t}]} - f\prn{\xopt}}}{ \prn*{\sum_{k=0}^{t} \rbar[k]/\rbar[t+1] }^2 } }.	
	\]
	Applying \Cref{lem: remove sqrt sub opt} and noting that $\theta_{T,\delta} \le \hat{c}_t$ simplifies the bound to
	\[
		f\prn{\xbar[t]} - f\prn{x\opt}
		\le \O{ \frac{
			\hat{c}_{T}^{2}\L \d^2 + \hat{c}_{T}\theta_{T,\delta}\d\sqrt{ T^3 V_T + T^2 \bstar^2} }{ \prn*{\sum_{k=0}^{t} \rbar[k]/\rbar[t+1] }^2 } }.	
			\numberthis
			\label{eq: prop9 2}
	\]	
	Combining the bounds \cref{eq: prop9 1} and \cref{eq: prop9 2} and noting that $\theta_{T,\delta} \le \hat{c}_T$, we conclude that, with probability at least $1-5\delta$, for all $t<T$, 
	\begin{align*}
		f\prn{\xbar[t]} - f\prn{x\opt}
		\le \O{\hat{c}_{T}^2 \cdot \frac{ \min\crl*{ \L \d^2 , \lip \d T^{3/2} } +  \d \sqrt{ T^3  \empVar[T-1] + T^2\bstar^2 }  }{ \prn*{\sum_{k=0}^{\tau} \rbar[k]/\rbar[t+1] }^2 }  }
		.
	\end{align*}
	For $\tau = \argmax_{t < T} \sum_{i \le t} \frac{\rbar[i]}{\rbar[t+1]}$, 
	\Cref{lem:bound-a-ratios} gives us that
	\begin{equation*}
		\sum_{k=0}^{\tau} \rbar[k]/\rbar[t+1] \ge \frac{1}{e}\prn*{ \frac{\numSteps}{\log_{+}(\rbar[\numSteps] / \reps)} -1}.
	\end{equation*}
	Thus, for $\numSteps \ge 2\log_{+}(\rbar[\numSteps] / \reps)$ we get (under the event $\rbar[T] \le 4\d$)
	\begin{align*}
		f\prn{\xbar[\tau]} - f\prn{x\opt}
		\le \O{\hat{c}_{T}^2 \logp^2 \prn*{\frac{\d}{\reps}} \cdot \frac{ \min\crl*{ \L \d^2 , \lip \d T^{3/2} } +  \d \sqrt{ T^3  \empVar[T-1] + T^2\bstar^2 }  }{ T^2  } },
	\end{align*}
	which establishes the theorem, since
	\begin{align*}
		\hat{c}_{T}
		&\overset{(i)}{\le} \O{ \logp^2\prn*{1 + \frac{T^2 \bstar^2 + \maxQ[T-1]}{\norm{\exm[0]}^2} } }\le \O{ \logp^2\prn*{1 + \frac{T^3 \bstar^2 + T^3 \sum_{k=0}^{T-1} \norm{ \exg[k] - \exm[k] }^2 }{\norm{\exm[0]}^2} } }
		\\ &
		\le  \O{ \logp^2\prn*{1 + \frac{T^3 \bstar^2 + T^3 \min\crl*{ \L \d, \lip } }{\norm{\exm[0]}^2} } }\overset{(ii)}{\le} \O{ \logp^2\prn*{1 + \frac{T^3 \bstar^2 \d^2 + T^3 \min\crl*{ \L \d^3, \lip \d^2 } }{f\prn{x_0} - f\prn{x\opt}} } }
		\\&
		= \O{ \logp^2\prn*{1 + T\frac{\bstar \d + \min\crl*{ \L \d^2, \lip \d } }{f\prn{x_0} - f\prn{x\opt}} } }
		,
	\end{align*}
	where $(i)$ is because $\norm{\exm[0]}^2 \le \norm{\exm[0] - \m[0] + \m[0]}^2 \le 2 \norm{\m[0]}^2 + \noisep[0] $, and 
	$(ii)$ is from convexity: $f\prn{x_0} - f\prn{x\opt} \le \d \norm{\exm[0]}$.

	Finally, when $\numSteps \le 2\log_{+}(\rbar[\numSteps] / \reps)$ the required bound is immediate from problem geometry, as explained at the end of the proof of \Cref{thm:noiseless-main}.
\end{proof}

\subsection{Proof of \Cref{coro:subGuassian-noise}}
\label{app: subGuassian-noise}

\begin{proof}
	Define
	\begin{align*}
		\delta'_t = \frac{\delta}{5\prn{t+1}^2}.
	\end{align*}
	A black-box reduction from sub-Gaussian to bounded stochastic gradient (\Cref{lem: sub-gaussian to bounded}) shows that at each iteration $t$, with probability at least $1-\delta_t'$, a call to a $\sigma^2$-sub-Gaussian subgradient oracle produces an identical result to a call to an alternative stochastic gradient that is bounded by $3\sigma\sqrt{\log\prn{ 3 / \delta'_t }}$.

	We apply \Cref{thm:sochastic-main} to $\UDoG$ with the alternative, bounded stochastic gradient oracle. Thus, for this setting, with probability at least $1-5\delta$, we have $\dbar[T] \le 2 \d$, $\rbar[T] \le 4 \d$, and the suboptimality bound~\eqref{eq:stochastic-main-bound} holds for $\bstar = \sstar \stob_{T-1,\delta}$. To conclude the proof we use \Cref{lem: sub-gaussian to bounded} to show that the algorithm described above produces output different than $\UDoG$ with the original sub-Gaussian oracle as at most 
	\begin{align*}
		3\sum_{t=0}^{\infty}  \delta'_t
		&\le \frac{3\delta}{5} \sum_{t=1}^{\infty} \frac{1}{t^2}
		\le \frac{3 \cdot \pi^2}{5 \cdot 6}\delta
		\le \delta
		,
	\end{align*}
	where the factor of $3$ comes from the fact that every $\UDoG$ iteration involves 3 stochastic gradient queries.
\end{proof}

\subsection{Proof of \Cref{coro:mini-batch}}
\label{app: mini-batch}

\begin{proof}
	A mini-batch of  $B$ gradient oracle results, each with noise bounded by $\lip$, is a $\frac{2 \lip^2}{B}$-sub-Gaussian (see \Cref{lem: mini-batch to sub-Gausian}), and we can therefore apply \Cref{coro:subGuassian-noise} with $\sigma_t^2 = \frac{2 \lip^2}{B}$. Moreover, reusing the  sub-Gaussian-to-bounded reduction in the proof of \Cref{coro:subGuassian-noise} (\Cref{app: subGuassian-noise}) we get that, with probability at least $1-6\delta$, 
	\begin{align*}
		\sqrt{\empVar[T]} \le \frac{\sqrt{2} \lip}{\sqrt{B}} \stob_{T,\delta}
	\end{align*}
	holds in addition to the suboptimality bound given by \Cref{coro:subGuassian-noise}. Substituting the above bound on $\sqrt{\empVar[T]}$ along with $\bstar \le \sqrt{2}\frac{\lip}{\sqrt{B}} \zeta_{T,\delta}$ concludes the proof.
\end{proof} 
\section{Suboptimality lemmas}

\subsection{Weighted regret to suboptimality conversion (\Cref{lem: sub-optimality inequality})}

The following lemma is a straightforward reproduction of Lemma 1 from \citet{kavis2019unixgrad} with minor changes.
In addition, we use the proof of the following lemma as a starting point for the proof of \Cref{lem: bound for distance}.
\begin{lemma}[{\citet{kavis2019unixgrad}}]
	\label{lem: sub-optimality inequality}
	For any sequence of positive numbers $\w[0], \w[1], \w[2], \dots$, define
	\begin{align*}
		\xbar[t] \defeq \frac{\sum_{k=0}^{t} \w[k] x_{k+1} }{\sum_{k=0}^{t} \w[k]}
		.
	\end{align*}
	We have that for any $T>0$
	\begin{align*}
		f\prn*{ \xbar[T-1] } - f(\xopt)
		\le \frac{1}{\sum_{t=0}^{T-1} \w[t]}\sum_{t=0}^{T-1} \w[t] \inner{ \exg }{ x_{t+1} -\xopt }
		.
	\end{align*}
\end{lemma}

\begin{proof}
	For any $t\ge0$ we have that
	\begin{align*}
		\w \inner{\exg}{ x_{t+1}-\xopt }
		&= \w \inner{\exg}{ \frac{ \sum_{k=0}^{t} \w[k] }{ \w }\xbar - \frac{ \sum_{k=0}^{t-1} \w[k] }{ \w }\xbar[t-1] - \xopt }\\
		&= \w \inner{\exg}{ \frac{ \sum_{k=0}^{t} \w[k] }{ \w }\prn{\xbar -\xopt} - \frac{ \sum_{k=0}^{t-1} \w[k] }{ \w }\prn{\xbar[t-1] - \xopt} }\\
		&= \sum_{k=0}^{t} \w[k] \inner{\exg}{ \xbar -\xopt } - \sum_{k=0}^{t-1} \w[k] \inner{\exg}{ \xbar[t-1] -\xopt }\\
		&= \w \inner{\exg}{ \xbar -\xopt } + \sum_{k=0}^{t-1} \w[k] \inner{\exg}{ \xbar - \xbar[t-1] }
		.
	\end{align*}
	By using the convexity of $f$, we get
	\begin{align}
		\label{eq: single element inequality}
		\w \inner{\exg}{ x_{t+1}-\xopt }
		&\ge \w \prn*{ f\prn{\xbar} - f\prn{\xopt} } + \sum_{k=0}^{t-1} \w[k] \prn*{ f\prn{ \xbar } - f\prn{ \xbar[t-1] } }
		.
	\end{align}
	Therefore, for any $T>0$
	\begin{align*}
		\sum_{t=0}^{T-1} \w \inner{\exg}{ x_{t+1}-\xopt }
		&\ge \sum_{t=0}^{T-1} \w \prn*{ f\prn{ \xbar } - f\prn{\xopt} } + \sum_{t=0}^{T-1} \sum_{k=0}^{t-1} \w[k] \prn*{ f\prn{ \xbar } - f\prn{ \xbar[t-1] } }\\
		&= \sum_{t=0}^{T-1} \w \prn*{ f\prn{ \xbar } - f\prn{\xopt} } + \sum_{k=0}^{T-2} \sum_{t=k+1}^{T-1} \w[k] \prn*{ f\prn{ \xbar } - f\prn{ \xbar[t-1] } }
		.
	\end{align*}
	By performing a telescopic summation, we obtain
	\begin{align*}
		\sum_{t=0}^{T-1} \w \inner{\exg}{ x_{t+1}-\xopt }
		&\ge \sum_{t=0}^{T-1} \w \prn*{ f\prn{ \xbar } - f\prn{ \xopt } } + \sum_{t=0}^{T-2} \w \prn*{ f\prn{\xbar[T-1]} - f\prn{\xbar} }\\
		&= \w[T-1] \prn*{ f\prn{ \xbar[T-1] } -f \prn{ \xopt } } + \sum_{t=0}^{T-2} \w \prn*{ f\prn{ \xbar } - f\prn{ \xopt } + f\prn{ \xbar[T-1] } - f\prn{ \xbar } }\\
		&= \sum_{t=0}^{T-1} \w \prn*{ f\prn{\xbar[T-1]} - f\prn{\xopt} }
		.
	\end{align*}
	Dividing both sides by $\sum_{t=0}^{T-1} \w$ concludes the proof.
\end{proof}

\subsection{Inductive suboptimality bound (\Cref{lem: remove sqrt sub opt})}

\begin{lemma}
	\label{lem: remove sqrt sub opt}
	Let $s_0,s_1,\dots,s_{T-1}$ and $h_0,h_1,\dots,h_{T-1}$ be non-negative non-decreasing sequences.
	Let $b>1$  such that $\rbar[t+1]/\rbar[t] \le b$ for any $t \in \crl*{ 0,1,2 \dots, T-1 }$.
	If for all $t \in \crl*{ 0,1,2 \dots, T-1 }$ there exist $\kappa_t\in \crl*{ 0,1,2 \dots, t } $ such that
	\begin{align*}
		f\prn{\xbar} - f\prn{\xopt}
		\le \frac{ \alpha_{\kappa_t} \sqrt{s_{t}}\sqrt{f\prn{\xbar[\kappa_t]} - f\prn{\xopt}} + h_{t} }{ \prn*{\sum_{k=0}^{t} \rbar[k]/\rbar[t+1] }^2 }
		,
	\end{align*}
	then for all $t \in \crl*{ 0,1,2 \dots, T-1 }$ we have that
	\begin{align*}
		f\prn{\xbar} - f\prn{\xopt}
		&\le \frac{ 4 b^2 \prn{s_t + h_t} }{ \prn*{\sum_{k=0}^{t} \rbar[k]/\rbar[t+1] }^2 }
		.
	\end{align*}
\end{lemma}

\begin{proof}
	We prove by induction that
	\begin{align*}
		f\prn{\xbar} - f\prn{\xopt}
		&\le \frac{ 4 b^2 \prn{s_t + h_t} }{ \prn*{\sum_{k=0}^{t} \rbar[k]/\rbar[t+1] }^2 }
		.
	\end{align*}
	We will only use the induction assumption for the case where $\kappa_t < t$.
	
	\paragraph{If $\kappa_t = t$:}
	We have that
	\begin{align*}
		f\prn{\xbar} - f\prn{\xopt}
		&\le \frac{ \alpha_{\kappa_t} \sqrt{s_{t}}\sqrt{f\prn{\xbar[\kappa_t]} - f\prn{\xopt}} + h_{t} }{ \prn*{\sum_{k=0}^{t} \rbar[k]/\rbar[t+1] }^2 }\\
		&\le \frac{ \frac{\rbar[t+1]}{\rbar[t]} \sqrt{s_{t}}\sqrt{f\prn{\xbar[\kappa_t]} - f\prn{\xopt}} }{ \sum_{k=0}^{t} \rbar[k]/\rbar[t+1] }
		+ \frac{ h_{t} }{ \prn*{\sum_{k=0}^{t} \rbar[k]/\rbar[t+1] }^2 }\\
		&\le \frac{ b \sqrt{s_{t}}\sqrt{f\prn{\xbar[\kappa_t]} - f\prn{\xopt}} }{ \sum_{k=0}^{t} \rbar[k]/\rbar[t+1] }
		+ \frac{ h_{t} }{ \prn*{\sum_{k=0}^{t} \rbar[k]/\rbar[t+1] }^2 }
		.
	\end{align*} 
	Thus,
	\begin{align*}
		f\prn{\xbar} - f\prn{\xopt} - \frac{ b \sqrt{s_{t}}\sqrt{f\prn{\xbar[\kappa_t]} - f\prn{\xopt}} }{ \sum_{k=0}^{t} \rbar[k]/\rbar[t+1] }
		\le \frac{ h_{t} }{ \prn*{\sum_{k=0}^{t} \rbar[k]/\rbar[t+1] }^2 }
		.
	\end{align*}
	If
	\begin{align*}
		\frac{f\prn{\xbar} - f\prn{\xopt}}{2}
		\le f\prn{\xbar} - f\prn{\xopt} - \frac{ b \sqrt{s_{t}}\sqrt{f\prn{\xbar[\kappa_t]} - f\prn{\xopt}} }{ \sum_{k=0}^{t} \rbar[k]/\rbar[t+1] }
		,
	\end{align*}
	then
	\begin{align*}
		f\prn{\xbar} - f\prn{\xopt}
		\le \frac{ 2h_{t} }{ \prn*{\sum_{k=0}^{t} \rbar[k]/\rbar[t+1] }^2 }
		.
	\end{align*}
	Otherwise,
	\begin{align*}
		\frac{f\prn{\xbar} - f\prn{\xopt}}{2}
		\le \frac{ b \sqrt{s_{t}}\sqrt{f\prn{\xbar[\kappa_t]} - f\prn{\xopt}} }{ \sum_{k=0}^{t} \rbar[k]/\rbar[t+1] }
		.
	\end{align*}
	Therefore,
	\begin{align*}
		\sqrt{f\prn{\xbar} - f\prn{\xopt}}
		\le \frac{ 2b \sqrt{s_{t}} }{ \sum_{k=0}^{t} \rbar[k]/\rbar[t+1] }
		.
	\end{align*}
	Consequentially, 
	\begin{align*}
		f\prn{\xbar} - f\prn{\xopt}
		\le \frac{ 4b^2 s_{t} }{ \prn*{ \sum_{k=0}^{t} \rbar[k]/\rbar[t+1] }^2 }
		.
	\end{align*}
	In either case, we obtain that
	\begin{align*}
		f\prn{\xbar} - f\prn{\xopt}
		\le \frac{ 4b^2 \prn*{  s_{t} + h_t } }{ \prn*{ \sum_{k=0}^{t} \rbar[k]/\rbar[t+1] }^2 }
		.
	\end{align*}
	
	\paragraph{If $\kappa_t < t$:}
	We assume by induction that
	\begin{align*}
		f\prn{\xbar[\kappa_t]} - f\prn{\xopt}
		\le \frac{  4 b^2 \prn*{ s_{\kappa_t} + h_{\kappa_t} } }{ \prn*{ \sum_{k=0}^{\kappa_t} \rbar[k]/\rbar[\kappa_t+1] }^2 }
		.
	\end{align*}
	Therefore,
	\begin{align*}
		\alpha_{\kappa_t} \sqrt{s_{t}}\sqrt{f\prn{\xbar[\kappa_t]} - f\prn{\xopt}}
		&\le 2 b \sqrt{ s_{t}}  \sqrt{  s_{\kappa_t} +  h_{\kappa_t} } \frac{ \alpha_{\kappa_t} }{ \sum_{k=0}^{\kappa_t} \rbar[k]/\rbar[\kappa_t+1] }\\
		&\le 2 b \frac{\rbar[t+1]}{\rbar[t]} \sqrt{s_{t}}  \sqrt{ s_{t} + h_{t} }\\
		&\le 2 b^2 \prn*{ s_{t} + h_{t} }
		.
	\end{align*}
	Thus,
	\begin{align*}
		f\prn{\xbar} - f\prn{\xopt}
		&\le \frac{ 2 b^2 \prn*{ s_{t} + h_{t} } + h_t }{ \prn*{\sum_{k=0}^{t} \rbar[k]/\rbar[t+1] }^2 }\\
		&\le \frac{ 4 b^2 \prn*{ s_{t} + h_{t} } }{ \prn*{\sum_{k=0}^{t} \rbar[k]/\rbar[t+1] }^2 }
		.
	\end{align*}
	
	\paragraph{Finalizing the induction:}
	For $t=0$ we have $\kappa_t = 0 = t$. For the case $\kappa_t = t$ we did not use the induction assumption, and therefore we have the base of the induction:
	\begin{align*}
		f\prn{\xbar[0]} - f\prn{\xopt}
		\le \frac{ 4 b^2 \prn*{ s_{0} + h_{0} } }{ \prn*{ \rbar[0]/\rbar[1] }^2 }
		.
	\end{align*}
	Thus, by induction we get that for all $t\in \crl*{ 0,1,2,\dots, T-1 }$,
	\begin{align*}
		f\prn{\xbar} - f\prn{\xopt}
		&\le \frac{ 4 b^2 \prn*{ s_{t} + h_{t} } }{ \prn*{\sum_{k=0}^{t} \rbar[k]/\rbar[t+1] }^2 }
		.
	\end{align*}
	
\end{proof}

\subsection{General regret bound (\Cref{lem: unixgrad inequality})}

The following lemma is inspired by the regret analysis of $\UniXGrad$~\cite{kavis2019unixgrad}.

\begin{lemma}
	\label{lem: unixgrad inequality}
	Using \Cref{alg: general unixgrad dog}, \cref{eq:step-size-form} and \cref{eq: etat notation}, for any $t\ge0$, $\rho_t > 0$, we have that
	\begin{align*}
		\rbar[t] \alpha_t \inner{\g }{ x_{t+1}-\xopt}
		&\le \frac{\rbar[t]^2 \alpha_t^2 \rho_t}{2} \norm{\g-\m}^2 - \frac{1}{2\rho_t}\norm{x_{t+1} - y_{t}}^2\\
		&\qquad+ \prn*{ \frac{1}{2\rho_t} - \frac{1}{2\etat_{x,t}} } \prn*{ \norm{x_{t+1} - y_{t}}^2 + \norm{x_{t+1}-y_{t+1}}^2 }\\
		&\qquad+ \frac{1}{2\etat_{y,t}} \prn*{ \norm{y_{t}-\xopt}^2 - \norm{y_{t+1}-\xopt}^2 }
		.
	\end{align*}
\end{lemma}

\begin{proof}
	We have
	\begin{align*}
		\rbar[t] \alpha_t & \inner{\g }{ x_{t+1}-\xopt}
		\\&
		= \rbar[t] \alpha_t \inner{\g-\m }{ x_{t+1}-y_{t+1}}
		+ \rbar[t] \alpha_t \inner{\m }{ x_{t+1}-y_{t+1}}
		+ \rbar[t]\alpha_t \inner{\g }{ y_{t+1}-\xopt}. \numberthis\label{eq: subopt sum}
	\end{align*}
	In addition
	\begin{align*}
		\rbar[t] \alpha_t \inner{\g-\m }{ x_{t+1}-y_{t+1}}
		&\overset{(i)}{\le} \rbar[t]  \alpha_t \norm{\g-\m} \norm{x_{t+1}-y_{t+1}}\\
		&\overset{(ii)}{\le} \frac{\rho_t \rbar[t] \alpha_t^2}{2} \norm{\g - \m}^2 + \frac{1}{2 \rho_t}\norm{x_{t+1}-y_{t+1}}\label{eq: uni inequality opt} \numberthis
		,
	\end{align*}
	where $(i)$ is from Holder’s Inequality and $(ii)$ is due to Young’s Inequality.
	
	For the Euclidean Bregman divergence $\DR{x}{y} = \frac{1}{2} \norm{x-y}^2$ we have that the update rule $x_{t+1} = \Proj{\kset}{y_{t} - \alpha_t \eta_{x,t} \m }=\Proj{\kset}{y_{t} - \rbar[t]\alpha_t \etat_{x,t} \m }$ is equivalent to the update rule $x_{t+1} = \argmin_{x\in\kset} \crl*{\rbar[t]\alpha_t\inner{x }{ \m} + \frac{1}{\etat_{x,t}} \DR{x}{y_{t}}}$.
	Therefore, from the optimality of $x_{t+1}$ we get
	\begin{align}
		\rbar[t] \alpha_t \inner{\m }{ x_{t+1}-y_{t+1}}
		&\le \frac{1}{\etat_{x,t}} \inner{ \nabla_x \DR{x_{t+1}}{y_{t}} }{ x_{t+1}-y_{t+1} }\nonumber\\
		&= \frac{1}{\etat_{x,t}} \prn*{ \DR{y_{t+1}}{y_{t}} - \DR{x_{t+1}}{y_{t}} - \DR{y_{t+1}}{x_{t+1}} }\label{eq: uni x opt}
		.
	\end{align}
	
	Similarly, $y_{t+1} = \argmin_{y\in\kset} \crl*{\rbar[t]\alpha_t\inner{y }{ \g} + \frac{1}{\etat_{y,t}} \DR{y}{y_{t}}}$.
	Therefore, from the optimality of $y_{t+1}$ we get
	\begin{align}
		\rbar[t] \alpha_t \inner{\g }{ y_{t+1}-\xopt}
		&\le \frac{1}{\etat_{y,t}} \inner{ \nabla_x \DR{y_{t+1}}{y_{t}} }{ \xopt - y_{t+1} }\nonumber\\
		&= \frac{1}{\etat_{y,t}} \prn*{ \DR{\xopt}{y_{t}} - \DR{y_{t+1}}{y_{t}} - \DR{\xopt}{y_{t+1}} } \label{eq: uni y opt}
		.
	\end{align}
	
	By combining \cref{eq: uni inequality opt,eq: uni x opt,eq: uni y opt} into \cref{eq: subopt sum} we obtain that
	\begin{align*}
		\rbar[t] \alpha_t \inner{\g }{ x_{t+1}-\xopt}
		&\le \frac{\rbar[t]^2 \alpha_t^2 \rho_t}{2} \norm{\g-\m}^2 + \frac{1}{2 \rho_t}\norm{x_{t+1}-y_{t+1}}^2\\
		&\qquad+ \frac{1}{2\etat_{x,t}} \prn*{ \norm{y_{t+1} - y_{t}}^2 - \norm{x_{t+1} - y_{t}}^2 - \norm{y_{t+1} - x_{t+1}}^2 }
		\\&
		\qquad+ \frac{1}{2\etat_{y,t}} \prn*{ \norm{\xopt - y_{t}}^2 - \norm{y_{t+1} - y_{t}}^2 - \norm{\xopt - y_{t+1}}^2 }\\
		&= \frac{\rbar[t]^2 \alpha_t^2 \rho_t}{2} \norm{\g-\m}^2 - \frac{1}{2\rho_t}\norm{x_{t+1} - y_{t}}^2\\
		&\qquad+ \prn*{ \frac{1}{2\rho_t} - \frac{1}{2\etat_{x,t}} } \prn*{ \norm{x_{t+1} - y_{t}}^2 + \norm{x_{t+1}-y_{t+1}}^2 }\\
		&\qquad+ \frac{1}{2\etat_{y,t}} \prn*{ \norm{\xopt - y_{t}}^2 - \norm{\xopt - y_{t+1}}^2 }+ \prn*{ \frac{1}{2\etat_{x,t}} - \frac{1}{2\etat_{y,t}} } \norm{y_{t+1} - y_{t}}^2
		.
	\end{align*}
	Since $\etat_{y,t} \le \etat_{x,t}$, we may drop the final term in the above display, completing the proof.
\end{proof} 

\section{Iterate stability lemmas}

\subsection{A weighted regret bound (\Cref{lem: bound for distance})}

\begin{lemma}
	\label{lem: bound for distance}
	For any sequence of positive numbers $\w[0], \w[1], \w[2], \dots$, define
	\begin{align*}
		\xbar[t] \defeq \frac{\sum_{k=0}^{t} \w[k] x_{k+1} }{\sum_{k=0}^{t} \w[k]}
		.
	\end{align*}
	Let $\etat_{0}, \etat_{1}, \etat_{2},\dots$ be a non-increasing sequence of positive numbers.
	We have that for any $T>0$,
	\begin{align*}
		\sum_{t=0}^{T-1} \w[t] \etat_{t} \inner{ \exg }{ x_{t+1} -\xopt } \ge 0
		.
	\end{align*}
\end{lemma}

\begin{proof}
	Define
	\begin{align*}
		\tf\prn{x} = f\prn{x} - f\prn{\xopt}
		.
	\end{align*}
	We start from \cref{eq: single element inequality} inside the proof of \Cref{lem: sub-optimality inequality}, which says that for all $t\ge0$
	\begin{align*}
		\w \inner{\exg}{ x_{t+1}-\xopt }
		&\ge \w \prn*{ f\prn{\xbar} - f\prn{\xopt} } + \sum_{k=0}^{t-1} \w[k] \prn*{ f\prn{ \xbar } - f\prn{ \xbar[t-1] } }
		.
	\end{align*}
	Multiplying each side by $\etat_{t}$ and summing, we obtain 
	\begin{align*}
		\sum_{t=0}^{T-1} \w \etat_{t} \inner{\exg}{ x_{t+1}-\xopt }
		&\ge \sum_{t=0}^{T-1} \w \etat_{t} \prn*{ f\prn{\xbar} - f\prn{\xopt} } + \sum_{t=0}^{T-1} \sum_{k=0}^{t-1} \w[k] \etat_{t} \prn*{ f\prn{\xbar} - f\prn{\xbar[t-1]} }\\
		&= \sum_{t=0}^{T-1} \w \etat_{t} \tf\prn{\xbar} + \sum_{t=0}^{T-1} \sum_{k=0}^{t-1} \w[k] \etat_{t} \prn*{ \tf\prn{\xbar} - \tf\prn{\xbar[t-1]} }\\
		&\overset{(\star)}{\ge} \sum_{t=0}^{T-1} \w \etat_{t} \tf\prn{\xbar} + \sum_{t=0}^{T-1} \sum_{k=0}^{t-1} \w[k] \prn*{ \etat_{t}\tf\prn{\xbar} - \etat_{t-1}\tf\prn{\xbar[t-1]} }\\
		&= \sum_{t=0}^{T-1} \w \etat_{t} \tf\prn{ \xbar } + \sum_{k=0}^{T-2} \sum_{t=k+1}^{T-1} \w[k] \prn*{ \etat_{t} \tf\prn{\xbar} - \etat_{t-1}\tf\prn{\xbar[t-1]} }
		,
	\end{align*}
	where $(\star)$ is because that $\tf\prn{\xbar[t-1]}\ge 0$ and $\etat_{t-1} \ge \etat_{t} > 0$.
	
	We can now perform a telescopic summation and obtain
	\begin{align*}
		\sum_{t=0}^{T-1} & \w \etat_{t} \inner{\exg}{ x_{t+1}-\xopt }
		\ge \sum_{t=0}^{T-1} \w \etat_{t} \tf\prn{\xbar} + \sum_{t=0}^{T-2} \w \prn*{ \etat_{T-1} \tf\prn{\xbar[T-1]} - \etat_{t} \tf\prn{\xbar} }\\
		&= \w[T-1] \etat_{T-1} \tf\prn{\xbar[T-1]} + \sum_{t=0}^{T-2} \w \prn*{ \etat_{t} \tf\prn{ \xbar } + \etat_{T-1} \tf\prn{\xbar[T-1]} - \etat_{t} \tf\prn{\xbar} }\\
		&= \w[T-1] \etat_{T-1} \tf\prn{\xbar[T-1]} + \sum_{t=1}^{T-1} \w \etat_{T-1} \tf\prn{\xbar[T-1]}
		.
	\end{align*}
	Thus, because $\tf\prn{\xbar[T-1]} \ge 0$, we obtain that
	\begin{align*}
		\sum_{t=0}^{T-1} \w \etat_{t} & \inner{\exg}{ x_{t+1}-\xopt } \ge 0
		.
	\end{align*}
\end{proof}

\subsection{Inductive stability bound (\Cref{lem: d bound recursive})}

\begin{lemma}
	\label{lem: d bound recursive}
	If $\reps=\r[0] \le \d[0]$, and for all $t\ge1$ we have that
	\begin{align*}
		\norm{y_{t}-x_{t}} &\le \frac{\rbar[t-1]}{4} ~~\text{and}\\
		\d[t]^2 &\le \prn*{ \d[0] + \frac{1}{4} \rbar[t-1] }^2
		,
	\end{align*}
	then for all  $t\ge0$ we get that
	\begin{align*}
		\d[t] \le 2 \d[0] ~~\text{and}~~ \r[t] \le 4 \d[0]
		.
	\end{align*}
\end{lemma}

\begin{proof}
	We prove this lemma by induction. The basis of the induction is that
	for $t=0$ we get that $\d[0] \le 2 \d[0]$ and  $\r[0] \le \d[0] \le 4 \d[0]$.
	
	For any $t\ge1$, we assume that $\dbar[t-1] \le 2 \d[0]$ and  $\rbar[t-1] \le 4 \d[0]$.
	Thus,
	\begin{align*}
		\d[t] \le \d[0] + \frac{1}{4} \rbar[t-1] \le 2 \d[0]
		.
	\end{align*}
	Also,
	\begin{align*}
		\norm{ y_t - x_0 } \le \norm{ y_t - \xopt } + \norm{ x_0 - \xopt } = \d[t] + \d[0] \le 3 \d[0]
		.
	\end{align*}
	In addition,
	\begin{align*}
		\norm{ x_{t} - x_0 }
		&\le \norm{ y_{t} - x_0 } + \norm{ x_{t} - y_{t} }\\
		&\overset{(\star)}{\le} 3 \d[0] + \frac{\rbar[t-1]}{4}\\
		&\le 4 \d[0]
		.
	\end{align*}
	where $(\star)$ is because $\norm{ x_{t} - y_{t} } \le \frac{\rbar[t-1]}{4}$.
	As a result,
	\begin{align*}
		\d[t] \le 2 \d[0] ~~\text{and}~~ \r[t] \le 4 \d[0]
		.
	\end{align*}
	Finally, by induction, we get that for all  $t\ge0$
	\begin{align*}
		\d[t] \le 2 \d[0] ~~\text{and}~~ \r[t] \le 4 \d[0]
		.
	\end{align*}
\end{proof}

\subsection{Single-step iterate stability (\Cref{lem: rbar grwoth})}

\begin{lemma}
	\label{lem: rbar grwoth}
	Let $c$ be a positve number. Using \Cref{alg: general unixgrad dog}, for any $t\ge0$, if $\eta_{x,t} \le \frac{\rbar[t]}{c \alpha_t \norm{ \m } }$, $\eta_{y,t} \le \frac{\rbar[t]}{c \alpha_t \norm{ \g - \m } }$ and $\eta_{y,t} \le \eta_{x,t}$ then
	\begin{align*}
		\norm{ x_{t+1} - y_{t} } &\le \frac{\rbar[t]}{c}\\
		\norm{ y_{t+1} - y_{t} } &\le \frac{2\rbar[t]}{c}\\
		\norm{ x_{t+1} - y_{t+1} } &\le \frac{2\rbar[t]}{c}\\
		\rbar[t+1] &\le \rbar[t] \prn*{ 1 + \frac{2}{c} }
		.
	\end{align*}
\end{lemma}

\begin{proof}
	First, by definition of the iterates and the fact that $\kset$ is convex 
	(and projection onto a closed convex set is nonexpansive) we have
	\begin{align}
		\label{eq: norm 1}
		\norm{ x_{t+1} - y_{t} }
		= \norm{\Proj{\kset}{y_{t} - \alpha_t \eta_{x,t} \m } - y_t} 
		\le \alpha_t \eta_{x,t} \norm{\m}
		\le \frac{\rbar[t]}{c}
		.
	\end{align}
	Second, by definition of the iterates and the fact that $\kset$ is convex, we also have
	\begin{align}
		\norm{ y_{t+1} - y_{t} } &= \norm{\Proj{\kset}{y_{t} - \alpha_t \eta_{y,t} \g } - y_{t} }
		\le \alpha_t \eta_{y,t} \norm{\g} \nonumber \\
		&\le \alpha_t \eta_{y,t} \norm{\g - \m} + \alpha_t \eta_{y,t} \norm{\m}
		\le \frac{2\rbar[t]}{c} \label{eq: norm 2}
		.
	\end{align}
	Third, by definition of the iterates, the fact that $\kset$ is convex, 
	the fact $\eta_{y,t} \le \eta_{x,t}$, and the assumed upper bounds on $\eta_{y,t}$ and $\eta_{x,t}$ in the premise of this lemma
	we have
	\begin{align*}
		\norm{ x_{t+1} - y_{t+1} } 
		&= \norm{\Proj{\kset}{y_{t} - \alpha_t \eta_{x,t} \m } - \Proj{\kset}{y_{t} - \alpha_t \eta_{y,t} \g } } \\
		&\le \alpha_t \norm{\eta_{x,t}\m - \eta_{y,t} \g }
		\le \alpha_t \eta_{y,t} \norm{\g - \m} + \alpha_t \prn*{ \eta_{x,t} - \eta_{y,t} } \norm{\m}\\
		&\le \alpha_t \eta_{y,t} \norm{\g - \m} + \alpha_t \eta_{x,t} \norm{\m} \le \frac{2\rbar[t]}{c}
		.
	\end{align*}
	Finally,
	\begin{align*}
		\r[t+1] \le \r[t] + \max \prn*{ \norm{ x_{t+1} - y_{t} }, \norm{ y_{t+1} - y_{t} } }
		.
	\end{align*}
	Therefore, using \cref{eq: norm 1} and \cref{eq: norm 2} we obtain
	\begin{align*}
		\rbar[t+1]
		= \max \prn*{ \rbar[t] , \r[t+1] }
		\le \rbar[t] + \max \prn*{ \norm{ x_{t+1} - y_{t} }, \norm{ y_{t+1} - y_{t} } }
		\le \rbar[t] \prn*{ 1 + \frac{2}{c} }
		.
	\end{align*}
\end{proof} 
\section{Concentration bounds}

\subsection{An empirical-Bernstein-type time uniform concentration bound (\Cref{cor:product-mg-concentration})}

\begin{lemma}[From \citet{ivgi2023dog}]\label{cor:product-mg-concentration}
	Let $S$ be the set of nonnegative and nondecreasing sequences.
	Let $C_t\in\mathcal{F}_{t-1}$ and let $X_{t}$ be a martingale difference sequence adapted to $\filt_{t}$ such that $\left|X_{t}\right| \le C_t$
	with probability 1 for all $t$.
	Then, for all $\delta\in\left(0,1\right)$, $c > 0$,
	and $\hat{X}_{t}\in\filt_{t-1}$ such that $\abs{\hat{X}_{t}} \le C_t$
	with probability 1,
	\begin{flalign*}
		&\P\prn*{
			\exists t \le \numSteps, \exists \{y_i\}_{i=1}^{\infty} \in S :\abs[\Bigg]{\sum_{i=1}^t y_{i} X_{i}}
			\ge
			8 y_t \sqrt{ \TimeUniformLog[t,\delta]  \sum_{i= 1}^t  \left(X_{i}-\hat{X}_{i}\right)^{2} + c^2 \TimeUniformLog[t,\delta]^2}
			\,}\\
		&\qquad\qquad\le \delta + \P\prn*{\exists t \le \numSteps : C_t > c }.
	\end{flalign*}
\end{lemma}

\subsection{Concentration bound for suboptimally proof (\Cref{lem: bound on noise mul x dist})}

\begin{lemma}
	\label{lem: bound on noise mul x dist}
	Let $\NB>0$ and $\delta\in\prn*{0,1}$. In the bounded noise setting (\Cref{ass:bounded-noise}), using \Cref{alg: general unixgrad dog} and \cref{eq: b_t def}, with probability of at least $1-\delta-\P\brk*{ \bkbar{T-1} > \NB }$ we get that for all $t\in \crl*{ 0,1, \dots, T-1 } $ then
	\begin{align*}
		\abs*{ \sum_{k=0}^{t}\rbar[t] \alpha_k \inner{\exg[k] - \g[k]}{ x_{k+1}-\xopt } }
		&\le 8 \alpha_t \rbar[t] \prn*{ \rbar[t+1] + \d[0] } \sqrt{ \theta_{t+1,\delta}\sum_{k=0}^{t} \norm{ \exg[k] - \g[k] }^2 + \prn*{ \theta_{t+1,\delta} \NB }^2 }
		.
	\end{align*}
\end{lemma}

\begin{proof}
	For $k\in\crl*{ 0,1, \dots, T }$ define
	\begin{align*}
		\dtilde[k] = \max_{i \le k} \norm{ x_{k} - \xopt }
		.
	\end{align*}
	For $k\in\crl*{ 0,1, \dots, T-1 }$ define the random variables:
	\begin{align*}
		Y_k = \alpha_k \rbar[k] \dtilde[k+1], ~~ \text{and} ~~
		X_k = \inner{\exg[k] - \g[k]}{\frac{ x_{k+1} - \xopt }{\dtilde[k+1]}}
		.
	\end{align*}
	From these definitions we get
	\begin{align*}
		\sum_{k=0}^{t} Y_k X_k
		&= \sum_{k=0}^{t}\rbar[t] \alpha_k \inner{\exg[k] - \g[k]}{ x_{k+1}-\xopt }
		,
	\end{align*}
	and that $\crl*{Y_k}_{k=0}^{T-1}$ is a non-decreasing sequence of non-negative numbers.
	In addition, as $x_{k+1}$ and $\dtilde[k+1]$ are independent of the noise of $\g[k]$ then $X_{k}$ is a martingale difference sequence.
	Therefore, as $\abs{X_k} \le \bkbar{k}$ with probability of 1,  \Cref{cor:product-mg-concentration} gives us that
	\begin{align*}
		\P\prn*{ \exists t<T  ~:~ \abs*{ \sum_{k=0}^{t} Y_k X_k } \ge 8 Y_t \sqrt{ \theta_{t+1,\delta}\sum_{k=0}^{t} \prn*{ X_k - 0 }^2 + \prn*{ \theta_{t+1,\delta} \NB }^2 } }
		\le \delta + \P\brk*{ \bkbar{T-1} > \NB }
		.
	\end{align*}
	Therefore, by using the Cauchy–Schwarz inequality, we obtain that, with a probability of at least $1-\delta-\P\brk*{ \bkbar{T-1} > \NB }$, for all $t\in\crl*{ 0,1, \dots, T-1 }$ 
	\begin{align*}
		\abs*{ \sum_{k=0}^{t}\rbar[t] \alpha_k \inner{\exg[k] - \g[k]}{ x_{k+1}-\xopt } }
		&\le 8 \alpha_t \rbar[t] \dtilde[t+1] \sqrt{ \theta_{t+1,\delta}\sum_{k=0}^{t} \norm{ \exg[k] - \g[k] }^2 + \prn*{ \theta_{t+1,\delta} \NB }^2 }\\
		&\le 8 \alpha_t \rbar[t] \prn*{ \rbar[t+1] + \dtilde[0] } \sqrt{ \theta_{t+1,\delta}\sum_{k=0}^{t} \norm{ \exg[k] - \g[k] }^2 + \prn*{ \theta_{t+1,\delta} \NB }^2 }
		.
	\end{align*}
	Thus,
	\begin{align*}
		\abs*{ \sum_{k=0}^{t}\rbar[t] \alpha_k \inner{\exg[k] - \g[k]}{ x_{k+1}-\xopt } }
		&\le 8 \alpha_t \rbar[t] \prn*{ \rbar[t+1] + \d[0] } \sqrt{ \theta_{t+1,\delta}\sum_{k=0}^{t} \norm{ \exg[k] - \g[k] }^2 + \prn*{ \theta_{t+1,\delta} \NB }^2 }
		.
	\end{align*}
\end{proof}

\subsection{Concentration bound for iterate stability proof (\Cref{lem: uni noise bound})}

\begin{lemma}
	\label{lem: uni noise bound}
	Let $\etat_{y,t}$ be such that, for some $c,s>0$ we have
	\begin{align*}
		\frac{1}{\etat_{y,t}} &\ge c  
		\max\crl*{\sqrt{s+\Q[t]} \logp\prn*{ \frac{ s + \Q[t] }{s}}, \alpha_t\norm{\exg[t]-\m[t]},\alpha_t\norm{\exg[t]-\g[t]}}. 
	\end{align*}
	If for all $t\ge0$ we have that $\eta_{y,t}=\rbar[t] \etat_{y,t}$ is independent of $\g[t]$ given $x_0, \ldots, x_t$, then, with probability of at least $1-\delta$, for all $t\ge0$,
	\begin{align*}
		\abs*{ \sum_{k=0}^{t} \alpha_k \eta_{y,k} \inner{ \g[k] - \exg[k] }{ x_{k+1} - \xopt } }
		&\le \frac{ 12 \theta_{t+1,\delta} }{c} \rbar[t] \prn*{ \rbar[t+1] + \d[0] }
		.
	\end{align*}
\end{lemma}

\begin{proof}
	For $t\in\crl*{ 0,1, \dots, T }$ define
	\begin{align*}
		\dtilde[t] = \max_{k \le t} \norm{ x_{t} - \xopt }
		.
	\end{align*}
	For $t\in\crl*{ 0,1, \dots, T-1 }$ define
	\begin{align*}
		X_t &= \alpha_t \etat_{y,t} \inner{ \g - \exg }{ \frac{x_{t+1} - \xopt}{\dtilde[t+1]} } ~~,\\
		\hat{X}_t &= \alpha_t \etat_{y,t} \inner{ \exg-\m }{ \frac{x_{t+1} - \xopt}{\dtilde[t+1]} }~~\text{and}\\
		Y_t &= \rbar[t] \dtilde[t+1]
		.
	\end{align*}
	The assumption $\frac{1}{\etat_{y,t}} \ge c \alpha_t \max\crl*{\norm{\exg[t]-\m[t]},\norm{\g[t]-\exg[t]}}$ implies that $\max\crl*{ \abs{X_t}, \abs{\hat{X}_t} } \le \frac{1}{c}$.
	In addition, as $\m[t]$ ,$x_{t+1}$ and $\dtilde[t+1]$ are independent of the noise of $\g[t]$ then $\hat{X}_t$ is independent of the noise of $\g[t]$ and $X_{t}$ is a martingale difference sequence.
	Thus, \Cref{cor:product-mg-concentration} gives us that
	\begin{align*}
		\P\prn*{ \forall t \in \crl*{ 0,1,\dots } ~:~ \abs*{ \sum_{k=0}^{t} Y_k X_k } < 8 \rbar[t] \dtilde[t+1] \sqrt{ \theta_{t+1,\delta} \sum_{k=0}^{t} \prn*{ X_k - \hat{X}_k }^2 + \frac{1}{c^2} \theta_{t+1,\delta}^2 } }  \ge 1 - \delta
		.
	\end{align*}
	Furthermore, we have
	\begin{align*}
		\sum_{k=0}^{t}\prn*{ X_t - \hat{X}_t }^2 &  =
		\sum_{k=0}^{t} \prn*{ \alpha_k \etat_{y,k} \inner{ \g[k]-\m[k] }{ \frac{x_{k+1} - \xopt}{ \dtilde[t+1] } } }^2
		\le 
		\sum_{k=0}^{t} \alpha_t^2 {\etat_{y,k}}^2 \norm{\g[k]-\m[k]}^2
		\\ & 
		\overle{(i)}
		\frac{1}{ c^2 } \sum_{k=0}^{t} \frac{\alpha_t^2 \norm{\g[k]-\m[k]}^2}{ \prn*{ s + \sum_{k=0}^{t} \alpha_k^2 \norm{\g[k]-\m[k]}^2 } \logp^2\prn*{ \frac{ s + \sum_{k=0}^{t} \alpha_k^2 \norm{\g[k]-\m[k]}^2 }{s} } }\overle{(ii)} \frac{1}{c^2}
		,
	\end{align*}
	where $(i)$ follows from the assumption that $\frac{1}{\etat_{y,t}} \ge c \sqrt{s+\Q[t]} \logp\prn*{ \frac{ s + \Q[t] }{s}}$ and the definition of $Q_t$, and $(ii)$ is a direct result of \Cref{lem:bound-a-k-infinite-sum} with $a_k = s + \sum_{k=0}^t \alpha_k^2 \| g_k - m_k \|^2$.
	In addition, we have that 
	\[
		Y_t X_t = \alpha_t \eta_{y,t} \inner{ \g - \exg }{ x_{t+1} - \xopt }.
		\]
	Therefore, with probability of at least $1-\delta$, for all $t\ge0$ we have that
	\begin{align*}
		\abs*{ \sum_{k=0}^{t} \alpha_k \eta_{y,k} \inner{ \g[k] - \exg[k] }{ x_{k+1} - \xopt } }
		&\le 8 \rbar[t] \dtilde[t+1] \sqrt{ \frac{ \theta_{t+1,\delta} }{c^2} + \frac{\theta_{t+1,\delta}^2}{c^2} }\\
		&\le \frac{ 12 \theta_{t+1,\delta} }{c} \rbar[t] \prn*{ \rbar[t+1] + \dtilde[0] }\\
		&\le \frac{ 12 \theta_{t+1,\delta} }{c} \rbar[t] \prn*{ \rbar[t+1] + \d[0] }
		.
	\end{align*}
	
\end{proof}

\subsection{Relating $\maxQ[t]$ to $\Q[t]$ (\Cref{lem: bound on g tide variance})}

\begin{lemma}
	\label{lem: bound on g tide variance}
	Let $\NB>0$ and $\delta\in\prn*{0,1}$. In the bounded noise setting (\Cref{ass:bounded-noise}), using \Cref{alg: general unixgrad dog} and the step sizes \eqref{eq: step size option 3}, with probability of at least $1-\delta-\P\brk*{ \bkbar{T-1} > \NB }$ we get that, for all $t\in \crl*{ 0,1, \dots, T-1 }$,
	\begin{align*}
		\maxQ[t] 
		&\le 5 \Q[t] + 80 \prn*{t+1}^3 \sqrt{ \theta_{t+1,\delta} } \empVar[t] + 2\prn*{t+1}^2 \theta_{t+1,\delta} \NB^2
		.
	\end{align*}
\end{lemma}

\begin{proof}
	For all $k\ge 0$ we have
	\begin{align*}
		\norm{ \tg[k] - \m[k] }^2
		&\le 2 \norm{ \g[k] - \m[k] }^2  
		+ 2\norm{\g[k]-\tg[t]}^2 
		\\& \le 
		2\norm{ \g[k] - \m[k] }^2 + 4\norm{\g[k]-\exg[k]}^2 + 4\norm{\tg[k]-\exg[k]}^2
		.
	\end{align*}
	Therefore, since $\alpha_k \le k+1$, 
	\begin{align*}
		\sum_{k=0}^{t} \alpha_k^2 \norm{ \tg[k] - \m[k] }^2 
		&
		\le 2\sum_{k=0}^{t} \alpha_k^2 \norm{ \g[k] - \m[k] }^2 + 8\sum_{k=0}^{t} \prn*{k+1}^2 \norm{ \g[k] - \exg[k] }^2\\
		&\qquad+ 4\sum_{k=0}^{t} \prn*{k+1}^2 \prn{ \norm{ \tg[k] - \exg[k] }^2 - \norm{ \g[k] - \exg[k] }^2 } \numberthis \label{eq: g tide variance bound}
		.
	\end{align*}
	
	We now bound $\sum_{k=0}^{t} \prn*{k+1}^2 \prn{ \norm{ \tg[k] - \exg[k] }^2 - \norm{ \g[k] - \exg[k] }^2 }$.
	Define
	\begin{align*}
		X_t &= \prn{ \norm{ \tg[t] - \exg[t] }^2 - \norm{ \g[t] - \exg[t] }^2 } ~~,\\
		\hat{X}_t &= \norm{ \tg[t] - \exg[t] }^2 ~~\text{and}\\
		Y_t &= \prn*{t+1}^2
		.
	\end{align*}
	We have that for all $t\ge0$ then $\abs{X_t} \le \bkbar{t}^2$ and $\abs{\hat{X}_t} \le \bkbar{t}^2$ with probability 1.
	Therefore, \Cref{cor:product-mg-concentration} gives us that
	\begin{align*}
		\P&\prn*{ \forall t \in \crl*{ 0,1,\dots, T-1 } ~:~ \abs*{ \sum_{k=0}^{t} Y_k X_k } < 8 Y_t \sqrt{ \theta_{t+1,\delta} \sum_{k=0}^{t} \prn*{ X_k - \hat{X}_k }^2 +  \theta_{t+1,\delta}^2 \NB^4 } } \\
		&\qquad\qquad\qquad\qquad\qquad\ge 1 - \delta - \P\prn*{ \bkbar{T-1} > \NB }
		.
	\end{align*}
	
	Consequentially, by combining this result with \cref{eq: g tide variance bound}, we get that with probability at least $1 - \delta - \P\prn*{ \bkbar{T-1} > \NB }$ that for all $t\in \crl*{ 0,1,\dots, T-1 }$ we have that
	\begin{align*}
		\sum_{k=0}^{t} \alpha_k^2 \norm{ \tg[k] - \m[k] }^2 
		&\le 2\sum_{k=0}^{t} \alpha_k^2 \norm{ \g[k] - \m[k] }^2 + 40 \prn*{t+1}^2 \sqrt{ \theta_{t+1,\delta} } \sum_{k=0}^{t} \norm{ \g[k] - \exg[k] }^2 + \prn*{t+1}^2 \theta_{t+1,\delta} \NB^2
		.
	\end{align*}
	Substituting into the above equation the definition of $Q_t$ and $V_t$ given in \cref{eq: Q notation} and \cref{eq:emp-var-def}, respectively, and recalling the definition of $\maxQ[t]$ given in \cref{def: qtilde Qbar and pt}
	\[
	\maxQ[t] = \sum_{k=0}^t \alpha_k^2 \max\{\norm{\g[k]-\m[k]}^2, 2\norm{\tg[k]-\m[k]}^2\} \le \Q[t] + 2\sum_{k=0}^{t} \alpha_k^2 \norm{ \tg[k] - \m[k] }^2
	\]
	completes the proof.
\end{proof}

\subsection{Concentration inequality for bounded random vectors (\Cref{lem: mini-batch to sub-Gausian})}

\begin{lemma}[\citet{howard2020time}]
	\label{lem: mini-batch to sub-Gausian}
	For $T \in \N$, let $\{U_t\}_{t \in [T]}$ be a sequence of mean zero random vectors in $\R^d$ with $\| U_t \| \le c$ almost surely. Then
	\[
	\P\prn*{ \left\| \sum_{t=1}^T U_t \right\| \ge x } \le 2 \exp\left( -\frac{x^2}{2 c^2 T} \right).
	\] 
\end{lemma}

\begin{proof}
	This result follows from \citet[Corollary 10.a]{howard2020time} with $Y_t = \sum_{k=1}^t U_k$,
	$\Psi(\cdot) = \| \cdot \|$, $c_t = c$ and $m = c^2 T$.
	The selection of $\Psi(\cdot) = \| \cdot \|$ yields $D_\star = 1$ (see discussion preceding \cite[Corollary 10.a]{howard2020time}).
	Setting $c_t = c$ yields $V_t = c^2 t$. Hence $\frac{D_\star^2}{2m} (V_T - m) \le 0$ and 
	\citet[eq.~(4.28)]{howard2020time} gives the desired result.
\end{proof}

\section{Auxiliary lemmas}\label{app:aux lemmas}

\subsection{The growth rate of $\sum_{k}\rbar[k]\alpha_k$ (\Cref{lem: weight sum of alphas})}

We note that in accelerated optimization algorithms we normally have that $\alpha_t=\Theta(t)$.
Even though this is not the case for \UDoG, $\alpha_t$ is roughly similar to $t$.
First, it is easy to see that $1 \le \alpha_t \le t$.
Secondly, the running sum of $\rbar[t]\alpha_t$ grows roughly quadratically. This is shown in the following lemma, in which we replace $\alpha_t$ and $\rbar[t]$ with $a_t$ and $s_t$, respectively
\begin{lemma}
	\label{lem: weight sum of alphas}
	Let $s_0, s_1,\dots, s_t$ be a non-decreasing sequence of positive numbers.
	Define $a_k \defeq \sum_{i=0}^{k} \frac{s_i}{s_k}$, then
	\begin{align*}
		s_t a_t^2 \le 2 \sum_{k=0}^{t} s_k a_k
		.
	\end{align*}
\end{lemma}

\begin{proof}
	We have
	\begin{align*}
		\sum_{k=0}^{t} s_k a_k
		&= \sum_{k=0}^{t} \sum_{i=0}^{k} s_i
		= \sum_{k=0}^{t} (t-k+1) s_k
		.
	\end{align*}
	And,
	\begin{align*}
		s_t a_t^2
		&= \frac{1}{s_t} \sum_{k=0}^{t} \sum_{i=0}^{t} s_k s_i 
		= \frac{2}{s_t} \sum_{k=0}^{t} s_k \sum_{i=k}^{t} s_i - \frac{1}{s_t} \sum_{k=0}^{t} s_k^2
		\le 2\sum_{k=0}^{t} s_k \sum_{i=k}^{t} \frac{s_i}{s_t}
		\le 2 \sum_{k=0}^{t} (t-k+1) s_k.
	\end{align*}
	Thus,
	\begin{align*}
		s_t a_t^2 \le 2 \sum_{k=0}^{t} s_k a_k
		.
	\end{align*}
\end{proof}

\subsection{Discrete derivative lemma (\Cref{lem: diff of etat})}

\begin{lemma}
	\label{lem: diff of etat}
	Let $c$ be a positive number, and let $s_0, s_1, s_2,\dots$ be a sequence of positive numbers.
	For every $t\ge0$ define
	\begin{align*}
		\rho_t = \frac{1}{c\sqrt{\sum_{k=0}^{t} s_k }}
		.
	\end{align*}
	We have that for every $t\ge0$
	\begin{align*}
		\frac{1}{ \rho_{t+1} } - \frac{1}{ \rho_{t} }
		\le c^2 \rho_{t+1} s_{t+1}
		.
	\end{align*}
\end{lemma}

\begin{proof}
	For every $t\ge0$ we have that
	\begin{align*}
		s_{t+1}
		&= \sum_{k=0}^{t+1} s_k - \sum_{k=0}^{t} s_k
		\ge \sqrt{\sum_{k=0}^{t+1} s_k} \prn*{ \sqrt{\sum_{k=0}^{t+1} s_k} - \sqrt{\sum_{k=0}^{t} s_k} }
		= \frac{1}{ c^2 \rho_{t+1} } \prn*{ \frac{1}{ \rho_{t+1} } - \frac{1}{ \rho_{t} } }
		.
	\end{align*}
	Thus,
	\begin{align*}
		\frac{1}{ \rho_{t+1} } - \frac{1}{ \rho_{t} }
		\le c^2 \rho_{t+1} s_{t+1}
		.
	\end{align*}
\end{proof}

\subsection{Discrete integral lemma (\Cref{lem: Bound on sqrt sum B})}

\begin{lemma}
	\label{lem: Bound on sqrt sum B}
	For any positive numbers $c_1, c_2$, for any $t\ge0$, and for any sequence of non-negative numbers $B_0, B_1, B_2, \dots, B_t$ we have that
	\begin{align*}
		c_1 \sqrt{\sum_{k=0}^{t} B_k^2 } - \sum_{k=0}^{t} \frac{B_k^2}{c_2} \sqrt{\sum_{j=0}^{k} B_j^2 }
		&\le 2 c_1^{3/2} c_2^{1/2}
		.
	\end{align*}
\end{lemma}

\begin{proof}
	Define
	\begin{align*}
		\eta_{B,k} = \frac{1}{\sqrt{\sum_{j=1}^{k} B_j^2 }}
		.
	\end{align*}
	\Cref{lem: sqrt inequalities} gives us that
	\begin{align*}
		\sqrt{\sum_{k=0}^{t} B_k^2 } 
		\le \sum_{k=0}^{t} \frac{B_k^2}{\sqrt{\sum_{j=0}^{k} B_j^2 }}
		.
	\end{align*}
	Therefore, we obtain
	\begin{align*}
		c_1 \sqrt{\sum_{k=0}^{t} B_k^2 } - \sum_{k=0}^{t} \frac{B_k^2}{c_2} \sqrt{\sum_{j=0}^{k} B_j^2 }
		&\le c_1 \sum_{k=0}^{t} \frac{B_k^2}{\sqrt{\sum_{j=0}^{k} B_j^2 }} - \sum_{k=0}^{t} \frac{B_k^2}{c_2 }\sqrt{\sum_{j=0}^{k} B_j^2 }\\
		&=  \sum_{k=0}^{t} \prn*{ c_1\eta_{B,k} - \frac{1}{c_2 \eta_{B,k}} } B_k^2
		.
	\end{align*}
	Define
	\begin{align*}
		\kappa = \max \brk*{ \crl*{t \in \crl*{0,1,\dots,t} ~:~ 2c_1\eta_{B,t} - \frac{1}{c_2 \eta_{B,t}} > 0} \cup \crl*{-1} }
		.
	\end{align*}
	We have,
	\begin{align*}
		c_1 \sqrt{\sum_{k=0}^{t} B_k^2 } - \sum_{k=0}^{t} \frac{B_k^2}{c_2} \sqrt{\sum_{j=0}^{k} B_j^2 }
		&\le \sum_{k=0}^{\kappa} c_1\eta_{B,k}  B_k^2
		= c_1 \sum_{k=0}^{\kappa} \frac{B_k^2}{ \sqrt{ \sum_{j=0}^{k} B_j^2} }
		\overset{(\star)}{\le} 2c_1 \sqrt{ \sum_{k=0}^{\kappa} B_k^2}
		= \frac{2c_1}{\eta_{B,\kappa}} \indic{\kappa \ge 0}
		,
	\end{align*}
	where $(\star)$ is because of \Cref{lem: sqrt inequalities}.
	From the definition of $\kappa$, we obtain that
	\begin{align*}
		c_1\eta_{B,\kappa} > \frac{1}{c_2 \eta_{B,\kappa}}.
	\end{align*}
	Thus,
	\begin{align*}
		c_1 \sqrt{\sum_{k=0}^{t} B_k^2 } - \sum_{k=0}^{t} \frac{B_k^2}{c_2} \sqrt{\sum_{j=0}^{k} B_j^2 }
		&\le \frac{2c_1}{\eta_{B,\kappa}} \indic{\kappa \ge 0}\le 2 c_1^{3/2} c_2^{1/2}
		.
	\end{align*}
\end{proof}

\subsection{Additional lemmas from prior work}
\begin{lemma}[{e.g., \citet{levy2018online}}]
	\label{lem: sqrt inequalities}
	For any $k\ge0$ and for any sequence on non-negative numbers $s_0, s_1, s_2, \dots, s_k$ the following holds:
	\begin{align*}
		\sqrt{ \sum_{i=0}^{k} s_i }
		\le \sum_{i=0}^{k} \frac{s_i}{ \sqrt{ \sum_{j=0}^{i} s_j } }
		\le 2 \sqrt{ \sum_{i=0}^{k} s_i }
		.
	\end{align*}
\end{lemma}

\begin{lemma}[{\citet[Lemma 3]{ivgi2023dog}}]\label{lem:bound-a-ratios}
	Let $s_0,s_1,\ldots,s_{\numSteps}$ be a positive nondecreasing sequence. Then
	\begin{equation*}
		\max_{t\le \numSteps} \sum_{i < t} \frac{s_i}{s_t} \ge \frac{1}{e}\prn*{ \frac{\numSteps}{\log_{+}(s_{\numSteps} / s_0)} -1}.
	\end{equation*}
\end{lemma}

\begin{lemma}[{\citet[Lemma 6]{ivgi2023dog}}]\label{lem:bound-a-k-infinite-sum}
	Let $a_{-1}, a_0, a_1, \dots, a_t$ be a non-decreasing sequence of non-negative numbers, then
	\begin{equation*}
		\sum_{k=0}^t \frac{a_{k} - a_{k-1}}{a_{k} \log_{+}^2(a_{k} / a_{-1})} \le 1.
	\end{equation*}
\end{lemma}

\begin{lemma}[{\citet[Lemma 15]{attia2023sgd}}]
	\label{lem: sub-gaussian to bounded}
	Let $X$ be a $\sigma^2(x)$-sub-Gaussian.
	For and $\delta\in\prn{0,1}$ here exist a random variable $\bar{X}$ such that:
	\begin{enumerate}%
		\setlength{\itemsep}{1pt}
		\item $\bar{X}$ is zero-mean: $\E \bar{X} = 0$.
		\item $\bar{X}$ is equal to $X$ w.h.p: $\P\prn*{ \bar{X} = X } \ge 1 - \delta$.
		\item $\bar{X}$ is bounded with probability 1: $\P\prn*{ \norm{\bar{X}} = 3 \sigma \sqrt{\log\prn*{ 4 / \delta }} } = 1$.
	\end{enumerate}
\end{lemma}
\section{Experimental details}\label{app:experimental-details}

\subsection{\UDoG step sizes}
In the experiments, we use the following step sizes for \UDoG
\begin{align*}
	\eta_{x,t} = \frac{\rbar[t]}{\sqrt{ \max\crl*{ \Q[t-1], \M[t] } }} 
	~~\text{and}~~ %
	\eta_{y,t} = \frac{\rbar[t]}{\sqrt{ \max\crl*{ \Q[t], \M[t] } }},
\end{align*}
with $\rbar[t]$, $\Q[t]$, and $\M[t]$ as defined in \Cref{sec:preliminaries}. 
This step size is similar to the choice in \cref{eq: step size option 2}, which enjoys proven stability in the noiseless case, except we replace the logarithmic factor in the denominator with $1$; preliminary experiments indicated $1$ was the smallest value for which the algorithm was stable in practice. This difference between practical and theoretical algorithms is analogous to the difference between \DoG and its theoretically stable variant \TDoG~\citep{ivgi2023dog}. However, we maintain the maximization with $\M[t]$ in the denominator, mainly in order to ensure that $\eta_{x,t}$ and $\eta_{y,t}$ are not too large early in the training.
As with \DoG, the additional step size adjustments necessary for the stochastic setting (given in \cref{eq: step size option 3}) do not appear to be useful in practical settings. 

\subsection{\AcceleGrad-\DoG (\ADoG)}\label{app:experiments-accelegrad-dog}
While $\UDoG$ enjoys strong theoretical guarantees, it requires an extra-gradient computation at each step, which can be expensive in practice.
To address this, we propose an alternative algorithm, $\ADoG$, which combines \AcceleGrad~\citep{levy2018online} and $\DoG$.
To complete the combination we set $\alpha_t$ in the same way as it is calculated in \UDoG (\cref{alg: general unixgrad dog}).
$\ADoG$ is a simple algorithm that does not require an extra-gradient computation at each step and is presented in Algorithm~\ref{alg: accelegrad dog}.
While we do not provide theoretical guarantees for $\ADoG$, our experiments demonstrate its efficacy in practice. The main challenge in proving guarantees for \ADoG appears to lie in deriving a suboptimality bound akin to \Cref{prop:noiseless-subopt}, whose proof strongly leverages \UDoG's extra-gradient structure.

\begin{algorithm2e}[h] 
	\setstretch{1.1}
	\caption{\AcceleGrad-\DoG (\ADoG)}
	\label{alg: accelegrad dog}
	\LinesNumbered
	\DontPrintSemicolon
	\Input{
		Initialization $z^{(0)} \in \kset$, positive constant $\reps$ and number of iterations $T$.
	}
	
	Set $y_0=x_0=z_0$ and 
	$ \rbar[0] = \reps $\;
	\For{$t = 0, 1,, \ldots, T-1$ }{
		$\mathrlap{\alpha_t} \phantom{x_{t+1}} = \sum_{k=0}^{t} \rbar[k] / \rbar[t]$\;
		$\mathrlap{g_t} \phantom{x_{t+1}} \pick \gradientOracle{x_{t+1}}$\;
		$\mathrlap{\eta_t} \phantom{x_{t+1}} = \frac{\bar{r}_t}{\sqrt{\sum_{k=0}^{t} \alpha_k^2 \norm{g_k}^2 }}$\;
		$x_{t+1} = \frac{\alpha_t}{\sum_{k=0}^{t} \alpha_k} z_t + \prn*{1-\frac{\alpha_t}{\sum_{k=0}^{t} \alpha_k}}y_t$\;
		$y_{t+1} = x_{t+1} - \eta_t g_t$\;
		$z_{t+1} = \Pi_\kset\prn*{ z_t - \alpha_t \eta_t g_t }$\;
		$\rbar[t+1] = \max\crl*{ \rbar[t], \norm{ z_{t+1} - z_0 } }$
	}
	\Return $x_T$ \qquad \Comment{returning $y_T$ gives similar results in practice}
\end{algorithm2e}

\subsection{Convex experiments}\label{app-subsec:convex} 

The bulk of our experiments focus on smooth stochastic convex optimization problems, matching our theoretical assumptions.

\paragraph{Multiclass logistic regression.} We experiment with multi-class logistic regression on multiple tasks from the VTAB benchmark and the LIBSVM~\citep{libsvm} suite (a full list is given in \Cref{app-subsec:implementation}).
For VTAB tasks we use features obtained from a pretrained ViT-B/32 \citep{dosovitskiy2021image} model (i.e., perform linear probes), and for LIBSVM tasks we use apply logistic regression directly on the features provided.
\Cref{fig:convex-svhn,fig:convex-cifar100,fig:convex-dmlab,fig:convex-resisc45,fig:convex-sun397,fig:convex-clevr-distance,fig:convex-covertype,fig:convex-pendigits} show a view of the results for different datasets analogous to \Cref{fig:main-results}. \Cref{fig:convex-svhn-curves,fig:convex-cifar100-curves,fig:convex-dmlab-curves,fig:convex-resisc45-curves,fig:convex-sun397-curves,fig:convex-clevr-distance-curves,fig:convex-covertype-curves,fig:convex-pendigits-curves} give a complementary view by providing training curves at different batch sizes. 
As discussed in \Cref{sec:experiments}, we find that both $\UDoG$ and $\ADoG$ are competitive with well-tuned accelerated SGD (ASGD) and often significantly outperform $\DoG$ and tuned SGD. This is especially true for the training loss (for which our theory directly holds) and at large batch sizes, with $\ADoG$ outperforming $\UDoG$ in most cases, as both algorithms take advantage of the reduced variance in the gradient estimates to scale effectively with the batch size, as the theory suggests. 
In most experiments $\ADoG$ attain and tuned ASGD attain superior convergence rate in terms of test accuracy as well as train loss; the only exception is CIFAR-100 (\Cref{fig:convex-cifar100,fig:convex-cifar100-curves}, bottom rows) where the test accuracy does not closely track the train loss.

\paragraph*{Least-squares.} We modify the loss on a subset of the previous experiments to least squares, learned over a one-hot encoding of the features. 
We use features obtained from a pretrained ViT-B/32, similar to what we used for the multiclass logistic regression.
We find that our algorithms perform well in this setting as well. In comparison, while SGD and ASGD can perform well when tuned correctly, they become more sensitive to the choice of step size and momentum, performing poorly when not properly tuned and sometimes diverging completely. Similar to the other experiments, the results are given in \Cref{fig:ls-svhn,fig:ls-svhn-curves,fig:ls-cifar100,fig:ls-cifar100-curves}.

\paragraph*{Noiseless quadratic experiments.} As a final experiment, we compare the performance of the different algorithms on the quadratic function $f(x) = \sum_{i=1}^n \prn*{\frac{i}{2n} x_i^2 +x_i }$ with $n=10^4$. The results agree with the theoretical analysis, with all algorithms reaching the optimal solution or very close to it, barring GD and AGD with excessively high momentum and learning rate. Results are depicted in \Cref{fig:quadratic}.

\subsection{Non-convex experiments}\label{app-subsec:non-convex} 
While we mainly focus on demonstrating the effectiveness of $\UDoG$ and $\ADoG$ in settings that match our theoretical analysis, we also perform preliminary experimentation in practical scenarios, namely training neural networks on datasets of moderate scales.
In particular,  we train a ResNet-50 \citep{he2015deep} from scratch on a subset of the VTAB benchmark (\Cref{fig:resnet50-svhn,fig:resnet50-sun397,fig:resnet50-dmlab,fig:resnet50-resisc45,fig:resnet50-clevr-distance}). 
Additionally, we repeat two experiments from \citep{ivgi2023dog}: fine-tuning a CLIP model~\citep{radford2021learning} on ImageNet (\Cref{fig:clip-imagenet}), and training a WideResnet-28-10 \citep{zagoruyko2016wide} model from scratch on CIFAR-10 (\Cref{fig:wrn28-cifar10}).
We observe that $\UDoG$ often fails to converge to competitive results, while $\ADoG$ is quite competitive with \DoG on the VTAB tasks, but under-performs it for CIFAR-10 and ImageNet fine-tuning, indicating that it is not a yet a viable general-purpose neural network optimizer.

\subsection{Implementation details}\label{app-subsec:implementation} 

\paragraph*{Environment settings.}
All of our experiments were based on \texttt{PyTorch} \citep{paszke2019pytorch} (version 1.12.0). For $\DoG$ and the implementation of polynomial-decay model averaging \citep{shamir2013stochastic}, we used the \texttt{dog-optimizer} package (version 1.0.3) \citep{ivgi2023dog}. 
For ASGD, we used the native \texttt{PyTorch} SGD\footnote{\url{https://pytorch.org/docs/stable/generated/torch.optim.SGD.html}} with the Nesterov option enabled. 

VTAB experiments were based on the PyTorch Image Models (\texttt{timm}, version0.7.0dev0) repository \citep{Wightman2019timm}, with \texttt{TensorFlow datasets} (version 4.6.0) as a dataset backend \citep{Abadi2015Tensorflow}. LIBSVM \citep{libsvm} experiments were based on the \texttt{libsvmdata} (version 0.4.1) package.

To support the training and analysis of the results, we used \texttt{numpy}~\cite{Harris2020Array}, \texttt{scipy}~\cite{virtanen2020scipy}, \texttt{pandas}~\cite{mckinney2010pandas} and \texttt{scikit-learn}~\cite{pedregosa2011scikit}.

As much as possible, we leveraged existing recipes as provided by \texttt{timm} to train the models. 

\paragraph*{Datasets.}
The subset of datasets used in our VTAB experiments are: \textbf{CIFAR-100} \citep{krizhevsky2009learning}, \textbf{CLEVR-Dist} \citep{johnson2017clevr}, \textbf{DMLab} \citep{beattie2016deepmind},  \textbf{Resisc45} \citep{cheng2017remote}, \textbf{Sun397} \citep{xiao2010sun,xiao2016sun}, and \textbf{SVHN} \citep{netzer2011reading}. From LIBSVM, we used the \textbf{Pendigits} \citep{alpaydin1998pen} and \textbf{Covertype} \citep{blackard1998covertype} datasets, where cover covertype we used the scaled features version (i.e., \texttt{covtype.scale}).  We also experiment with \textbf{CIFAR-10} \citep{krizhevsky2009learning} and \textbf{ImageNet} \citep{deng2009imagenet}. 

\paragraph*{Models.}
The computer vision pre-trained models were accessed via \texttt{timm}.
The strings used to load the models were:  `resnet50',  `vit\textunderscore base\textunderscore patch32\textunderscore 224\textunderscore in21k'.

\paragraph*{Complexity measure.} To fairly compare all algorithms, we measure complexity by the number of batches evaluated, i.e., the number of stochastic gradient queries performed by the algorithm. \UDoG requires two batches per iteration while the rest of the algorithms we consider require only one. We note that the algorithms we compare also have different memory footprints and runtimes per iteration (by constant factors). We focus on the number of batches as our complexity metric since it is most relevant to our theory. Memory and per-iteration runtime optimizations are potentially possible for \UDoG and \ADoG; we leave investigating those to future work.

\paragraph*{ASGD model selection.} 
In the convex optimization experiments, we run (A)SGD over a wide range of momentum and learning rate parameters. For the batch size scaling figures (e.g., the left panels in \Cref{fig:main-results}), we pick the parameters that reach the target metric in the smallest number of batches, providing a conservative upper bound on the performance obtainable with a very carefully tuned algorithm. The learning curve figures adjacent to the batch size scaling figures (e.g., the middle panels in \Cref{fig:main-results}) show the learning curve for the (A)SGD run attaining the best target performance at the batch size indicated. For plots of learning curves at different batch sizes (e.g., \Cref{fig:ls-svhn-curves}), we select the (A)SGD parameters that are the first to reach 95\% of the best metric attained by \ADoG. If no such parameters exist, we take the parameters that reach the best performance within the iteration budget.

\paragraph*{Iterate averaging.} When evaluating test accuracy, we follow \citet{ivgi2023dog} and apply polynomial-decay weight averaging~\cite{shamir2013stochastic} with parameter 8. We did not tune this parameter or comprehensively check how beneficial the averaging is. Nevertheless, a cursory examination of our data suggests that averaging is mostly helpful across the board, but much more so for \DoG and SGD than their accelerated counterparts. This is in line with the theory, which provides guarantees on (essentially) the last iterate of $\UDoG$, but only the averaged iterate of \DoG.

\paragraph*{Learning rate schedule.} We use a constant learning rate schedule for (A)SGD. We do not use a decaying schedule such as cosine decay~\cite{loshchilov2016sgdr} as it would complicate comparing the smallest number of steps required to reach a target metric, since a decaying schedule requires knowing the number of steps in advance. Preliminary experiments indicate that, in the settings we study, cosine decay is not significantly better than a constant schedule combined with iterate averaging.

\paragraph*{Setting $\reps$.} Similarly to \citet{ivgi2023dog} we set $\reps = \gamma \prn*{ 1 + \norm{ x_0 } }$ with $\gamma = 10^{-6}$.
Our theoretical analysis suggests that the particular choice of $\reps$ does not matter as long as it is sufficiently small relative to the distance between the weight initialization x0 and the optimum.

\paragraph*{Weight decay.} We do not use weight decay in most experiments, except for training from scratch on CIFAR-10 (\Cref{fig:wrn28-cifar10}), where we use a weight decay of $5 \cdot 10^{-4}$. For \DoG we decay the parameters toward zero, while for \UDoG and \ADoG we decay the parameters toward the initial point $x_0$. That is, for $\DoG$ we add $5 \cdot 10^{-4} x$ to the stochastic gradient evaluated at $x$, while for \UDoG and \ADoG we add $5 \cdot 10^{-4}(x-x_0)$.

\paragraph*{Gradient accumulation.} Due to GPU memory limitations, in the non-convex experiments, for large batch sizes we divide each batch into smaller sub-batches of size of either 128 or 256 samples.
We calculate the gradient for each sub-batch and average those into a single gradient which we then use to perform a single step. 
When batch normalization is used (that is, for ResNet50), this is not mathematically identical to computing the gradient in one large batch.

\insertfigurepage{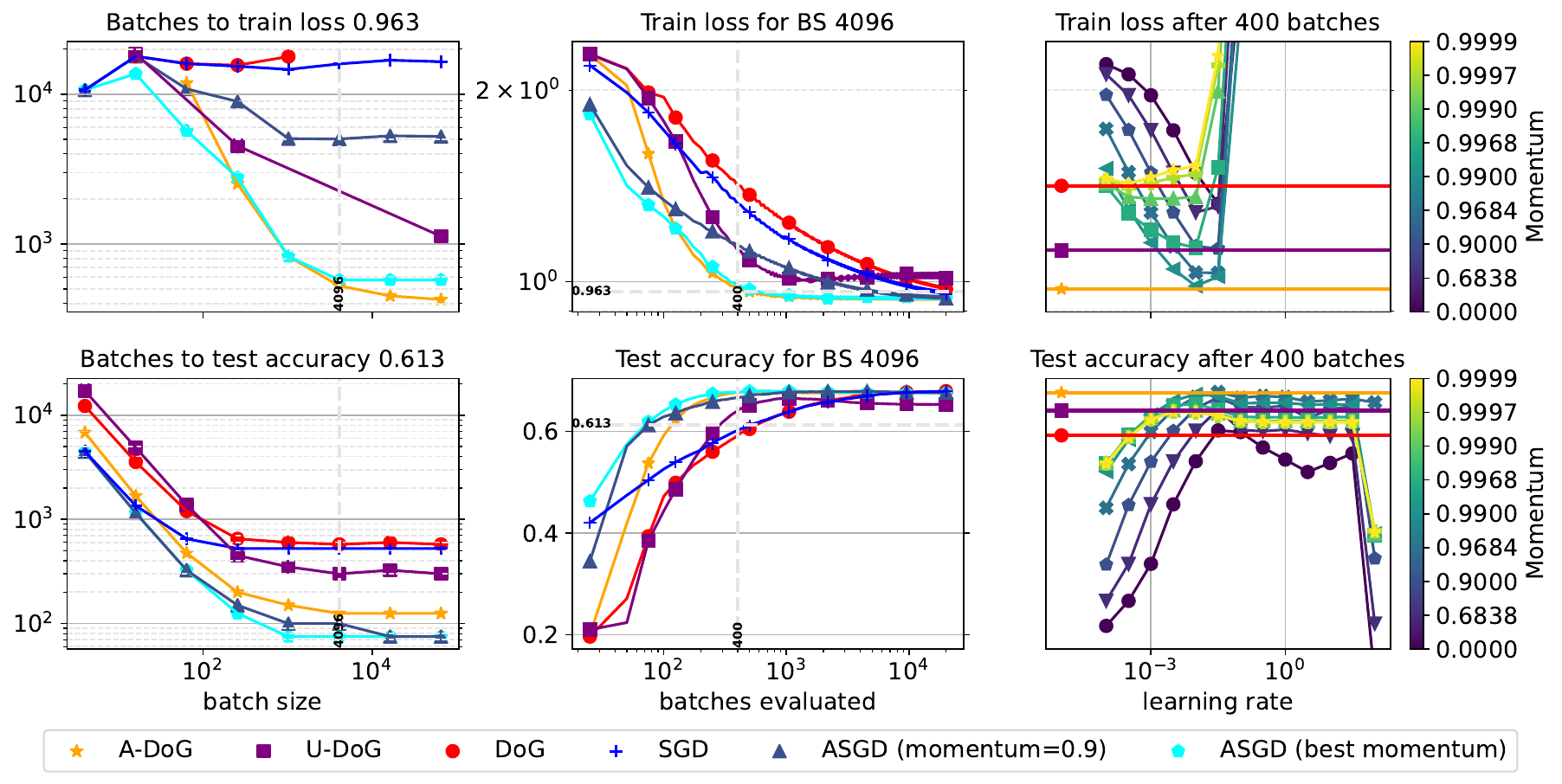}{\convexcaption{SVHN}{90}}{fig:convex-svhn}{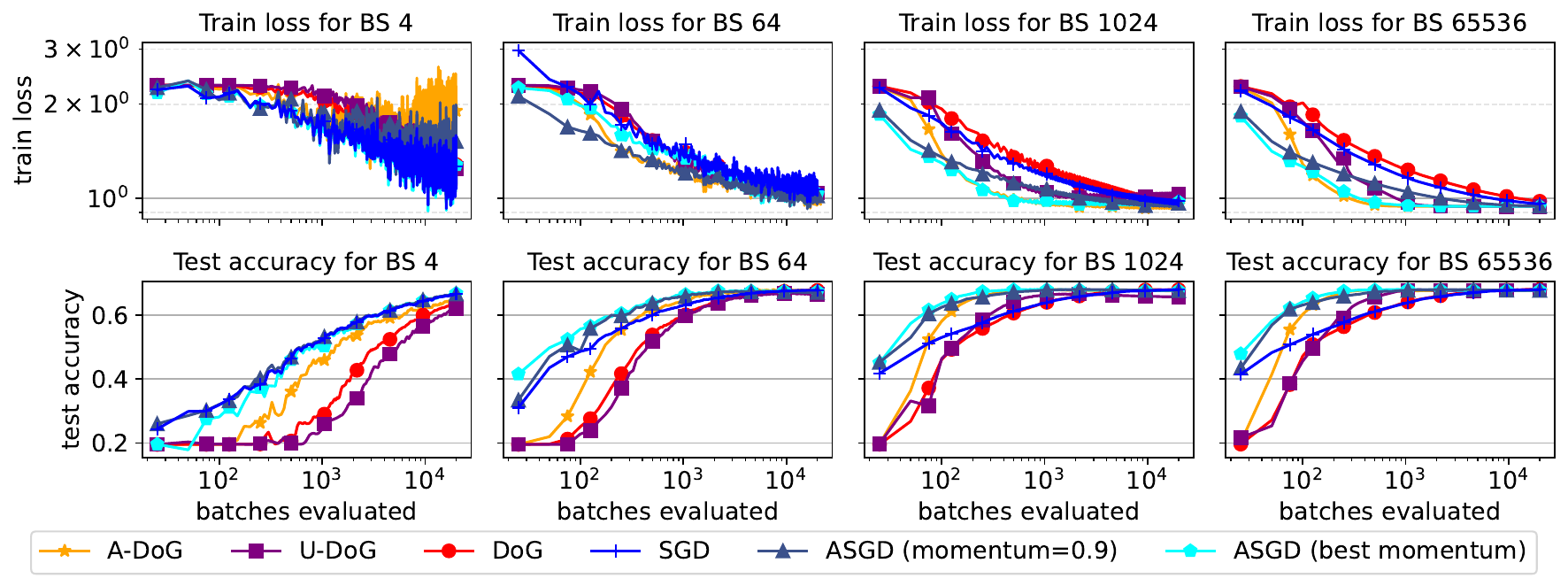}{\convexcaptioncurves{SVHN}{95}}{fig:convex-svhn-curves}

\insertfigurepage{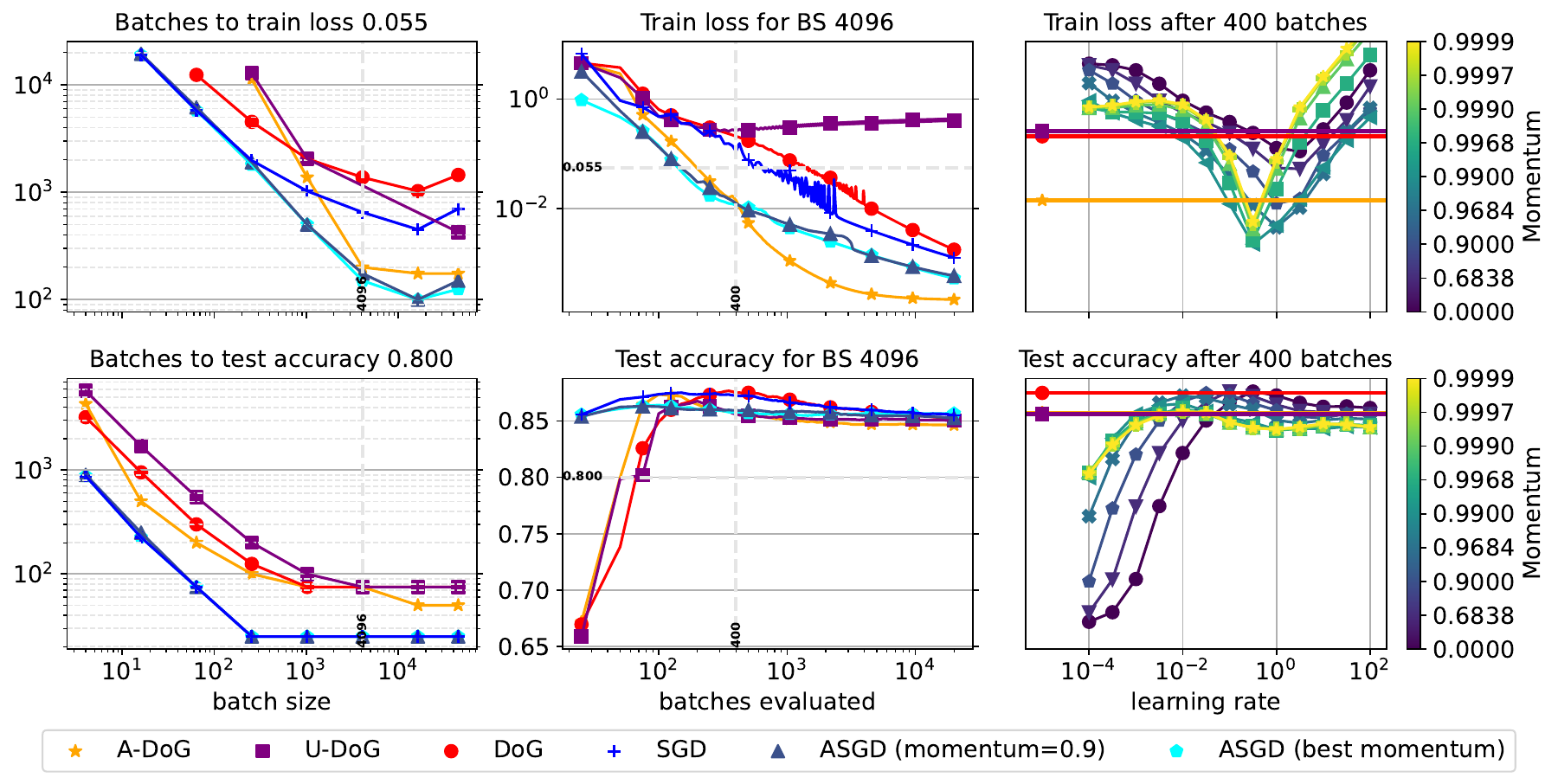}{\convexcaption{CIFAR-100}{90}}{fig:convex-cifar100}{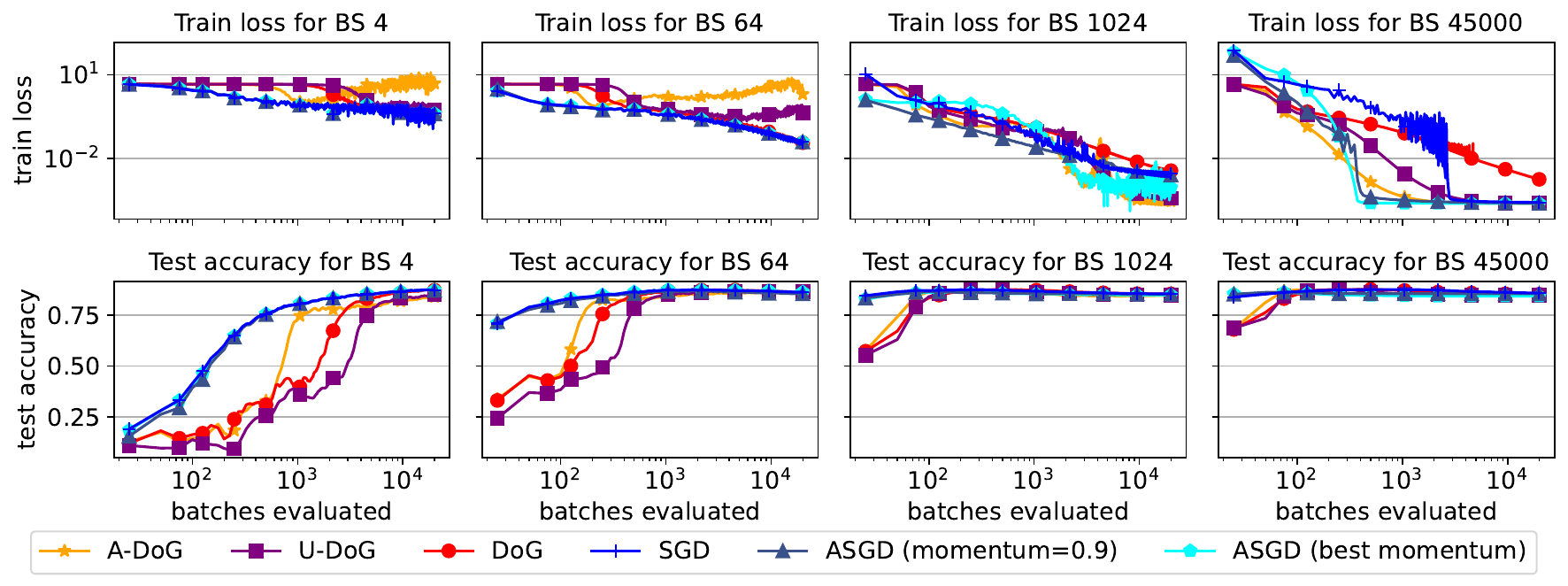}{\convexcaptioncurves{CIFAR-100}{95}}{fig:convex-cifar100-curves}

\insertfigurepage{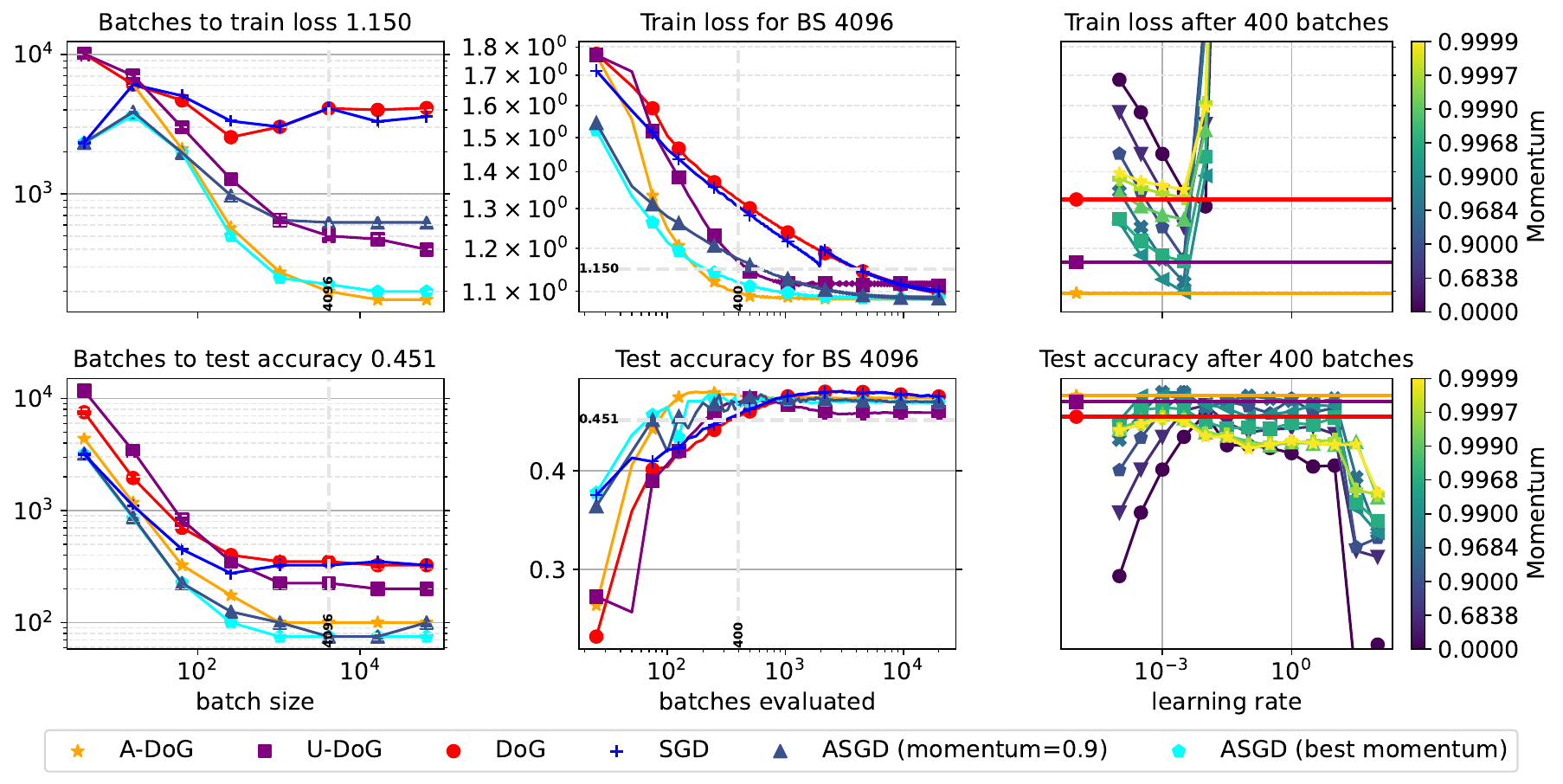}{\convexcaption{DMLab}{90}}{fig:convex-dmlab}{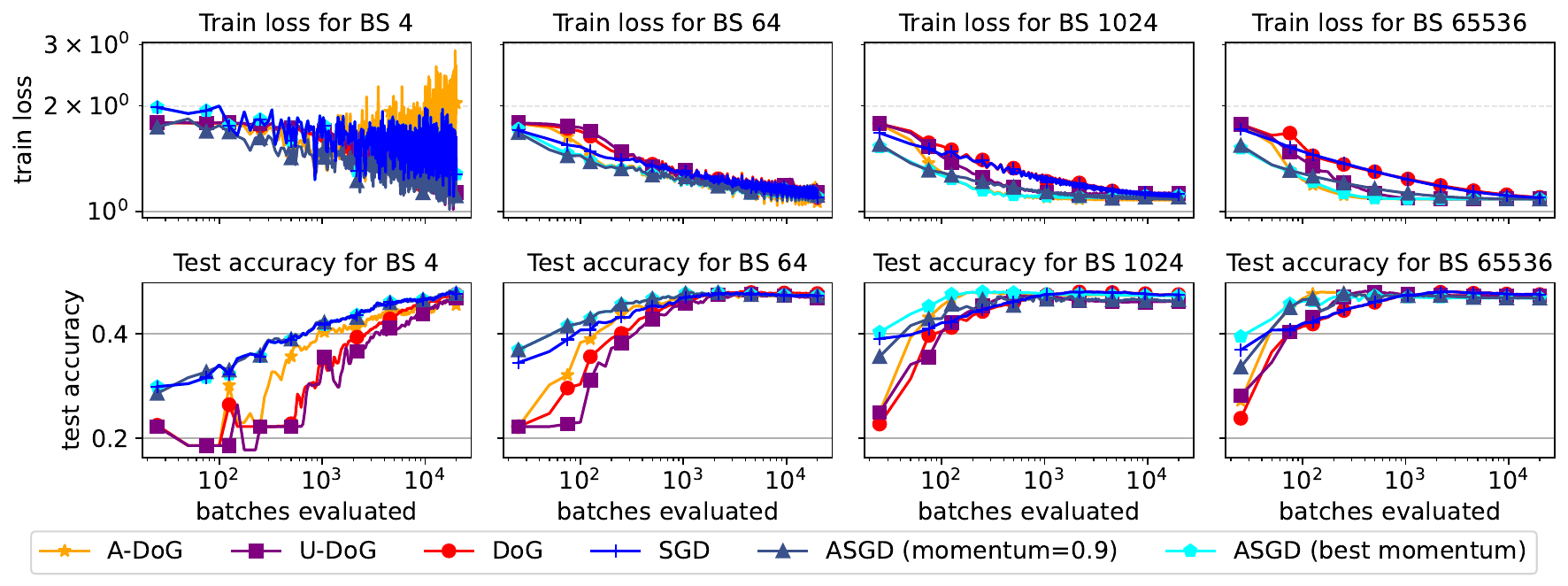}{\convexcaptioncurves{DMLab}{95}}{fig:convex-dmlab-curves}

\insertfigurepage{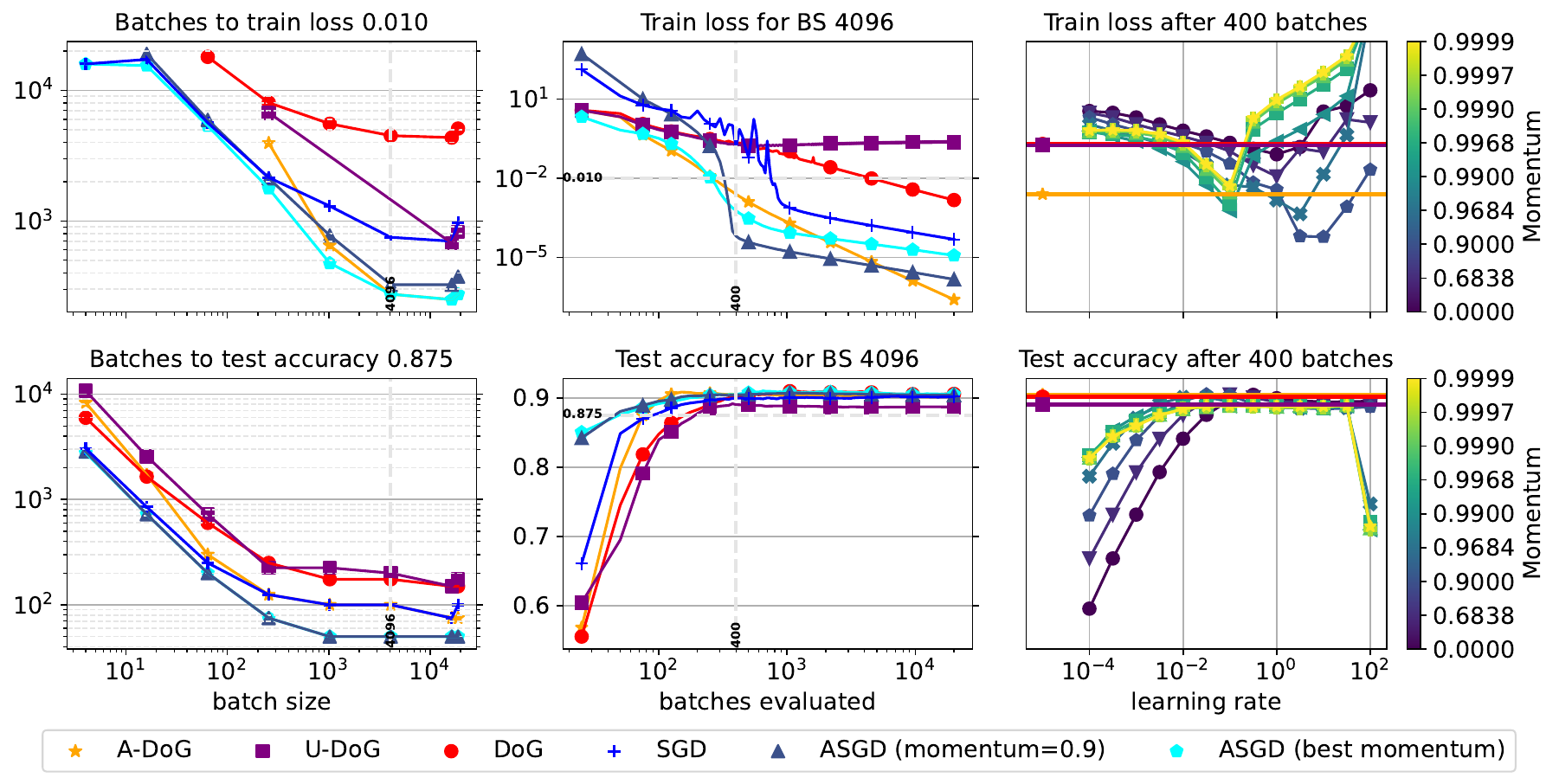}{\convexcaption{Resisc45}{90}}{fig:convex-resisc45}{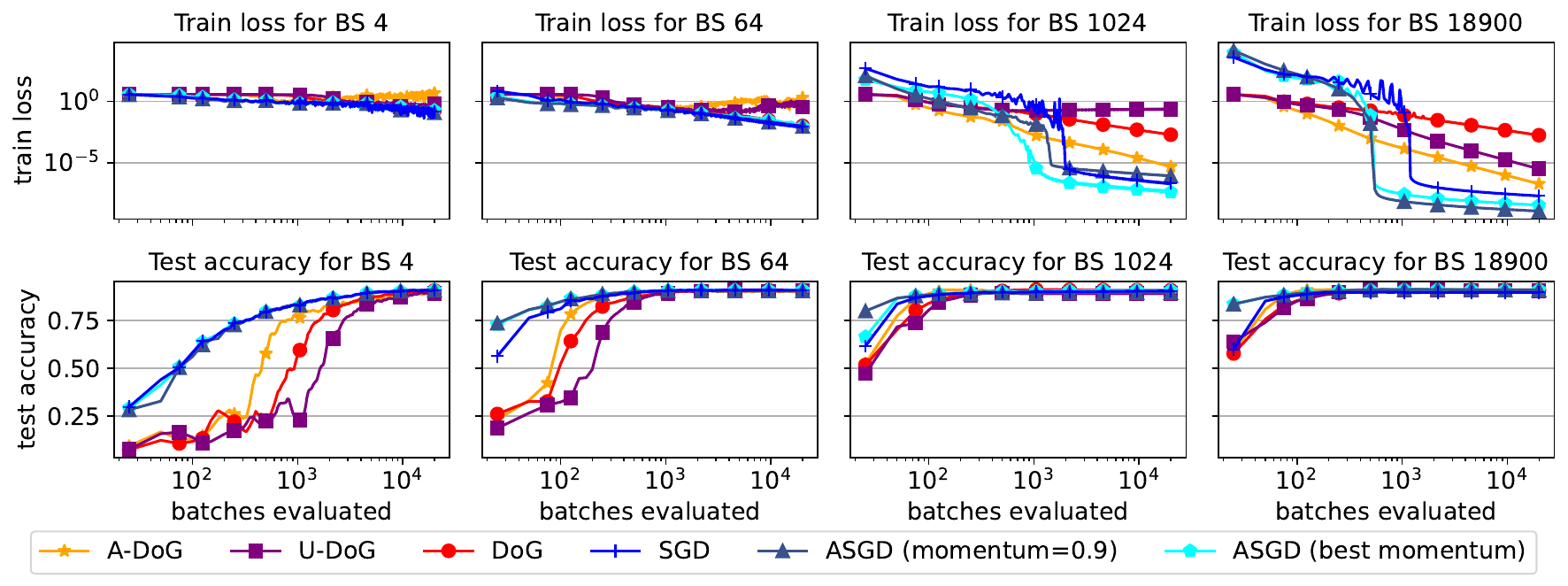}{\convexcaptioncurves{Resisc45}{95}}{fig:convex-resisc45-curves}

\insertfigurepage{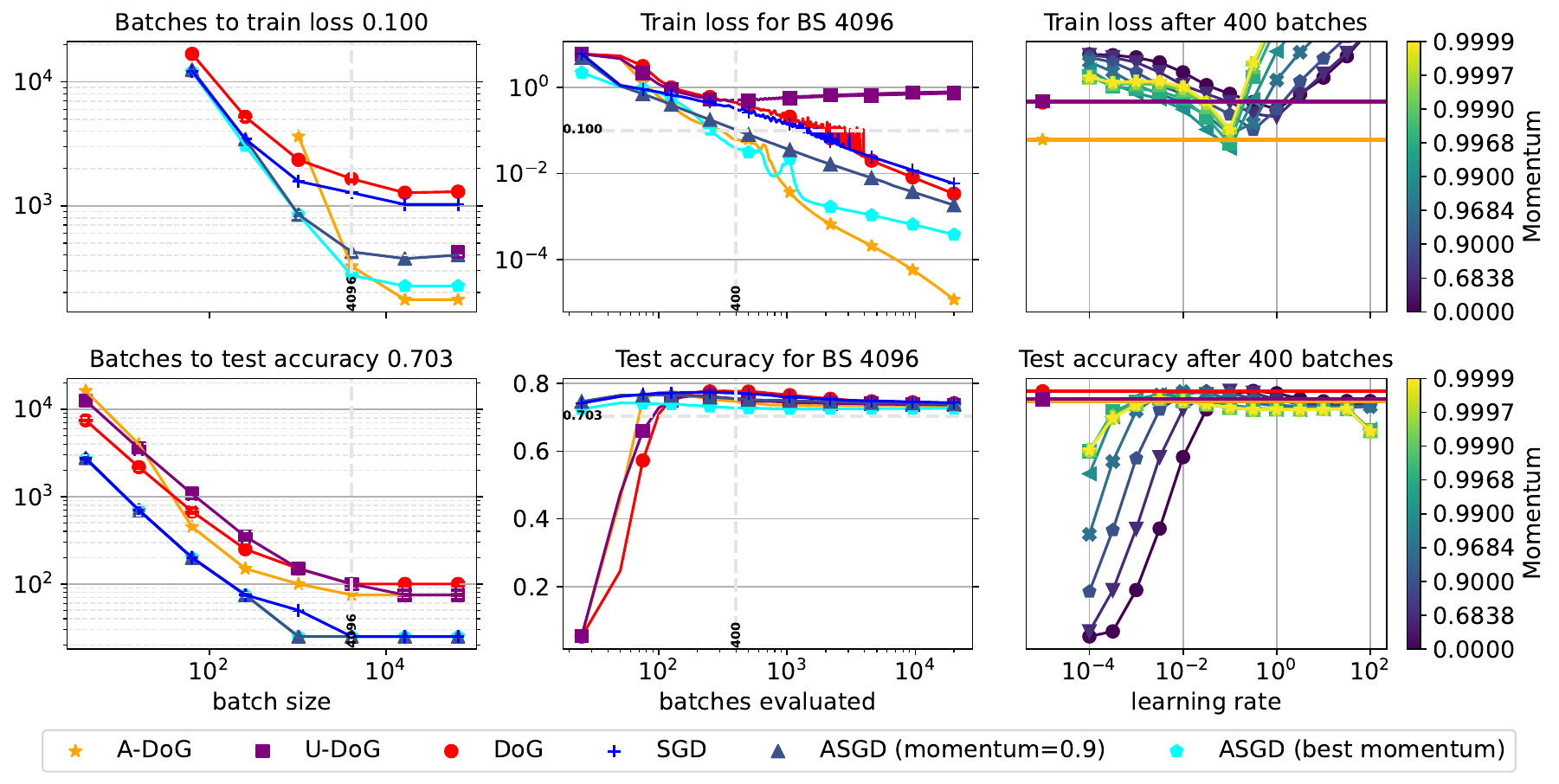}{\convexcaption{Sun397}{90}}{fig:convex-sun397}{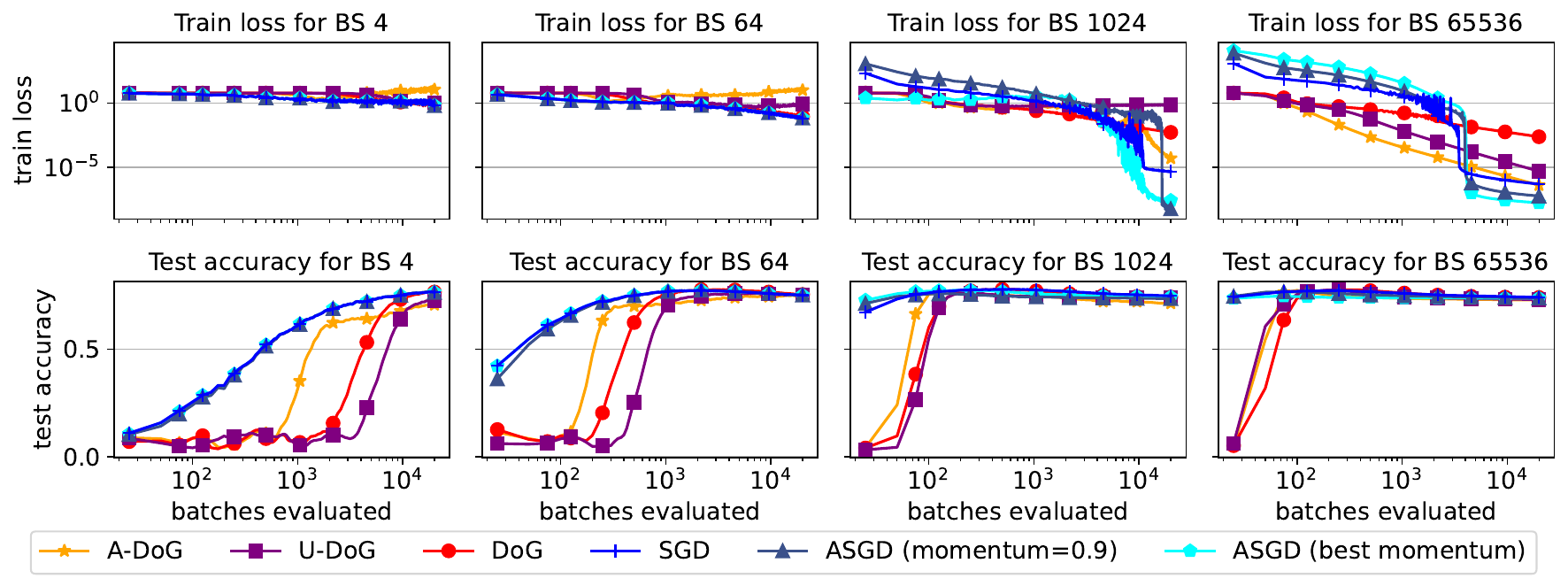}{\convexcaptioncurves{Sun397}{95}}{fig:convex-sun397-curves}

\insertfigurepage{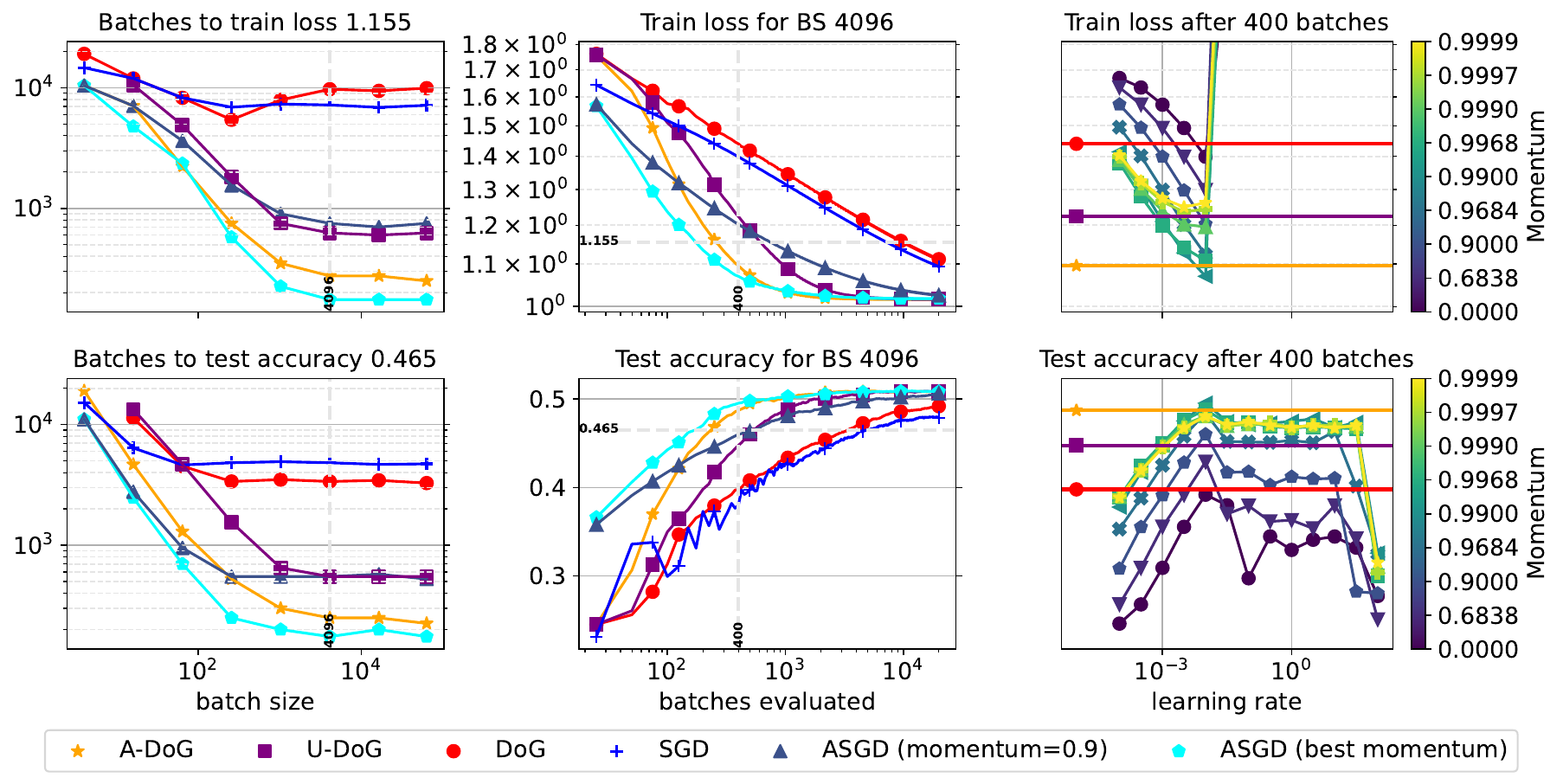}{\convexcaption{CLEVR-Dist}{85}}{fig:convex-clevr-distance}{convex_vtab+dmlab_vit_base_patch32_224_in21k_perc=0_95_training_curves.pdf}{\convexcaptioncurves{CLEVR-Dist}{95}}{fig:convex-clevr-distance-curves}

\insertfigurepage{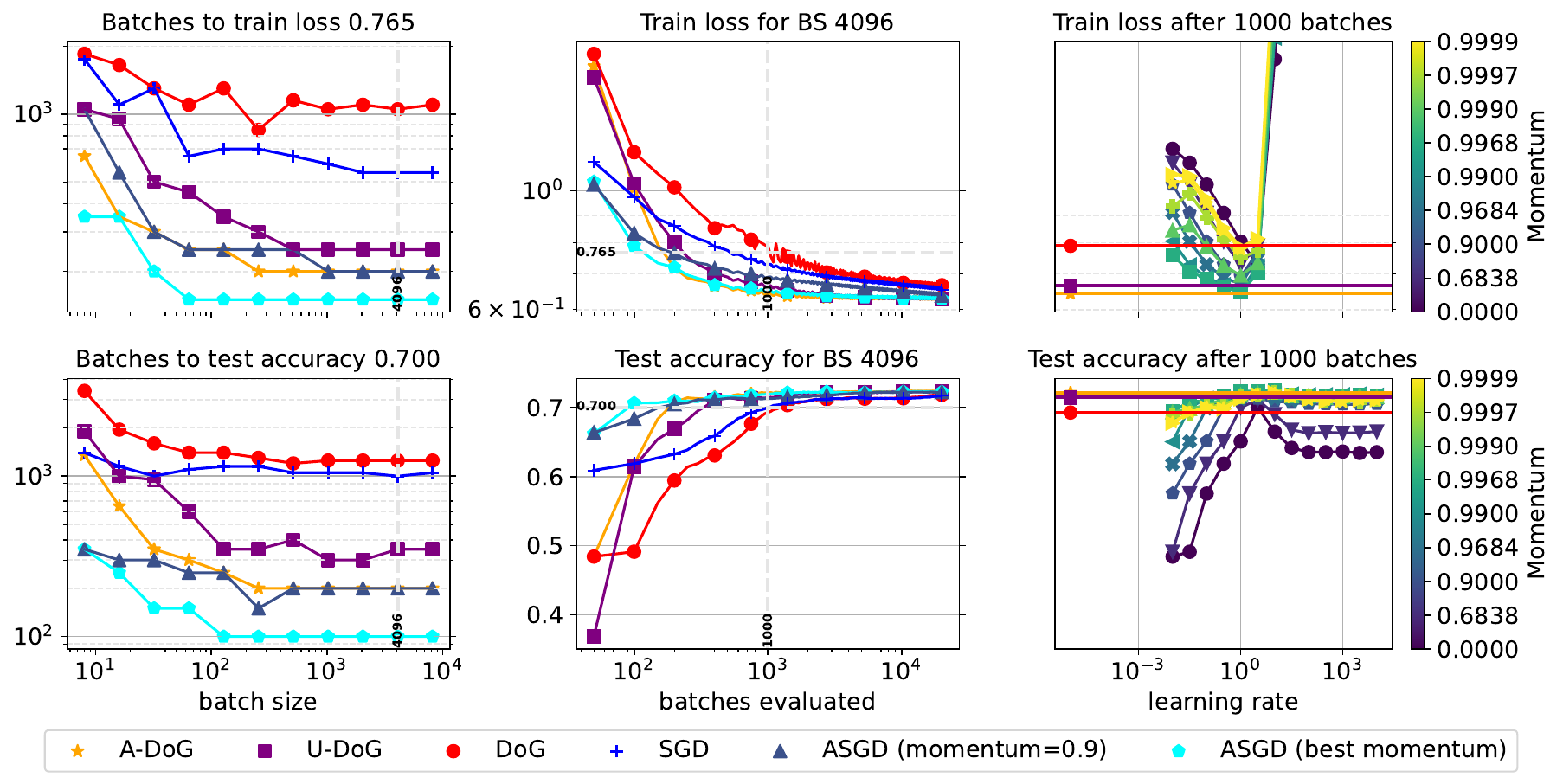}{\libsvmcaption{LIBSVM/CovertypeScale}{90}}{fig:convex-covertype}{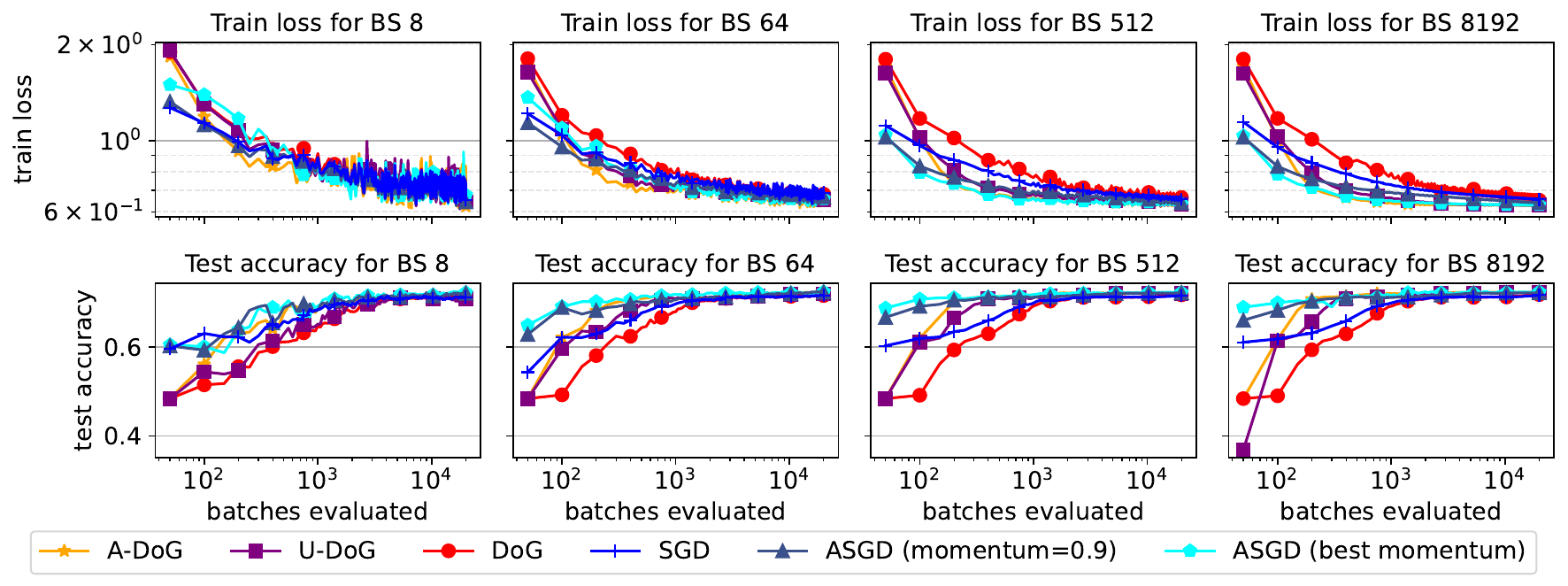}{\libsvmcaptioncurves{LIBSVM/CovertypeScale}{95}}{fig:convex-covertype-curves}

\insertfigurepage{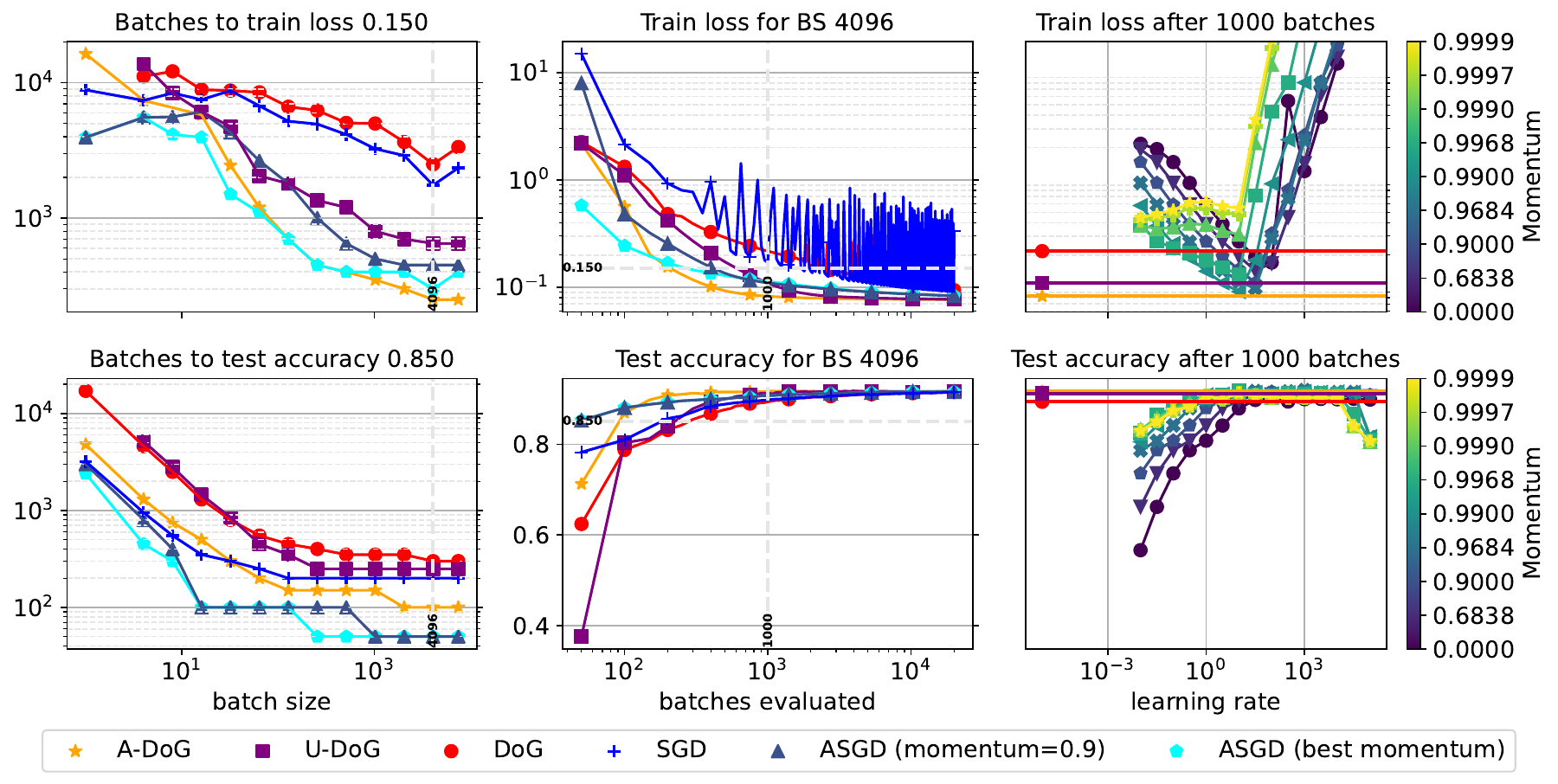}{\libsvmcaption{LIBSVM/Pendigits}{85}}{fig:convex-pendigits}{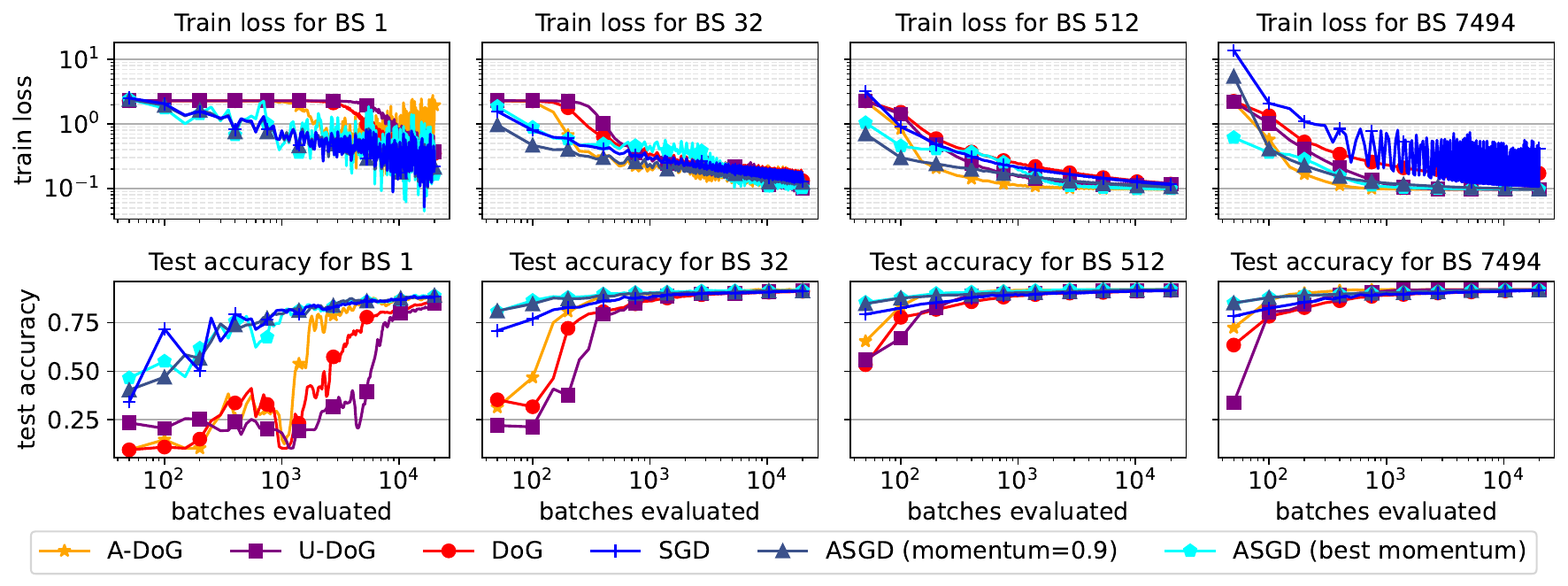}{\libsvmcaptioncurves{LIBSVM/Pendigits}{95} As most algorithms here fail to converge at reasonable rate, we use significantly lower targets to choose hyper-parameters.}{fig:convex-pendigits-curves}

\insertfigurepage{LS_vtab+svhn_vit_base_patch32_224_in21k_perc=-0_062,_0_633-_BS=4096.pdf}{\lscaption{SVHN}{90} This is the same as \Cref{fig:main-results}.}{fig:ls-svhn}{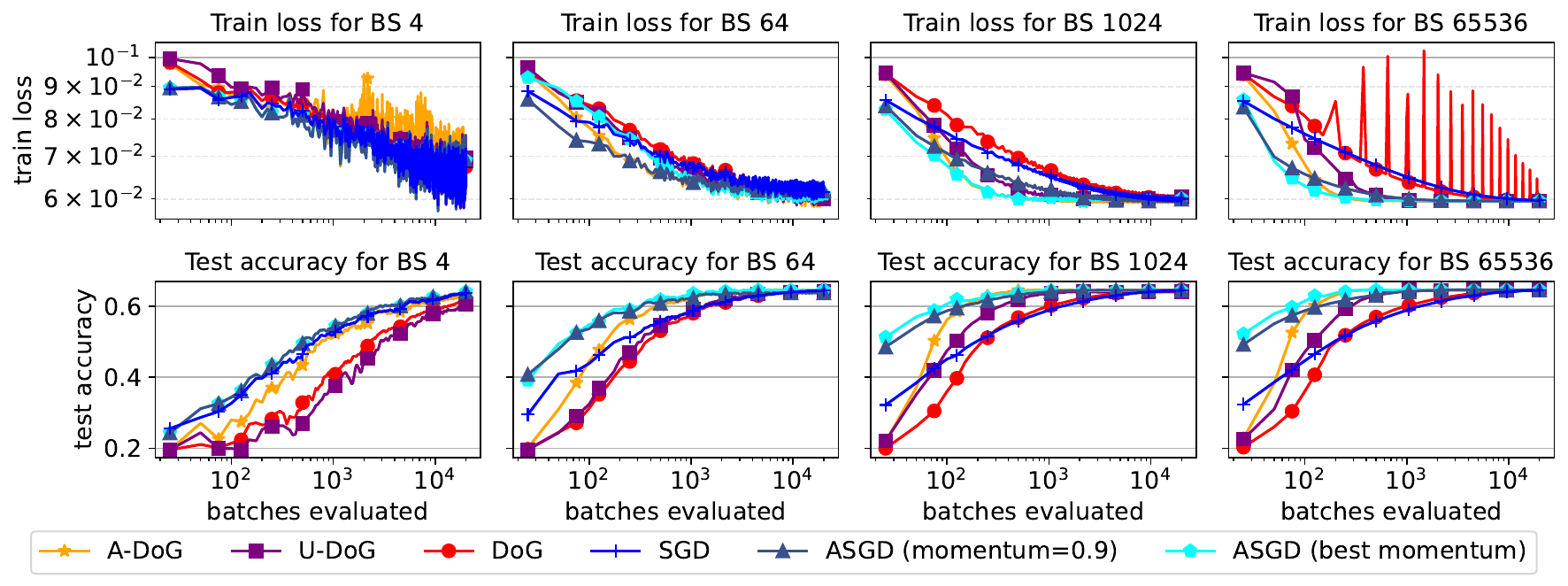}{\lscaptioncurves{SVHN}{95}}{fig:ls-svhn-curves}

\insertfigurepage{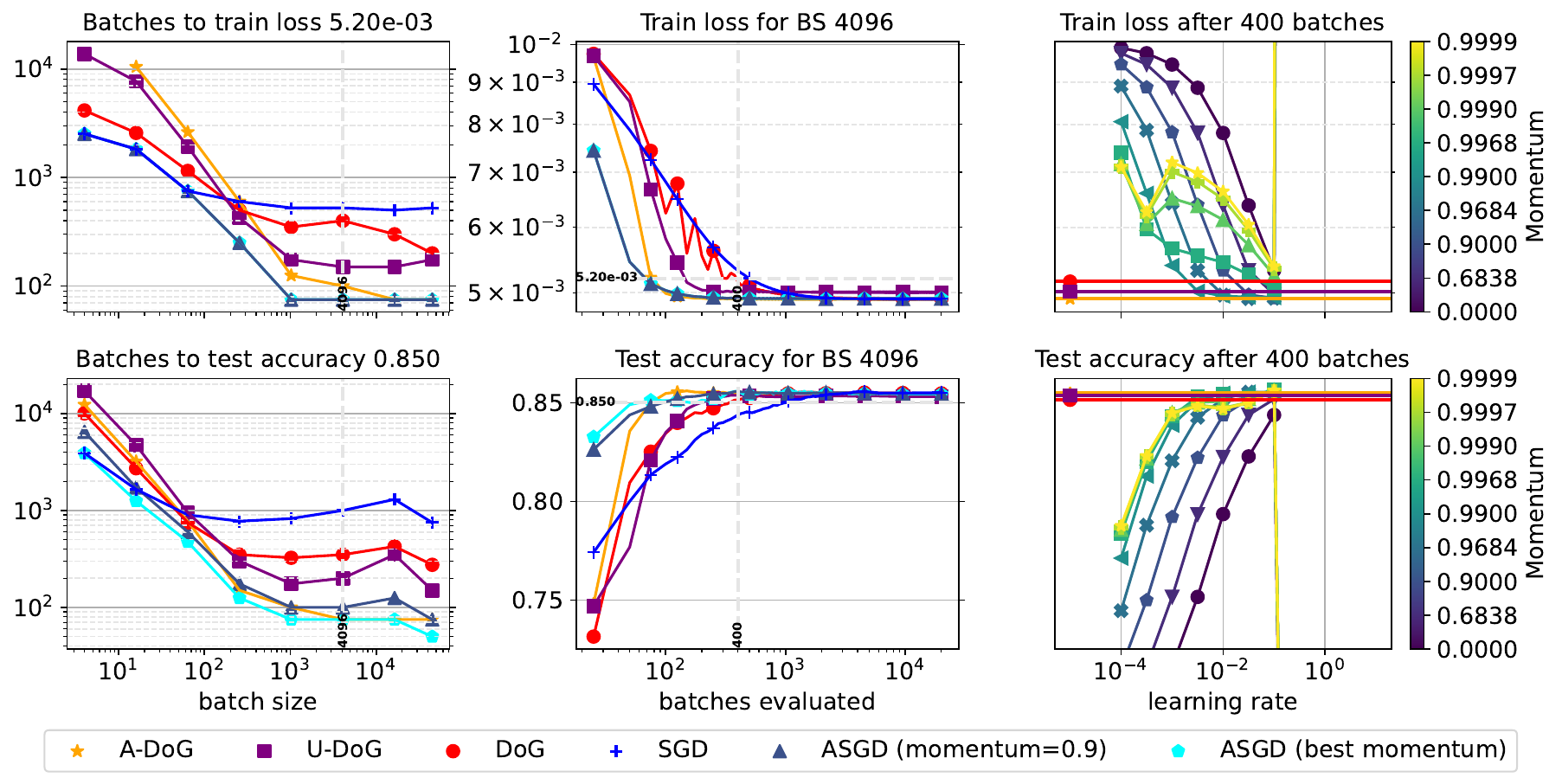}{\lscaption{CIFAR-100}{90}}{fig:ls-cifar100}{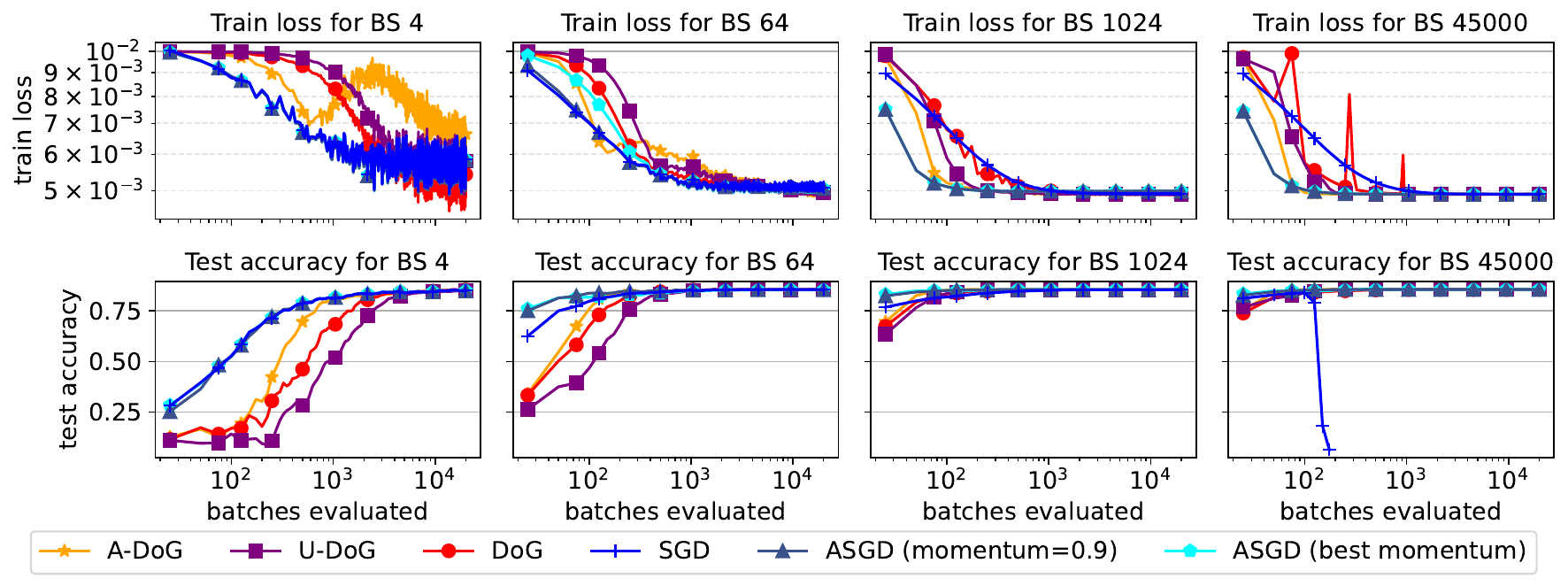}{\lscaptioncurves{CIFAR-100}{95}}{fig:ls-cifar100-curves}

\insertfigure{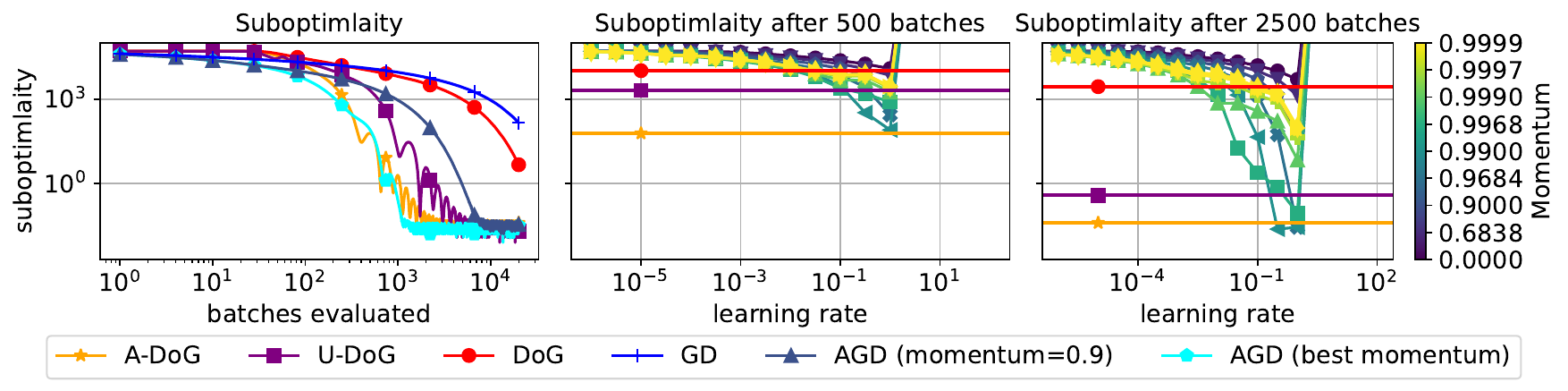}{Training a model on a noiseless quadratic problem. At larger base learning rates, all AGD variants diverge while $\DoG$ variants remain stable, and $\UDoG$ and $\ADoG$ perform especially well.}{fig:quadratic}

\insertfigure{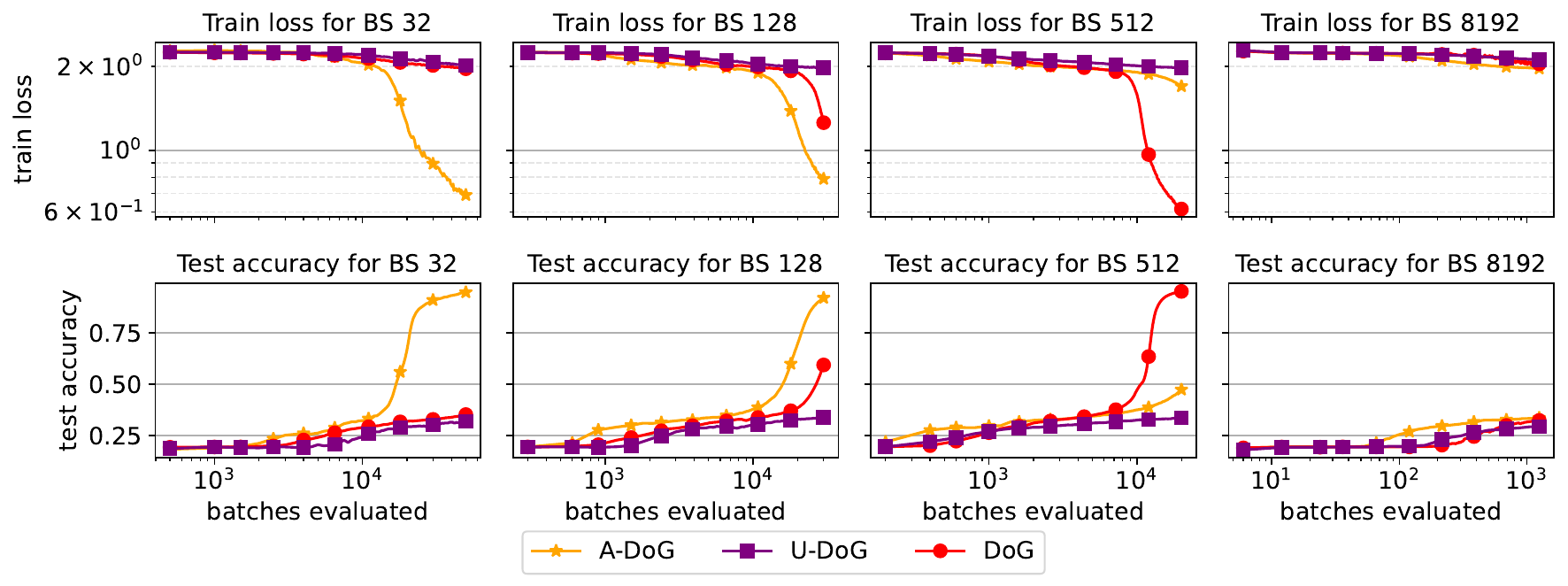}{\nonconvexcaptioncurves{ResNet50}{SVHN}}{fig:resnet50-svhn}

\insertfigure{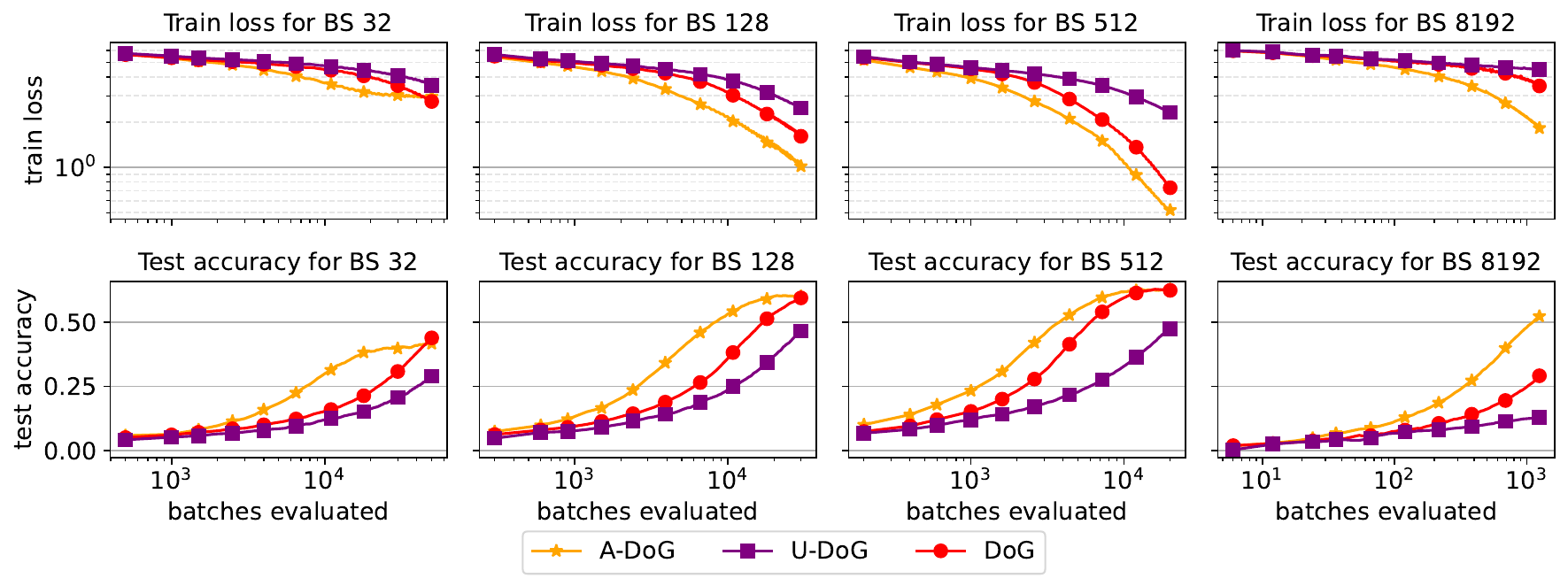}{\nonconvexcaptioncurves{ResNet50}{Sun397}}{fig:resnet50-sun397}

\insertfigure{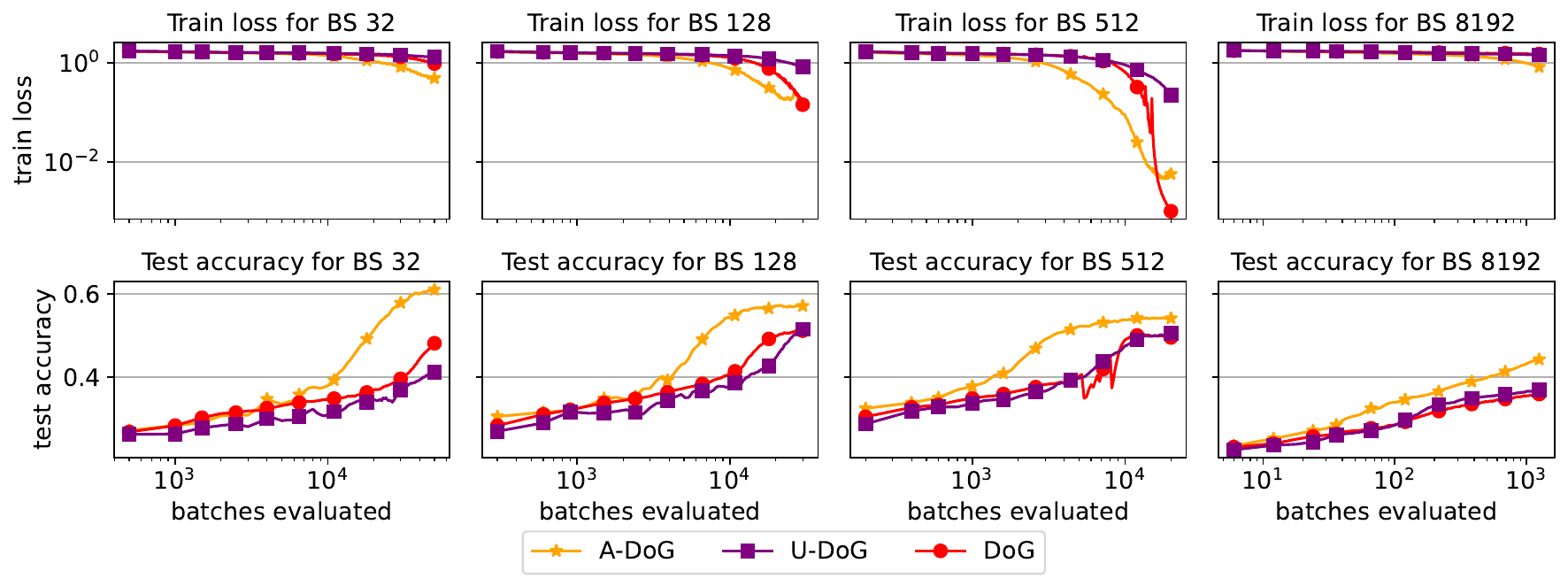}{\nonconvexcaptioncurves{ResNet50}{DMLab}}{fig:resnet50-dmlab}

\insertfigure{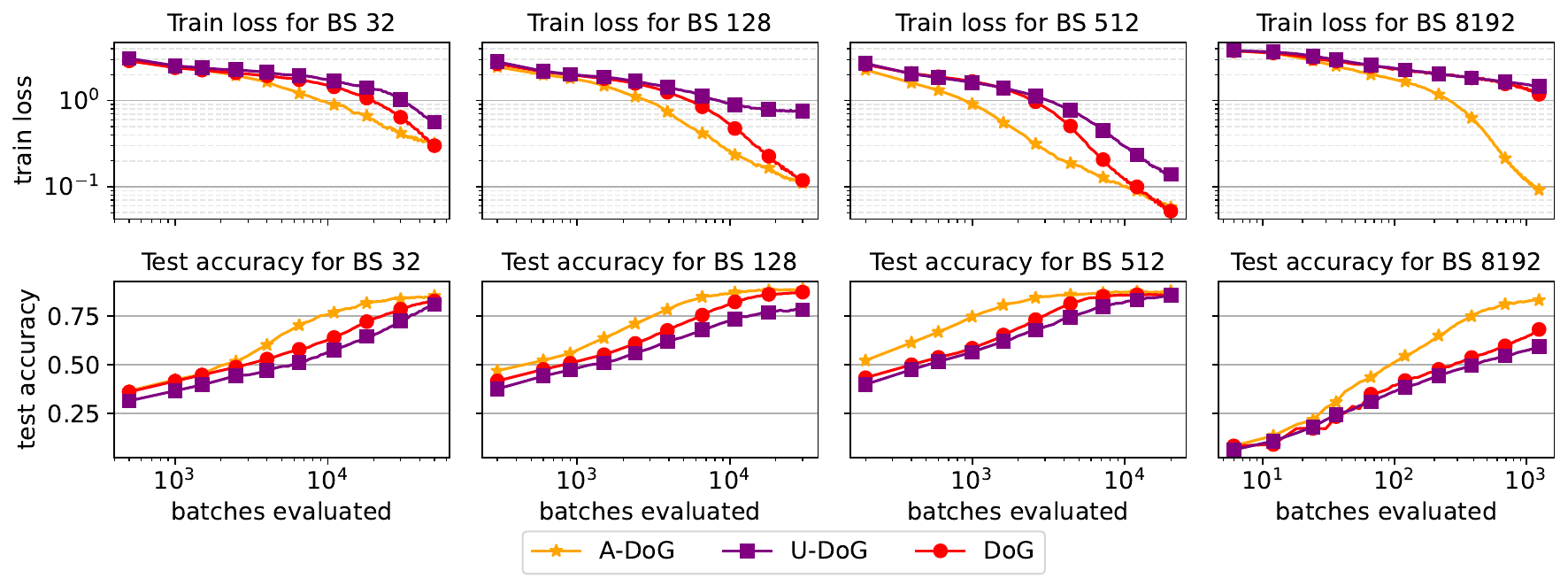}{\nonconvexcaptioncurves{ResNet50}{Resisc45}}{fig:resnet50-resisc45}

\insertfigure{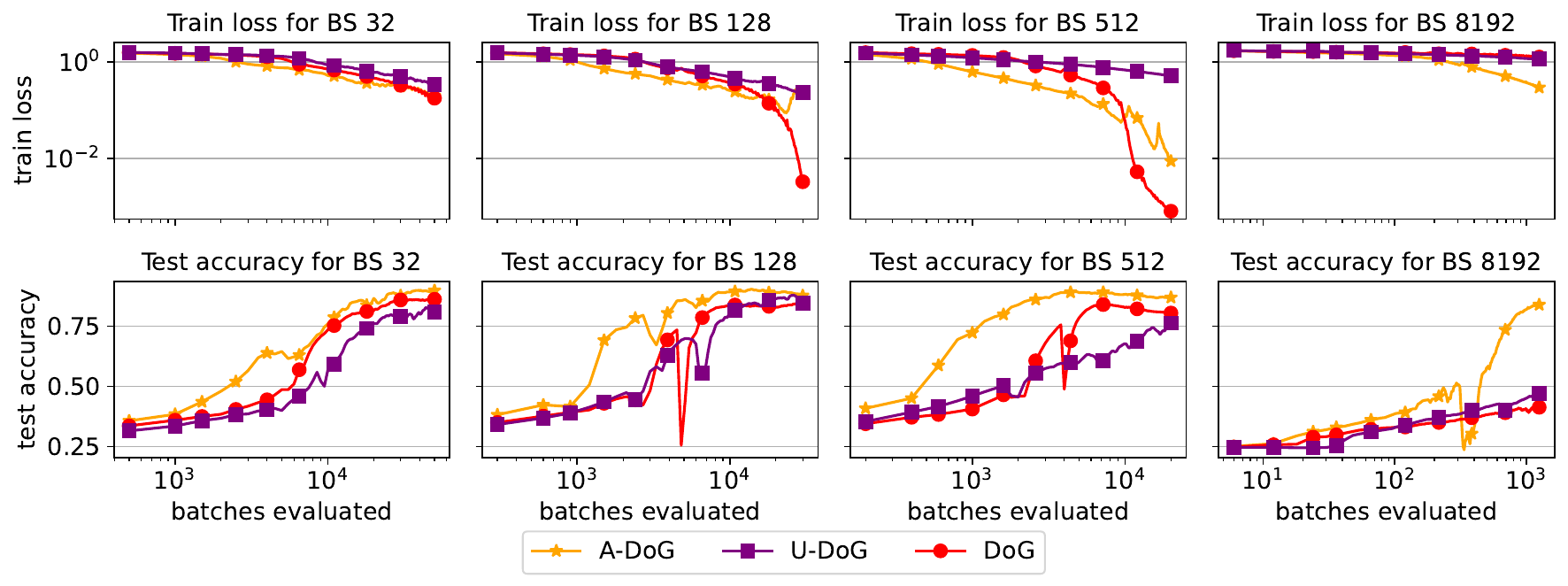}{\nonconvexcaptioncurves{ResNet50}{CLEVR-Dist}}{fig:resnet50-clevr-distance}

\insertfigure{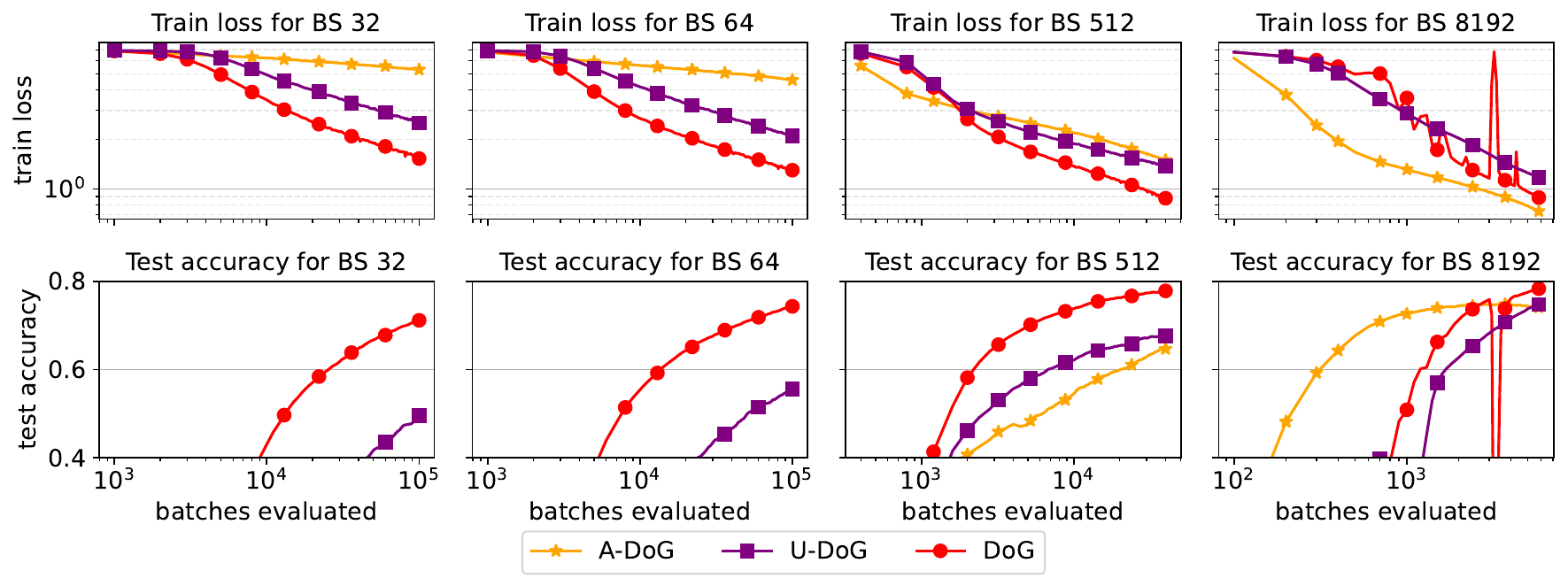}{Fine-tuning a Clip-ViT-B/32 model on ImageNet, at different batch sizes. Top: Loss vs. step training curve for different batch sizes. Bottom: Test accuracy vs. step curve for averaged iterates at varied batch sizes.}{fig:clip-imagenet}

\insertfigure{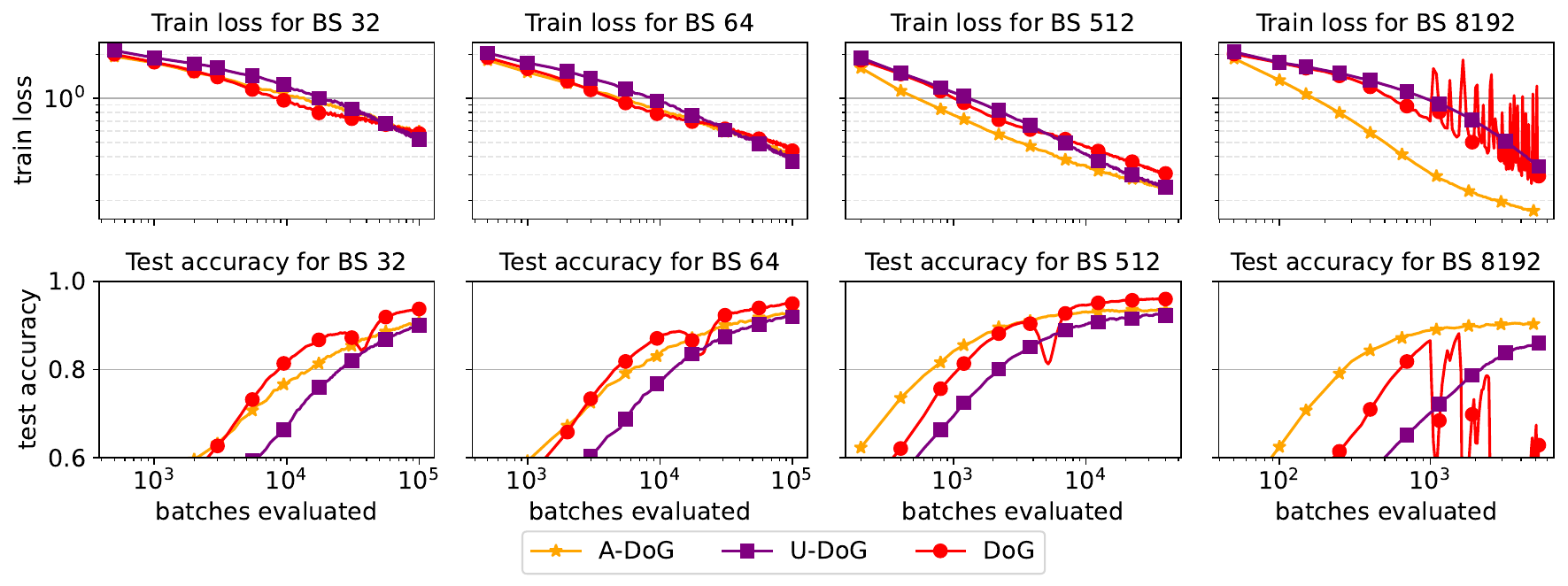}{Training a Wide-ResNet-28-10 model on CIFAR-10 from scratch, at different batch sizes. Top: Loss vs. step training curve for different batch sizes. Bottom: Test accuracy vs. step curve for averaged iterates at varied batch sizes.}{fig:wrn28-cifar10} 

\begin{thebibliography}{67}
\providecommand{\natexlab}[1]{#1}
\providecommand{\url}[1]{\texttt{#1}}
\expandafter\ifx\csname urlstyle\endcsname\relax
  \providecommand{\doi}[1]{doi: #1}\else
  \providecommand{\doi}{doi: \begingroup \urlstyle{rm}\Url}\fi

\bibitem[Abadi et~al.(2015)Abadi, Agarwal, Barham, Brevdo, Chen, Citro,
  Corrado, Davis, Dean, Devin, Ghemawat, Goodfellow, Harp, Irving, Isard, Jia,
  Jozefowicz, Kaiser, Kudlur, Levenberg, Man\'{e}, Monga, Moore, Murray, Olah,
  Schuster, Shlens, Steiner, Sutskever, Talwar, Tucker, Vanhoucke, Vasudevan,
  Vi\'{e}gas, Vinyals, Warden, Wattenberg, Wicke, Yu, and
  Zheng]{Abadi2015Tensorflow}
Mart\'{i}n Abadi, Ashish Agarwal, Paul Barham, Eugene Brevdo, Zhifeng Chen,
  Craig Citro, Greg~S. Corrado, Andy Davis, Jeffrey Dean, Matthieu Devin,
  Sanjay Ghemawat, Ian Goodfellow, Andrew Harp, Geoffrey Irving, Michael Isard,
  Yangqing Jia, Rafal Jozefowicz, Lukasz Kaiser, Manjunath Kudlur, Josh
  Levenberg, Dandelion Man\'{e}, Rajat Monga, Sherry Moore, Derek Murray, Chris
  Olah, Mike Schuster, Jonathon Shlens, Benoit Steiner, Ilya Sutskever, Kunal
  Talwar, Paul Tucker, Vincent Vanhoucke, Vijay Vasudevan, Fernanda Vi\'{e}gas,
  Oriol Vinyals, Pete Warden, Martin Wattenberg, Martin Wicke, Yuan Yu, and
  Xiaoqiang Zheng.
\newblock {TensorFlow}: Large-scale machine learning on heterogeneous systems,
  2015.

\bibitem[Alpaydin and Alimoglu(1998)]{alpaydin1998pen}
E.~Alpaydin and Fevzi. Alimoglu.
\newblock {Pen-Based Recognition of Handwritten Digits}.
\newblock UCI Machine Learning Repository, 1998.
\newblock {DOI}: https://doi.org/10.24432/C5MG6K.

\bibitem[Attia and Koren(2023)]{attia2023sgd}
Amit Attia and Tomer Koren.
\newblock {SGD} with {A}da{G}rad stepsizes: Full adaptivity with high
  probability to unknown parameters, unbounded gradients and affine variance.
\newblock In \emph{International Conference on Machine Learning (ICML)}, 2023.

\bibitem[Attia and Koren(2024)]{attia2024free}
Amit Attia and Tomer Koren.
\newblock How free is parameter-free stochastic optimization?
\newblock In \emph{International Conference on Machine Learning (ICML)}, 2024.

\bibitem[Beattie et~al.(2016)Beattie, Leibo, Teplyashin, Ward, Wainwright,
  K{\"u}ttler, Lefrancq, Green, Vald{\'e}s, Sadik, et~al.]{beattie2016deepmind}
Charles Beattie, Joel~Z Leibo, Denis Teplyashin, Tom Ward, Marcus Wainwright,
  Heinrich K{\"u}ttler, Andrew Lefrancq, Simon Green, V{\'\i}ctor Vald{\'e}s,
  Amir Sadik, et~al.
\newblock Deepmind lab.
\newblock \emph{arXiv:1612.03801}, 2016.

\bibitem[Beck and Teboulle(2009)]{beck2009fast}
Amir Beck and Marc Teboulle.
\newblock A fast iterative shrinkage-thresholding algorithm for linear inverse
  problems.
\newblock \emph{SIAM journal on imaging sciences}, 2\penalty0 (1):\penalty0
  183--202, 2009.

\bibitem[Bhaskara et~al.(2020)Bhaskara, Cutkosky, Kumar, and
  Purohit]{bhaskara2020online}
Aditya Bhaskara, Ashok Cutkosky, Ravi Kumar, and Manish Purohit.
\newblock Online learning with imperfect hints.
\newblock In \emph{International Conference on Machine Learning (ICML)}, 2020.

\bibitem[Blackard(1998)]{blackard1998covertype}
Jock Blackard.
\newblock {Covertype}.
\newblock UCI Machine Learning Repository, 1998.
\newblock {DOI}: https://doi.org/10.24432/C50K5N.

\bibitem[Carmon and Hinder(2022)]{carmon2022making}
Yair Carmon and Oliver Hinder.
\newblock Making {SGD} parameter-free.
\newblock In \emph{Conference on Learning Theory (COLT)}, 2022.

\bibitem[Carmon and Hinder(2024)]{carmon2024price}
Yair Carmon and Oliver Hinder.
\newblock The price of adaptivity in stochastic convex optimization.
\newblock In \emph{Conference on Learning Theory (COLT)}, 2024.

\bibitem[Carmon et~al.(2022)Carmon, Hausler, Jambulapati, Jin, and
  Sidford]{carmon2022optimal}
Yair Carmon, Danielle Hausler, Arun Jambulapati, Yujia Jin, and Aaron Sidford.
\newblock Optimal and adaptive monteiro-svaiter acceleration.
\newblock In \emph{Advances in Neural Information Processing Systems
  (NeurIPS)}, 2022.

\bibitem[Chang and Lin(2011)]{libsvm}
Chih-Chung Chang and Chih-Jen Lin.
\newblock {LIBSVM}: a library for support vector machines.
\newblock \emph{ACM Transactions on Intelligent Systems and Technology}, 2011.

\bibitem[Cheng et~al.(2017)Cheng, Han, and Lu]{cheng2017remote}
Gong Cheng, Junwei Han, and Xiaoqiang Lu.
\newblock Remote sensing image scene classification: Benchmark and state of the
  art.
\newblock \emph{Proceedings of the IEEE}, 105\penalty0 (10):\penalty0
  1865--1883, 2017.

\bibitem[Cutkosky(2019{\natexlab{a}})]{cutkosky2019anytime}
Ashok Cutkosky.
\newblock Anytime online-to-batch, optimism and acceleration.
\newblock In \emph{International Conference on Machine Learning (ICML)}, pages
  1446--1454, 2019{\natexlab{a}}.

\bibitem[Cutkosky(2019{\natexlab{b}})]{cutkosky2019artificial}
Ashok Cutkosky.
\newblock Artificial constraints and hints for unbounded online learning.
\newblock In \emph{Conference on Learning Theory (COLT)}, 2019{\natexlab{b}}.

\bibitem[Cutkosky and Orabona(2018)]{cutkosky2018black}
Ashok Cutkosky and Francesco Orabona.
\newblock Black-box reductions for parameter-free online learning in {B}anach
  spaces.
\newblock In \emph{Conference on Learning Theory (COLT)}, 2018.

\bibitem[Defazio and Mishchenko(2023)]{defazio2023learning}
Aaron Defazio and Konstantin Mishchenko.
\newblock Learning-rate-free learning by {D}-adaptation.
\newblock In \emph{International Conference on Machine Learning (ICML)}, 2023.

\bibitem[Deng et~al.(2009)Deng, Dong, Socher, Li, Li, and
  Fei-Fei]{deng2009imagenet}
Jia Deng, Wei Dong, Richard Socher, Li-Jia Li, Kai Li, and Li~Fei-Fei.
\newblock Image{N}et: A large-scale hierarchical image database.
\newblock In \emph{Conference on Computer Vision and Pattern Recognition
  (CVPR)}, 2009.

\bibitem[Diakonikolas and Orecchia(2018)]{diakonikolas2018accelerated}
Jelena Diakonikolas and Lorenzo Orecchia.
\newblock Accelerated extra-gradient descent: A novel accelerated first-order
  method.
\newblock In \emph{Innovations in Theoretical Computer Science (ITCS)}, 2018.

\bibitem[Dosovitskiy et~al.(2021)Dosovitskiy, Beyer, Kolesnikov, Weissenborn,
  Zhai, Unterthiner, Dehghani, Minderer, Heigold, Gelly, Uszkoreit, and
  Houlsby]{dosovitskiy2021image}
Alexey Dosovitskiy, Lucas Beyer, Alexander Kolesnikov, Dirk Weissenborn,
  Xiaohua Zhai, Thomas Unterthiner, Mostafa Dehghani, Matthias Minderer, Georg
  Heigold, Sylvain Gelly, Jakob Uszkoreit, and Neil Houlsby.
\newblock An image is worth 16x16 words: Transformers for image recognition at
  scale.
\newblock In \emph{International Conference on Learning Representations
  (ICLR)}, 2021.

\bibitem[Duchi et~al.(2011)Duchi, Hazan, and Singer]{duchi2011adaptive}
John Duchi, Elad Hazan, and Yoram Singer.
\newblock Adaptive subgradient methods for online learning and stochastic
  optimization.
\newblock \emph{Journal of Machine Learning Research}, 12\penalty0 (7), 2011.

\bibitem[Gupta et~al.(2017)Gupta, Koren, and Singer]{gupta2017unified}
Vineet Gupta, Tomer Koren, and Yoram Singer.
\newblock A unified approach to adaptive regularization in online and
  stochastic optimization.
\newblock \emph{arXiv:1706.06569}, 2017.

\bibitem[Harris et~al.(2020)Harris, Millman, van~der Walt, Gommers, Virtanen,
  Cournapeau, Wieser, Taylor, Berg, Smith, Kern, Picus, Hoyer, van Kerkwijk,
  Brett, Haldane, del R{\'{i}}o, Wiebe, Peterson, G{\'{e}}rard-Marchant,
  Sheppard, Reddy, Weckesser, Abbasi, Gohlke, and Oliphant]{Harris2020Array}
Charles~R. Harris, K.~Jarrod Millman, St{\'{e}}fan~J. van~der Walt, Ralf
  Gommers, Pauli Virtanen, David Cournapeau, Eric Wieser, Julian Taylor,
  Sebastian Berg, Nathaniel~J. Smith, Robert Kern, Matti Picus, Stephan Hoyer,
  Marten~H. van Kerkwijk, Matthew Brett, Allan Haldane, Jaime~Fern{\'{a}}ndez
  del R{\'{i}}o, Mark Wiebe, Pearu Peterson, Pierre G{\'{e}}rard-Marchant,
  Kevin Sheppard, Tyler Reddy, Warren Weckesser, Hameer Abbasi, Christoph
  Gohlke, and Travis~E. Oliphant.
\newblock Array programming with {NumPy}.
\newblock \emph{Nature}, 585\penalty0 (7825):\penalty0 357--362, 2020.

\bibitem[He et~al.(2016)He, Zhang, Ren, and Sun]{he2015deep}
Kaiming He, X.~Zhang, Shaoqing Ren, and Jian Sun.
\newblock Deep residual learning for image recognition.
\newblock \emph{Conference on Computer Vision and Pattern Recognition (CVPR)},
  2016.

\bibitem[Howard et~al.(2020)Howard, Ramdas, McAuliffe, and
  Sekhon]{howard2020time}
Steven~R Howard, Aaditya Ramdas, Jon McAuliffe, and Jasjeet Sekhon.
\newblock Time-uniform chernoff bounds via nonnegative supermartingales.
\newblock \emph{Probability Surveys}, 17:\penalty0 257--317, 2020.

\bibitem[Howard et~al.(2021)Howard, Ramdas, McAuliffe, and
  Sekhon]{howard2021time}
Steven~R Howard, Aaditya Ramdas, Jon McAuliffe, and Jasjeet Sekhon.
\newblock Time-uniform, nonparametric, nonasymptotic confidence sequences.
\newblock \emph{The Annals of Statistics}, 49\penalty0 (2):\penalty0
  1055--1080, 2021.

\bibitem[Ivgi et~al.(2023)Ivgi, Hinder, and Carmon]{ivgi2023dog}
Maor Ivgi, Oliver Hinder, and Yair Carmon.
\newblock {D}o{G} is {SGD}'s best friend: A parameter-free dynamic step size
  schedule.
\newblock In \emph{International Conference on Machine Learning (ICML)}, 2023.
\newblock We refer to the latest arXiv version:
  \url{https://arxiv.org/abs/2302.12022}.

\bibitem[Jacobsen and Cutkosky(2022)]{jacobsen2022parameter}
Andrew Jacobsen and Ashok Cutkosky.
\newblock Parameter-free mirror descent.
\newblock In \emph{Conference on Learning Theory (COLT)}, 2022.

\bibitem[Johnson et~al.(2017)Johnson, Hariharan, Van Der~Maaten, Fei-Fei,
  Lawrence~Zitnick, and Girshick]{johnson2017clevr}
Justin Johnson, Bharath Hariharan, Laurens Van Der~Maaten, Li~Fei-Fei,
  C~Lawrence~Zitnick, and Ross Girshick.
\newblock {CLEVR}: A diagnostic dataset for compositional language and
  elementary visual reasoning.
\newblock In \emph{Conference on Computer Vision and Pattern Recognition
  (CVPR)}, 2017.

\bibitem[Kavis et~al.(2019)Kavis, Levy, Bach, and Cevher]{kavis2019unixgrad}
Ali Kavis, Kfir~Y Levy, Francis Bach, and Volkan Cevher.
\newblock Uni{XG}rad: A universal, adaptive algorithm with optimal guarantees
  for constrained optimization.
\newblock \emph{Advances in Neural Information Processing Systems (NeurIPS)},
  2019.

\bibitem[Khaled and Jin(2024)]{khaled2024tuning}
Ahmed Khaled and Chi Jin.
\newblock Tuning-free stochastic optimization.
\newblock In \emph{International Conference on Machine Learning (ICML)}, 2024.

\bibitem[Khaled et~al.(2023)Khaled, Mishchenko, and Jin]{khaled2023dowg}
Ahmed Khaled, Konstantin Mishchenko, and Chi Jin.
\newblock {DoWG} unleashed: An efficient universal parameter-free gradient
  descent method.
\newblock In \emph{Advances in Neural Information Processing Systems
  (NeurIPS)}, 2023.

\bibitem[Kingma and Ba(2015)]{kingma2015adam}
Diederik~P Kingma and Jimmy Ba.
\newblock {ADAM}: A method for stochastic optimization.
\newblock In \emph{International Conference on Learning Representations
  (ICLR)}, 2015.

\bibitem[Krizhevsky(2009)]{krizhevsky2009learning}
Alex Krizhevsky.
\newblock Learning multiple layers of features from tiny images.
\newblock Technical report, University of Toronto, 2009.

\bibitem[Lan(2012)]{lan2012optimal}
Guanghui Lan.
\newblock An optimal method for stochastic composite optimization.
\newblock \emph{Mathematical Programming}, 133\penalty0 (1):\penalty0 365--397,
  2012.

\bibitem[Levy et~al.(2018)Levy, Yurtsever, and Cevher]{levy2018online}
Kfir~Y Levy, Alp Yurtsever, and Volkan Cevher.
\newblock Online adaptive methods, universality and acceleration.
\newblock \emph{Advances in Neural Information Processing Systems (NeurIPS)},
  2018.

\bibitem[Loshchilov and Hutter(2017)]{loshchilov2016sgdr}
Ilya Loshchilov and Frank Hutter.
\newblock {SGDR}: Stochastic gradient descent with warm restarts.
\newblock \emph{International Conference on Learning Representations}, 2017.

\bibitem[McMahan(2017)]{mcmahan2017survey}
H~Brendan McMahan.
\newblock A survey of algorithms and analysis for adaptive online learning.
\newblock \emph{The Journal of Machine Learning Research}, 18\penalty0
  (1):\penalty0 3117--3166, 2017.

\bibitem[McMahan and Orabona(2014)]{mcmahan2014unconstrained}
H~Brendan McMahan and Francesco Orabona.
\newblock Unconstrained online linear learning in {H}ilbert spaces: Minimax
  algorithms and normal approximations.
\newblock In \emph{Conference on Learning Theory (COLT)}, 2014.

\bibitem[McMahan and Streeter(2010)]{mcmahan2010adaptive}
H~Brendan McMahan and Matthew Streeter.
\newblock Adaptive bound optimization for online convex optimization.
\newblock \emph{arXiv:1002.4908}, 2010.

\bibitem[Mhammedi and Koolen(2020)]{mhammedi2020lipschitz}
Zakaria Mhammedi and Wouter~M Koolen.
\newblock {L}ipschitz and comparator-norm adaptivity in online learning.
\newblock In \emph{Conference on Learning Theory (COLT)}, 2020.

\bibitem[Mishchenko and Defazio(2023)]{mishchenko2023prodigy}
Konstantin Mishchenko and Aaron Defazio.
\newblock Prodigy: An expeditiously adaptive parameter-free learner.
\newblock \emph{arXiv:2306.06101}, 2023.

\bibitem[Nemirovski(2004)]{nemirovski2004prox}
Arkadi Nemirovski.
\newblock Prox-method with rate of convergence $o(1/t)$ for variational
  inequalities with {L}ipschitz continuous monotone operators and smooth
  convex-concave saddle point problems.
\newblock \emph{SIAM Journal on Optimization}, 15\penalty0 (1):\penalty0
  229--251, 2004.

\bibitem[Nesterov(1983)]{nesterov1983method}
Yurii Nesterov.
\newblock A method of solving a convex programming problem with convergence
  rate ${O}(1/k^2)$.
\newblock \emph{Soviet Mathematics Doklady}, 27\penalty0 (2):\penalty0
  372--376, 1983.

\bibitem[Nesterov(2013)]{nesterov2013introductory}
Yurii Nesterov.
\newblock \emph{Introductory {L}ectures on {C}onvex {O}ptimization: A {B}asic
  {C}ourse}, volume~87.
\newblock Springer Science \& Business Media, 2013.

\bibitem[Netzer et~al.(2011)Netzer, Wang, Coates, Bissacco, Wu, and
  Ng]{netzer2011reading}
Yuval Netzer, Tao Wang, Adam Coates, A.~Bissacco, Bo~Wu, and A.~Ng.
\newblock Reading digits in natural images with unsupervised feature learning.
\newblock In \emph{NIPS Workshop on Deep Learning and Unsupervised Feature
  Learning 2011}, 2011.

\bibitem[Orabona(2013)]{orabona2013dimension}
Francesco Orabona.
\newblock Dimension-free exponentiated gradient.
\newblock \emph{Advances in Neural Information Processing Systems (NeurIPS)},
  2013.

\bibitem[Orabona(2021)]{orabona2021modern}
Francesco Orabona.
\newblock A modern introduction to online learning.
\newblock \emph{arXiv:1912.13213}, 2021.

\bibitem[Orabona and P{\'a}l(2016)]{orabona2016coin}
Francesco Orabona and D{\'a}vid P{\'a}l.
\newblock Coin betting and parameter-free online learning.
\newblock In \emph{Advances in Neural Information Processing Systems
  (NeurIPS)}, 2016.

\bibitem[Paquette and Scheinberg(2020)]{paquette2020stochastic}
Courtney Paquette and Katya Scheinberg.
\newblock A stochastic line search method with expected complexity analysis.
\newblock \emph{SIAM Journal on Optimization}, 30\penalty0 (1):\penalty0
  349--376, 2020.

\bibitem[Paszke et~al.(2019)Paszke, Gross, Massa, Lerer, Bradbury, Chanan,
  Killeen, Lin, Gimelshein, Antiga, Desmaison, Kopf, Yang, DeVito, Raison,
  Tejani, Chilamkurthy, Steiner, Fang, Bai, and Chintala]{paszke2019pytorch}
Adam Paszke, Sam Gross, Francisco Massa, Adam Lerer, James Bradbury, Gregory
  Chanan, Trevor Killeen, Zeming Lin, Natalia Gimelshein, Luca Antiga, Alban
  Desmaison, Andreas Kopf, Edward Yang, Zachary DeVito, Martin Raison, Alykhan
  Tejani, Sasank Chilamkurthy, Benoit Steiner, Lu~Fang, Junjie Bai, and Soumith
  Chintala.
\newblock Py{T}orch: An imperative style, high-performance deep learning
  library.
\newblock In \emph{Advances in Neural Information Processing Systems
  (NeurIPS)}, 2019.

\bibitem[Pedregosa et~al.(2011)Pedregosa, Varoquaux, Gramfort, Michel, Thirion,
  Grisel, Blondel, Prettenhofer, Weiss, Dubourg, et~al.]{pedregosa2011scikit}
Fabian Pedregosa, Ga{\"e}l Varoquaux, Alexandre Gramfort, Vincent Michel,
  Bertrand Thirion, Olivier Grisel, Mathieu Blondel, Peter Prettenhofer, Ron
  Weiss, Vincent Dubourg, et~al.
\newblock Scikit-learn: Machine learning in {P}ython.
\newblock \emph{Journal of Machine Learning Research}, 12:\penalty0 2825--2830,
  2011.

\bibitem[Radford et~al.(2021)Radford, Kim, Hallacy, Ramesh, Goh, Agarwal,
  Sastry, Askell, Mishkin, Clark, et~al.]{radford2021learning}
Alec Radford, Jong~Wook Kim, Chris Hallacy, Aditya Ramesh, Gabriel Goh,
  Sandhini Agarwal, Girish Sastry, Amanda Askell, Pamela Mishkin, Jack Clark,
  et~al.
\newblock Learning transferable visual models from natural language
  supervision.
\newblock In \emph{International Conference on Machine Learning (ICML)}, 2021.

\bibitem[Rakhlin and Sridharan(2013)]{rakhlin2013optimization}
Sasha Rakhlin and Karthik Sridharan.
\newblock Optimization, learning, and games with predictable sequences.
\newblock In \emph{Advances in Neural Information Processing Systems
  (NeurIPS)}, 2013.

\bibitem[Reddi et~al.(2018)Reddi, Kale, and Kumar]{reddi2018convergence}
Sashank~J Reddi, Satyen Kale, and Sanjiv Kumar.
\newblock On the convergence of {A}dam and beyond.
\newblock In \emph{International Conference on Learning Representations
  (ICLR)}, 2018.

\bibitem[Shallue et~al.(2019)Shallue, Lee, Antognini, Sohl-Dickstein, Frostig,
  and Dahl]{shallue2019measuring}
Christopher~J Shallue, Jaehoon Lee, Joseph Antognini, Jascha Sohl-Dickstein,
  Roy Frostig, and George~E Dahl.
\newblock Measuring the effects of data parallelism neural network training.
\newblock \emph{Journal of Machine Learning Research}, 20:\penalty0 1--49,
  2019.

\bibitem[Shamir and Zhang(2013)]{shamir2013stochastic}
Ohad Shamir and Tong Zhang.
\newblock Stochastic gradient descent for non-smooth optimization: Convergence
  results and optimal averaging schemes.
\newblock In \emph{International Conference on Machine Learning (ICML)}, 2013.

\bibitem[Shazeer and Stern(2018)]{shazeer2018adafactor}
Noam Shazeer and Mitchell Stern.
\newblock Adafactor: Adaptive learning rates with sublinear memory cost.
\newblock In \emph{International Conference on Machine Learning (ICML)}, 2018.

\bibitem[Streeter and McMahan(2012)]{streeter2012no}
Matthew Streeter and H~Brendan McMahan.
\newblock No-regret algorithms for unconstrained online convex optimization.
\newblock In \emph{Advances in Neural Information Processing Systems
  (NeurIPS)}, 2012.

\bibitem[Vaswani et~al.(2019)Vaswani, Mishkin, Laradji, Schmidt, Gidel, and
  Lacoste-Julien]{vaswani2019painless}
Sharan Vaswani, Aaron Mishkin, Issam Laradji, Mark Schmidt, Gauthier Gidel, and
  Simon Lacoste-Julien.
\newblock Painless stochastic gradient: Interpolation, line-search, and
  convergence rates.
\newblock In \emph{Advances in Neural Information Processing Systems
  (NeurIPS)}, 2019.

\bibitem[Virtanen et~al.(2020)Virtanen, Gommers, Oliphant, Haberland, Reddy,
  Cournapeau, Burovski, Peterson, Weckesser, Bright, {van der Walt}, Brett,
  Wilson, Millman, Mayorov, Nelson, Jones, Kern, Larson, Carey, Polat, Feng,
  Moore, {VanderPlas}, Laxalde, Perktold, Cimrman, Henriksen, Quintero, Harris,
  Archibald, Ribeiro, Pedregosa, {van Mulbregt}, and {SciPy 1.0
  Contributors}]{virtanen2020scipy}
Pauli Virtanen, Ralf Gommers, Travis~E. Oliphant, Matt Haberland, Tyler Reddy,
  David Cournapeau, Evgeni Burovski, Pearu Peterson, Warren Weckesser, Jonathan
  Bright, St{\'e}fan~J. {van der Walt}, Matthew Brett, Joshua Wilson, K.~Jarrod
  Millman, Nikolay Mayorov, Andrew R.~J. Nelson, Eric Jones, Robert Kern, Eric
  Larson, C~J Carey, {\.I}lhan Polat, Yu~Feng, Eric~W. Moore, Jake
  {VanderPlas}, Denis Laxalde, Josef Perktold, Robert Cimrman, Ian Henriksen,
  E.~A. Quintero, Charles~R. Harris, Anne~M. Archibald, Ant{\^o}nio~H. Ribeiro,
  Fabian Pedregosa, Paul {van Mulbregt}, and {SciPy 1.0 Contributors}.
\newblock {{SciPy} 1.0: Fundamental Algorithms for Scientific Computing in
  Python}.
\newblock \emph{Nature Methods}, 17:\penalty0 261--272, 2020.

\bibitem[{W}es {M}c{K}inney(2010)]{mckinney2010pandas}
{W}es {M}c{K}inney.
\newblock {D}ata {S}tructures for {S}tatistical {C}omputing in {P}ython.
\newblock In \emph{{P}roceedings of the 9th {P}ython in {S}cience
  {C}onference}, 2010.

\bibitem[Wightman(2019)]{Wightman2019timm}
Ross Wightman.
\newblock Py{T}orch image models.
\newblock \url{https://github.com/rwightman/pytorch-image-models}, 2019.

\bibitem[Xiao et~al.(2010)Xiao, Hays, Ehinger, Oliva, and
  Torralba]{xiao2010sun}
Jianxiong Xiao, James Hays, Krista~A Ehinger, Aude Oliva, and Antonio Torralba.
\newblock Sun database: Large-scale scene recognition from abbey to zoo.
\newblock In \emph{Conference on Computer Vision and Pattern Recognition
  (CVPR)}, 2010.

\bibitem[Xiao et~al.(2016)Xiao, Ehinger, Hays, Torralba, and
  Oliva]{xiao2016sun}
Jianxiong Xiao, Krista~A Ehinger, James Hays, Antonio Torralba, and Aude Oliva.
\newblock Sun database: Exploring a large collection of scene categories.
\newblock \emph{International Journal of Computer Vision}, 119\penalty0
  (1):\penalty0 3--22, 2016.

\bibitem[Zagoruyko and Komodakis(2016)]{zagoruyko2016wide}
Sergey Zagoruyko and Nikos Komodakis.
\newblock Wide residual networks.
\newblock In \emph{British Machine Vision Conference (BMVC)}, 2016.

\bibitem[Zhai et~al.(2019)Zhai, Puigcerver, Kolesnikov, Ruyssen, Riquelme,
  Lucic, Djolonga, Pinto, Neumann, Dosovitskiy, Beyer, Bachem, Tschannen,
  Michalski, Bousquet, Gelly, and Houlsby]{zhai2019large}
Xiaohua Zhai, Joan Puigcerver, Alexander Kolesnikov, Pierre Ruyssen, Carlos
  Riquelme, Mario Lucic, Josip Djolonga, Andr{\'e}~Susano Pinto, Maxim Neumann,
  Alexey Dosovitskiy, Lucas Beyer, Olivier Bachem, Michael Tschannen, Marcin
  Michalski, Olivier Bousquet, Sylvain Gelly, and Neil Houlsby.
\newblock A large-scale study of representation learning with the visual task
  adaptation benchmark.
\newblock \emph{arXiv:1910.04867}, 2019.

\end{thebibliography}
\end{document}